\documentclass[10pt]{article}
\usepackage[accepted]{tmlr}

\usepackage[utf8]{inputenc}   
\usepackage[T1]{fontenc}      

\usepackage{comment}
\usepackage{hyperref}
\usepackage{natbib}

\usepackage[tbtags]{amsmath}
\usepackage{amsthm}
\allowdisplaybreaks
\usepackage{amssymb,mathrsfs}
\usepackage{amsfonts}
\usepackage{upgreek}
\usepackage{xspace}

\definecolor{airforceblue}{rgb}{0.36, 0.54, 0.66}
\definecolor{thistle}{rgb}{0.85, 0.75, 0.85}
\definecolor{ticklemepink}{rgb}{0.99, 0.54, 0.67}
\definecolor{thulianpink}{rgb}{0.67, 0.24, 0.43}
\definecolor{tealblue}{rgb}{0.11, 0.36, 0.43}

\usepackage{color}
\newcommand{\bl}[1]{\textcolor{tealblue}{#1}}
\newcommand{\rl}[1]{\textcolor{thulianpink}{#1}}

\usepackage{graphicx}
\usepackage{subfig}
\usepackage{color}
\usepackage{algorithm, algorithmic}

\usepackage{stmaryrd}
\usepackage[inline]{enumitem}
\usepackage{url}

\usepackage{tikz}
\usetikzlibrary{calc}

\usepackage{pgfplots}
\usepackage{bbm}
\usepackage{ifthen}
\usepackage{xargs}
\usepackage[textwidth=1.8cm]{todonotes}

\usepackage{aliascnt}
\usepackage{cleveref}
\usepackage{autonum}
\makeatletter
\newtheorem{theorem}{Theorem}
\crefname{theorem}{theorem}{Theorems}
\Crefname{Theorem}{Theorem}{Theorems}

\newtheorem*{lemma_nonumber*}{Lemma}

\newaliascnt{lemma}{theorem}
\newtheorem{lemma}[lemma]{Lemma}
\aliascntresetthe{lemma}
\crefname{lemma}{lemma}{lemmas}
\Crefname{Lemma}{Lemma}{Lemmas}

\newaliascnt{corollary}{theorem}
\newtheorem{corollary}[corollary]{Corollary}
\aliascntresetthe{corollary}
\crefname{corollary}{corollary}{corollaries}
\Crefname{Corollary}{Corollary}{Corollaries}

\newaliascnt{proposition}{theorem}
\newtheorem{proposition}[proposition]{Proposition}
\aliascntresetthe{proposition}
\crefname{proposition}{proposition}{propositions}
\Crefname{Proposition}{Proposition}{Propositions}

\newaliascnt{definition}{theorem}

\aliascntresetthe{definition}
\crefname{definition}{definition}{definitions}
\Crefname{Definition}{Definition}{Definitions}

\newaliascnt{remark}{theorem}

\aliascntresetthe{remark}
\crefname{remark}{remark}{remarks}
\Crefname{Remark}{Remark}{Remarks}

\crefname{example}{example}{examples}
\Crefname{Example}{Example}{Examples}

\crefname{technique}{technique}{techniques}
\Crefname{Technique}{Technique}{Techniques}

\crefname{figure}{figure}{figures}
\Crefname{Figure}{Figure}{Figures}

\newtheorem{assumption}{\textbf{A}\hspace{-3pt}}
\crefformat{assumption}{{\textbf{A}}#2#1#3}

\newtheorem{assumptionF}{\textbf{F}\hspace{-3pt}}
\crefformat{assumptionF}{{\textbf{F}}#2#1#3}

\Crefname{assumptionB}{\textbf{B}\hspace{-3pt}}{\textbf{B}\hspace{-3pt}}
\crefname{assumptionB}{\textbf{B}}{\textbf{B}}

\Crefname{assumptionC}{\textbf{C}\hspace{-3pt}}{\textbf{C}\hspace{-3pt}}
\crefname{assumptionC}{\textbf{C}}{\textbf{C}}

\Crefname{assumptionH}{\textbf{H}\hspace{-3pt}}{\textbf{H}\hspace{-3pt}}
\crefname{assumptionH}{\textbf{H}}{\textbf{H}}

\Crefname{assumptionT}{\textbf{T}\hspace{-3pt}}{\textbf{T}\hspace{-3pt}}
\crefname{assumptionT}{\textbf{T}}{\textbf{T}}

\Crefname{assumptionT}{\textbf{T}\hspace{-3pt}}{\textbf{T}\hspace{-3pt}}
\crefname{assumptionT}{\textbf{T}}{\textbf{T}}

\Crefname{assumptionL}{\textbf{L}\hspace{-3pt}}{\textbf{L}\hspace{-3pt}}
\crefname{assumptionL}{\textbf{L}}{\textbf{L}}

\Crefname{assumptionQ}{\textbf{Q}\hspace{-3pt}}{\textbf{Q}\hspace{-3pt}}
\crefname{assumptionQ}{\textbf{Q}}{\textbf{Q}}

\Crefname{assumptionAR}{\textbf{AR}\hspace{-3pt}}{\textbf{AR}\hspace{-3pt}}
\crefname{assumptionAR}{\textbf{AR}}{\textbf{AR}}

\usepackage{bm}
\usepackage{wrapfig}

\def\eps{\epsilon}
\def\erfc{\mathrm{erfc}}
\def\Vol{\mathrm{Vol}}
\def\softmin{\mathrm{softmin}}

\def\Ctt{\mathtt{C}}
\def\Dtt{\mathtt{D}}

\def\drifta{b^{(a)}}
\def\driftb{b^{(b)}}
\def\driftc{b^{(c)}}
\def\driftd{b^{(d)}}

\def\Deltaab{\Delta^{(a,b)} b}
\def\Deltabc{\Delta^{(b,c)} b}   
\def\Deltacd{\Delta^{(c,d)} b}   

\def\EM{\mathrm{EM}}

\newcommand{\tta}{\mathtt{A}}

\newcommand{\Capprox}{\tta}

\newcommandx\ctun[1][1=T]{\Capprox_{#1,1}}

\newcommand{\rref}[1]{\tup{\Cref{#1}}}

\newcommandx{\expec}[2]{{\mathbb E}\left[#1 \middle \vert #2  \right]} 

\def\dim{d}

\newcommand{\rme}{\mathrm{e}}

\newcommand{\Ltt}{\mathtt{L}}
\newcommand{\Mtt}{\mathtt{M}}
\newcommand{\Ktt}{\mathtt{K}}

\newcommandx{\norm}[2][1=]{\ifthenelse{\equal{#1}{}}{\left\Vert #2 \right\Vert}{\left\Vert #2 \right\Vert^{#1}}}
\newcommandx{\normLigne}[2][1=]{\ifthenelse{\equal{#1}{}}{\Vert #2 \Vert}{\Vert #2\Vert^{#1}}}

\def\bfc{\mathbf{c}}
\def\bfY{\mathbf{Y}}

\def\bhfY{\hat{\mathbf{Y}}}
\def\bbfY{\bar{\mathbf{Y}}}
\def\bfX{\mathbf{X}}

\def\bfZ{\mathbf{Z}}

\def\bfZ{\mathbf{Z}}

\def\bfB{\mathbf{B}}

\def\msa{\mathsf{A}}

\def\msu{\mathsf{U}}

\def\msx{\mathsf{X}}

\def\rset{\mathbb{R}}

\def\nset{\mathbb{N}}

\def\rmP{\mathrm{P}}
\def\rmQ{\mathrm{Q}}
\def\rmR{\mathrm{R}}

\def\rmd{\mathrm{d}}

\def\rme{\mathrm{e}}

\def\rmc{\mathrm{C}}

\newcommand{\M}{\mathcal M}

\newcommandx{\functionspace}[2][1=+]{\mathbb{F}_{#1}(#2)}

\newcommandx{\VarDeux}[3][3=]{\operatorname{Var}^{#3}_{#1}\left\{#2 \right\}}

\newcommand{\1}{\mathbbm{1}}

\newcommand{\LeftEqNo}{\let\veqno\@@leqno}

\newcommand{\N}{\ensuremath{\mathbb{N}}}

\newcommand{\PE}{\mathbb{E}}

\newcommand{\Pens}{\mathscr{P}}

\newcommand{\absLigne}[1]{\vert #1 \vert}
\newcommand{\tvnorm}[1]{\| #1 \|_{\mathrm{TV}}}

\newcommandx{\Vnorm}[2][1=V]{\| #2 \|_{#1}}
\newcommandx{\VnormEq}[2][1=V]{\left\| #2 \right\|_{#1}}

\newcommand{\probaLigne}[1]{\mathbb{P}( #1 )}
\newcommandx\probaMarkovTilde[2][2=]
{\ifthenelse{\equal{#2}{}}{{\widetilde{\mathbb{P}}_{#1}}}{\widetilde{\mathbb{P}}_{#1}\left[ #2\right]}}

\newcommand{\expeLigne}[1]{\PE [ #1 ]}
\newcommand{\sqexpeLigne}[1]{\PE^{1/2} [ #1 ]}

\newcommand{\bigO}{\ensuremath{\mathcal O}}

\def\ie{\textit{i.e.}}

\def\eqsp{}
\newcommand{\coint}[1]{\left[#1\right)}
\newcommand{\ocint}[1]{\left(#1\right]}
\newcommand{\ooint}[1]{\left(#1\right)}
\newcommand{\ccint}[1]{\left[#1\right]}

\newcommandx{\weight}[2][2=n]{\omega_{#1,#2}^N}

\newcommandx\sequence[3][2=,3=]
{\ifthenelse{\equal{#3}{}}{\ensuremath{\{ #1_{#2}\}}}{\ensuremath{\{ #1_{#2}, \eqsp #2 \in #3 \}}}}

\newcommandx\sequenceD[3][2=,3=]
{\ifthenelse{\equal{#3}{}}{\ensuremath{\{ #1_{#2}\}}}{\ensuremath{( #1)_{ #2 \in #3} }}}

\newcommandx{\sequencen}[2][2=n\in\N]{\ensuremath{\{ #1_n, \eqsp #2 \}}}
\newcommandx\sequenceDouble[4][3=,4=]
{\ifthenelse{\equal{#3}{}}{\ensuremath{\{ (#1_{#3},#2_{#3}) \}}}{\ensuremath{\{  (#1_{#3},#2_{#3}), \eqsp #3 \in #4 \}}}}
\newcommandx{\sequencenDouble}[3][3=n\in\N]{\ensuremath{\{ (#1_{n},#2_{n}), \eqsp #3 \}}}

\newcommand{\wrt}{w.r.t.}

\newcommand{\opnorm}[1]{{\left\vert\kern-0.25ex\left\vert\kern-0.25ex\left\vert #1
    \right\vert\kern-0.25ex\right\vert\kern-0.25ex\right\vert}}

\def\Id{\operatorname{Id}}

\newcommandx{\CPE}[3][1=]{{\mathbb E}_{#1}\left[#2 \middle \vert #3  \right]} 
\newcommandx{\CPELigne}[3][1=]{{\mathbb E}_{#1}[#2  \vert #3  ]} 
\newcommandx{\CPEsq}[3][1=]{{\mathbb{E}^{1/2}}_{#1}\left[#2 \middle \vert #3  \right]} 
\newcommandx{\CPVar}[3][1=]{\mathrm{Var}^{#3}_{#1}\left\{ #2 \right\}}
\newcommand{\CPP}[3][]
{\ifthenelse{\equal{#1}{}}{{\mathbb P}\left(\left. #2 \, \right| #3 \right)}{{\mathbb P}_{#1}\left(\left. #2 \, \right | #3 \right)}}

\newcommandx{\osc}[2][1=]{\mathrm{osc}_{#1}(#2)}

\def\Id{\operatorname{Id}}

\newcommand{\ensembleLigne}[2]{\{#1\,:\eqsp #2\}}

\def\rmD{\mathrm{D}}

\newcommand\coupling[2]{\Gamma(\mu,\nu)}

\newcommand{\complementary}{\mathrm{c}}

\def\diam{\mathrm{diam}}

\def\vareps{\varepsilon}

\newcommandx{\KL}[2]{\operatorname{KL}\left( #1 | #2 \right)}
\newcommandx{\KLLigne}[2]{\operatorname{KL}( #1 | #2 )}
\newcommandx{\KLLignesqrt}[2]{\operatorname{KL}^{1/2}( #1 | #2 )}

\def\gaStep
\def\QKer{Q}

\def\distance{\mathbf{d}}
\newcommandx{\wasserstein}[3][1=\distance,3=]{\mathbf{W}_{#1}^{#3}\left(#2\right)}
\newcommandx{\wassersteinLigne}[3][1=\distance,3=]{\mathbf{W}_{#1}^{#3}(#2)}
\newcommandx{\wassersteinD}[1][1=\distance]{\mathbf{W}_{#1}}
\newcommandx{\wassersteinDLigne}[1][1=\distance]{\mathbf{W}_{#1}}

\def\Rcoupling{\mathrm{R}}

\def\sigmaD{\sigma^2}

\newcommandx{\phibfs}[1][1=]{\pmb{\varphi}_{\sigmaD_{#1}}}

\newcommandx\sequenceg[3][2=,3=]
{\ifthenelse{\equal{#3}{}}{\ensuremath{( #1_{#2})}}{\ensuremath{( #1_{#2})_{ #2 \geq #3}}}}

\def\Rker{\Rcoupling}

\def\Pker{\mathrm{P}}
\def\Qker{\mathrm{Q}}

\newcommandx{\distV}[1][1=\bfc]{\mathbf{W}_{#1}}
\newcommandx{\distVdeux}[1][1=W_2]{\mathbf{d}_{#1}}

\newcommand{\tup}[1]{\textup{#1}}

\usepackage{xspace}

\makeatletter
\DeclareRobustCommand\onedot{\futurelet\@let@token\@onedot}
\def\@onedot{\ifx\@let@token.\else.\null\fi\xspace}

\def\ie{\emph{i.e}\onedot}

\def\wrt{w.r.t\onedot} 

\makeatother

\usepackage{microtype}
\usepackage{graphicx}
\usepackage{subfig}
\usepackage{booktabs} 

\makeatletter
\newcommand{\printfnsymbol}[1]{%
  \textsuperscript{\@fnsymbol{#1}}%
}
\makeatother

\title{Convergence of denoising diffusion models \\ under the manifold
  hypothesis}

\author{\name Valentin De Bortoli  \\
  \addr Department of Computer Science\\
  ENS, CNRS, PSL University\\
  Paris, France}

\def\month{11}  
\def\year{2022} 
\def\openreview{\url{https://openreview.net/forum?id=MhK5aXo3gB&}} 

\begin{document}
\maketitle

\begin{abstract}
  Denoising diffusion models are a recent class of generative models exhibiting
  state-of-the-art performance in image and audio synthesis. Such models
  approximate the time-reversal of a forward noising process from a target
  distribution to a reference measure, which is usually Gaussian. Despite their
  strong empirical results, the theoretical analysis of such models remains
  limited. In particular, all current approaches crucially assume that the
  target density admits a density w.r.t.\ the Lebesgue measure. This does not
  cover settings where the target distribution is supported on a
  lower-dimensional manifold or is given by some empirical distribution. In this
  paper, we bridge this gap by providing the first convergence results for
  diffusion models in this setting. In particular, we provide quantitative
  bounds on the Wasserstein distance of order one between the target data
  distribution and the generative distribution of the diffusion model.
\end{abstract}


\section{Introduction}
\label{sec:introduction}

Diffusion modeling, also called score-based generative modeling, is a new
paradigm for generative modeling which exhibits state-of-the-art performance in
image and audio synthesis
\citep{song2019generative,song2020score,ho2020denoising,nichol2021improved,nichol2021beatgans}. Such
models first consider a forward stochastic process, adding noise to the data
until a Gaussian distribution is reached. The model then approximates the
backward process associated with this forward noising process. It can be shown,
see \citep{haussmann1986time} for instance, that in order to compute the drift
of the backward trajectory, the gradient of the forward logarithmic density
(Stein score) must be estimated. Such an estimator is then obtained using score
matching techniques \citep{hyvarinen2005estimation,vincent2011connection} and
leveraging neural network techniques. At sampling time, the backward process is
initialized with a Gaussian and run backward in time using the approximation of
the Stein score. Despite impressive empirical results, theoretical understanding
and convergence analysis of diffusion models remain
limited. \citet{debortoli2021diffusion} establish the convergence of diffusion
models in total variation under the assumption that the target distribution
admits a density \wrt the Lebesgue measure and under dissipativity
conditions. More recently \citet{lee2022convergence} obtained convergence
results for diffusion models, including predictor-corrector schemes, under the
assumption that the target distribution admits a density \wrt the Lebesgue
measure and satisfies a log-Sobolev inequality.

However, these works implicitly assume that the score does not explode as
$t \to 0$, by imposing that the score of the data distribution is Lipschitz
continuous or satisfies some growth property. This is not observed in practice
and experimentally the norm of the score blows up when $t \to 0$, see
\citep{kim2021soft} for instance. Indeed, the assumptions that the target
distribution admits a density \wrt the Lebesgue measure and has a Lipschitz
logarithmic gradient does not hold if one assumes the \emph{manifold hypothesis}
\citep{tenenbaum2000global,fefferman2016testing,goodfellow2016deep,brown2022union}
or if the target measure is an empirical measure. In this setting, the target
distribution is supported on a lower dimensional compact set.  In the case of
image processing, this hypothesis is supported by empirical evidence
\citep{weinberger2006unsupervised,fefferman2016testing}.  Under this hypothesis,
even though the forward process admits a density for all $t >0$ its logarithmic
gradient explodes for small $t \rightarrow 0$. Consequently, previous
theoretical analyses of diffusion models do not apply to this setting. To our
knowledge, \citep{pidstrigach2022score} is the only existing work investigating
the convergence of diffusion models under such manifold assumptions by showing
that the limit of the continuous backward process with approximate score is
well-defined and that its distribution is equivalent to the one of the target
distribution under integrability conditions on the error of the score. In
particular, \cite{pidstrigach2022score} show that these distributions have the
same support.

In this work, we complement these results and study the
convergence rate of diffusion models under the manifold hypothesis. More
precisely, we derive quantitative convergence bounds in Wasserstein distance of
order one between the target distribution and the generative distribution of the
diffusion model. The rest of the paper is organized as follows. In
\Cref{sec:diff-models-gener}, we recall the basics of diffusion models. We
present our main results and discuss links with the existing literature in
\Cref{sec:main-results}. The rest of the paper is dedicated to the proof of
\Cref{thm:convergence_general} in \Cref{sec:proof_theorem_one}. We conclude and
explore future avenues in \Cref{sec:conclusion}.


\section{Diffusion models for generative modeling }
\label{sec:diff-models-gener}

In this section, we recall the basics of diffusion models.  Henceforth, let
$\pi \in \Pens(\rset^d)$ denote the target distribution, also known as the data
distribution, and $\pi_\infty = \mathrm{N}(0, \Id)$ the $d$-dimensional Gaussian
distribution with zero mean and identity covariance matrix. In what follows, we
let $T >0$ and consider the forward noising process
$(\bfX_t)_{t \in \ccint{0,T}}$ given by an Ornstein--Uhlenbeck\footnote{Also
  called Variance Preserving Stochastic Differential Equation (VPSDE) in
  \cite{song2020score}.} process as follows
\begin{equation}
  \label{eq:ornstein_ulhenbeck}
  \rmd \bfX_t = - \beta_t \bfX_t \rmd t + \sqrt{2 \beta_t} \rmd \bfB_t \eqsp , \qquad \bfX_0 \sim \pi \eqsp . 
\end{equation}
where $(\bfB_t)_{t \geq 0}$ is a $d$-dimensional Brownian motion and
$t \mapsto \beta_t$ is a (positive) weight function. In practice, setting
$\beta_0 \leq \beta_T$ allows for better control of the backward diffusion near
the target distribution, see \citep{nichol2021improved,song2020score} for
instance. In what follows, we assume that \eqref{eq:ornstein_ulhenbeck} admits a
strong solution. Under mild assumptions on the target distribution
\citep{haussmann1986time,cattiaux2021time}, the backward process
$(\bfY_t)_{t \in \ccint{0,T}} = (\bfX_{T-t})_{t \in \ccint{0,T}}$ satisfies
\begin{equation}
  \label{eq:time_reversal_sde}
  \textstyle{
    \rmd \bfY_t = \beta_{T-t} \{\bfY_t + 2 \nabla \log p_{T-t}(\bfY_t) \} \rmd t + \sqrt{2 \beta_{T-t}} \rmd \bfB_t \eqsp ,
    }
  \end{equation}
  where $\{p_t\}_{t \in \ocint{0,T}}$ the family of densities of
  $\{\mathcal{L}(\bfX_t)\}_{t \in \ocint{0,T}}$ \footnote{For any
    $\rset^d$-valued random variable $X$, $\mathcal{L}(X)$ is the distribution
    of $X$.} \wrt the Lebesgue measure. In order to define
  \eqref{eq:time_reversal_sde} we do not need to assume that $\pi$ admits a
  density \wrt the Lebesgue measure. In practice, instead of sampling from
  $\bfY_0 \sim \mathcal{L}(\bfX_T)$ we sample from
  $\bfY_0 \sim \pi_\infty = \mathrm{N}(0, \Id)$. For large $T >0$ the mismatch
  between the distribution of $\bfX_T$ and $\pi_\infty$ is small due to
  geometric convergence of the Ornstein--Uhlenbeck process.

  In practice, $\{\nabla \log p_t\}_{t \in \ccint{0,T}}$ cannot be computed
  exactly and is approximated by a family of estimators
  $\{\bm{s}(t, \cdot)\}_{t \in \ccint{0,T}}$. Those estimators minimize
  the denoising score matching loss function $\ell$ given by
  \begin{equation}
    \label{eq:dsm_loss}
    \textstyle{
      \ell(\bm{s}) = \int_0^T \phi(t) \expeLigne{\normLigne{\bm{s}(t, \bfX_t) - \nabla \log p_{t|0}(\bfX_t|\bfX_0)}^2} \rmd t \eqsp ,
      }
    \end{equation}
    with $p_{t|0}$ is the density of $\bfX_t$ given $\bfX_0$, \ie the density of
    the transition kernel associated with \eqref{eq:ornstein_ulhenbeck} and
    $\phi: \ \ccint{0,T} \to \rset_+$ is a weighting function. In practice,
    \eqref{eq:dsm_loss} is approximated using Monte Carlo samples and the loss
    function is minimized over the parameters of a neural network.

    Once the score estimator $\bm{s}$ is learned, we introduce a continuous-time
    backward process $(\hat{\bfY}_t)_{t \in \ccint{0,T}}$ approximating
    $(\bfY_t)_{t \in \ccint{0,T}}$ and given by
    \begin{equation}
      \label{eq:approximate_time_reversal_continuous}
  \textstyle{
    \rmd \hat{\bfY}_t = \beta_{T-t} \{\hat{\bfY}_t + 2 \bm{s}(T-t, \hat{\bfY}_t) \} \rmd t + \sqrt{2 \beta_{T-t}} \rmd \bfB_t \eqsp , \qquad \hat{\bfY}_0 \sim \pi_\infty = \mathrm{N}(0, \Id) \eqsp . 
    }
  \end{equation}
  In practice, one needs to discretize
  \eqref{eq:approximate_time_reversal_continuous} in order to define an
  algorithm which can be implemented. We consider a sequence of stepsizes
  $\{\gamma_k\}_{k \in \{0, \dots, K\}}$ such that
  $\sum_{k=0}^{K} \gamma_k = T$. In what follows, for any
  $k \in \{0, \dots, K\}$ we denote $t_{k+1} = \sum_{j=0}^k \gamma_j$ and
  $t_0 =0$\footnote{Note that $t_{K+1} = T$.}.  Given this sequence of stepsizes,
  we consider the interpolation process $(\bbfY_t)_{t\in \ccint{0,T}}$ defined
  for any $k \in \{0, \dots, K\}$ and $t \in \ccint{t_k, t_{k+1}}$ by
\begin{equation}
  \label{eq:OU_disc}
  \textstyle{
    \rmd \bbfY_t = \beta_{T-t} \{ \bbfY_{t} + 2 \bm{s}(T-t_k, \bbfY_{t_k}) \} \rmd t + \sqrt{2\beta_{T-t}} \rmd \bfB_t \eqsp , \qquad \bbfY_0 \sim \pi_\infty \eqsp .
    }
  \end{equation}
  This process is an Ornstein--Uhlenbeck process on
  the interval $\ccint{t_k, t_{k+1}}$.  Denoting
  $(Y_k)_{k \in \{0, \dots, K+1\}}$ such that for any $k \in \{0, \dots, K+1\}$,
  $Y_k = \bbfY_{t_k}$, we have for any $k \in \{0, \dots, K\}$
  \begin{align}
    \label{eq:discretization_improved}
      Y_{k+1} &=     \textstyle{Y_k + \gamma_{1,k} (Y_{k} + 2 \bm{s}(T-t_k, Y_{k})) + \sqrt{2 \gamma_{2,k}} Z_k \eqsp ,
                } \\
    \gamma_{1,k} &= \textstyle{\exp[\int_{T-t_{k+1}}^{T-t_k} \beta_s \rmd s] - 1 \eqsp , \qquad \gamma_{2,k} = (\exp[2 \int_{T-t_{k+1}}^{T-t_k} \beta_s \rmd s] - 1)/2 \eqsp . }
    \end{align}
    where $\{Z_k\}_{k \in \nset}$ is a sequence of independent $d$-dimensional
    Gaussian random variables with zero mean and identity covariance matrix.
    The discretization \eqref{eq:discretization_improved} approximately
    corresponds to the discrete-time scheme introduced in
    \citep{ho2020denoising}, see \Cref{sec:equivalence-with-}. We call this
    discretization scheme the \emph{exponential integrator} (EI) discretization,
    similarly to \cite{zhang2022exponential} who introduced a similar scheme in
    accelerated deterministic diffusion models. \citet{lee2022convergence}
    analyze a slightly different scheme corresponding to replacing
    $\beta_{T-t} \bbfY_t$ by $\beta_{T-t} \bbfY_{t_k}$ in
    \eqref{eq:approximate_time_reversal_continuous}. We summarize the processes
    we have introduced in \Cref{tab:processes} and discuss the links between
    \eqref{eq:discretization_improved} and the classical Euler--Maruyama
    discretization in \Cref{sec:link-with-euler}.
    \begin{table}[h]
\small
\centering
\renewcommand*{\arraystretch}{1.2}
\begin{tabular}{ll}
  \toprule 
 Description  &     Evolution equation \\ \hline
  Forward process    & $\rmd \bfX_t = - \beta_t \bfX_t \rmd t + \sqrt{2 \beta_t} \rmd \bfB_t$  \\
  Backward process (BP) & $\rmd \bfY_t = \beta_{T-t} \{\bfY_t + 2 \nabla \log p_{T-t}(\bfY_t) \} \rmd t + \sqrt{2 \beta_{T-t}} \rmd \bfB_t$ \\ 
  Score approximate BP (SBP) & $\rmd \hat{\bfY}_t = \beta_{T-t} \{\hat{\bfY}_t + 2 \bm{s}(T-t, \hat{\bfY}_t) \} \rmd t + \sqrt{2 \beta_{T-t}} \rmd \bfB_t $ \\
        EI interpolation of SBP      & $\rmd \bbfY_t = \beta_{T-t} \{ \bbfY_{t} + 2 \bm{s}(T-t_k, \bbfY_{t_k}) \} \rmd t + \sqrt{2\beta_{T-t}} \rmd \bfB_t $ \\
  EI discretization of SBP  &
$ Y_{k+1} = \textstyle{Y_k + \gamma_{1,k} (Y_{k} + 2 \bm{s}(T-t_k, Y_{k})) + \sqrt{2 \gamma_{2,k}}} Z_k$ \\ 
  \bottomrule
\end{tabular}
\vspace{.2cm}
\caption{\small Different processes considered in this paper.}
\label{tab:processes}
\end{table}
As emphasized in the introduction, under the manifold hypothesis or in the case
where the target distribution is an empirical measure, the true score
$\nabla \log p_t$ explodes when $t \to 0$. This behavior has been observed in
practice for image synthesis \citep{kim2021soft, song2020improved}. One way to
deal with this explosive behavior is to truncate the integration of the backward
diffusion, \ie instead of running $(\bfY_t)_{t \in \ccint{0,T}}$ we consider
$(\bfY_t)_{t \in \ccint{0, T - \vareps}}$ for a small hyperparameter
$\vareps >0$, \citep{vahdat2021score, song2020improved}. Translating this
condition on the associated discretized process, we assume that
$t_{K} = T - \vareps$ and study $\{Y_k\}_{k \in \{0, \dots, K\}}$ by
disregarding the last sample $Y_{K+1}$. We note that versions of diffusion
models defined in discrete time do not suffer from such shortcomings as the
truncation is embedded in the discretization scheme, see
\citep{song2020score,song2020improved,song2019generative,ho2020denoising} for
instance. Recently \citet{kim2021soft} have proposed a soft probabilistic
truncation to replace the proposed hard threshold.

\section{Main results}
\label{sec:main-results}

We first start by introducing and discussing our main assumptions. The
only assumption we consider on the data distribution $\pi$ is that it is
supported on a compact set $\M \subset \rset^d$ (\ie a bounded and closed
subset of $\rset^d$).
\begin{assumption}
  \label{assum:manifold_hyp}
  $\pi$ is supported on a compact set $\M$ and $0 \in \M$.
\end{assumption}
The assumption $0 \in \M$ can be omitted but is kept to simplify the proofs. We
denote $\mathrm{diam}(\M)$ the diameter of the manifold defined by $  \mathrm{diam}(\M) = \sup \ensembleLigne{\normLigne{x-y}}{x, y \in \M}$.

An assumption of compactness is natural in image processing as images are
encoded on a finite range (typically $\ccint{0, 255}$ for each channel). We
emphasize that this assumption encompasses not only all distributions which
admit a continuous density on a lower dimensional manifold but also all
empirical densities of the form $(1/N)\sum_{i=1}^N \updelta_{X^i}$. Next, we
turn to the temperature schedule $t \mapsto \beta_t$ and make the following
assumption.

\begin{assumption}
  \label{assum:assumption_beta}
  $t \mapsto \beta_t$ is continuous, non-decreasing and there exists
  $\bar{\beta} > 0$ such that for any $t \in \ccint{0,T}$,
  $1/\bar{\beta} \leq \beta_t \leq \bar{\beta}$.
\end{assumption}
Under this assumption, the integral of $t \mapsto \beta_t$ is well-defined and for any $t \in \ccint{0,T}$ we have that 
\begin{equation}
  \textstyle{
  \bfX_t = m_t \bfX_0 + \sigma_t Z \eqsp , \qquad m_t = \exp[-\int_0^t \beta_s \rmd s] \eqsp , \qquad \sigma_t^2 = 1 - \exp[-2\int_0^t\beta_s \rmd s] \eqsp , }
\end{equation}
where the first equality holds in distribution and $Z$ is a Gaussian random
variable with zero mean and identity covariance. Note that
\rref{assum:assumption_beta} is satisfied for every schedule used in practice,
see \Cref{sec:assumptions-schedule}. Finally, we make the following assumption
on the score network.
  \begin{assumption}
    \label{assum:score_control}
    There exist $\bm{s} \in \rmc(\ccint{0,T} \times \rset^d, \rset^d)$ and
    $\Mtt \geq 0$ such that for any $t \in \ccint{0,T}$ and $x_t \in \rset^d$,
  \begin{equation}
    \normLigne{\bm{s}(t, x_t) - \nabla \log p_t(x_t)} \leq \Mtt (1+\normLigne{x_t})/\sigma_{t}^2 \eqsp . 
  \end{equation}
\end{assumption}

Contrary to \citet{debortoli2021diffusion}, we do not assume a uniform bound in
time and space as we allow growth as $t \to 0$ and $\normLigne{x} \to 0$. This
assumption is more realistic as
$\norm{\nabla \log p_t(x_t)} \sim_{t\to 0}c_0(x_t)/\sigma_t^2$ and
$\norm{\nabla \log p_t(x_t)} \sim_{\normLigne{x_t}\to
  +\infty}c_1(t)\normLigne{x_t}$ as we will show in
\Cref{sec:grad-hess-contr}. This explosive behavior as $t\to0$ is accounted for
in practical implementations. For example \citet{song2020score} used a
parameterization of the score of the form $\bm{s}(t,x) = \bm{n}(t,x)/\sigma_t$,
where $\bm{n}$ is a neural network with learnable parameters.  Our assumption is
notably different from the one of \citet{lee2022convergence} which assume a
uniform in time $\mathrm{L}^2$ bound between the score estimator and the true
score. Nevertheless, in \Cref{sec:wass-contr-under} we derive
\Cref{thm:convergence_general_L2} which is the counterpart to our main result
under a $\mathrm{L}^2$ error assumption, using the theory of
\citet{lee2022convergence} to derive an $\mathrm{L}^\infty$ error from a
$\mathrm{L}^2$ one. However our $\mathrm{L}^2$ error bounds are weaker than the
ones of \citet{lee2022convergence} as they are estimated w.r.t. to the
distribution of the \emph{algorithm} and not \wrt the true backward
distribution. We highlight that $\mathrm{L}^2$ bounds are more realistic than
$\mathrm{L}^\infty$ as the score is estimated on the data.

Finally, we make the following assumption on the sequence of stepsizes. Recall
that for any $k \in \{0, \dots, N\}$ we have $t_{k+1} =\sum_{j=0}^k \gamma_j$
and $t_0 = 0$.

\begin{assumption}
  \label{assum:step_size}
  For any $k \in \{0, \dots, K-1\}$, we have 
  $\gamma_k \sup_{v \in \ccint{T-t_{k+1}, T-t_k}} \beta_v / \sigma_v^2 \leq
  \delta \leq 1/2$.
\end{assumption}

In the case where $\beta_t = \beta_0$ for any $t\in \ccint{0,T}$, \rref{assum:step_size} is
implied by the following condition: for any $k \in \{0, \dots, K-1\}$
\begin{equation}
  \label{eq:gamma_ineq}
  \textstyle{\gamma_k (\beta_0 + (2 \sum_{j=k+1}^K \gamma_j)^{-1}) \leq \delta \eqsp .}
\end{equation}
In the next section, we fix $\gamma_K = \vareps$ and in this case, the condition
\eqref{eq:gamma_ineq} is satisfied if
$\gamma_k \leq \delta \vareps /(2 + \beta_0 \vareps)$.

\subsection{Convergence bounds}
\label{sec:convergence-bounds}
We are now ready to state our main result.
\begin{theorem}
  \label{thm:convergence_general}
  Assume \rref{assum:manifold_hyp}, \rref{assum:assumption_beta},
  \rref{assum:score_control}, \rref{assum:step_size} that
  $T \geq 2\bar{\beta}(1 + \log(1+\diam(\M))$, $\gamma_K = \vareps$ and
  $\vareps, \Mtt, \delta \leq 1/32$.
   Then, there exists $\Dtt_0 \geq 0$ such that
   \begin{equation}
         \wassersteinD[1](\mathcal{L}(Y_K), \pi) \leq \Dtt_0 (\exp[\kappa/\vareps] (\Mtt + \delta^{1/2})/ \vareps^2 + \exp[\kappa/\vareps]\exp[-T/\bar{\beta}] + \vareps^{1/2}) \eqsp ,
       \end{equation}
       with $\kappa = \diam(\M)^2(1+\bar{\beta})/2$ and
       \begin{equation}
         \label{eq:constant_diam}         
         \Dtt_0 = D (1 + \bar{\beta})^7(1 + d + \diam(\M)^4) (1 + \log(1 + \diam(\M))) \eqsp ,
       \end{equation}
       and $D$ is a numerical constant.
\end{theorem}

First, we note that letting $T \to +\infty$, $\delta, \Mtt \to 0$ and then
$\vareps \to 0$ we get that $\wassersteinD[1](\mathcal{L}(Y_{K}), \pi) \to
0$. This consequence is to be expected since
$\lim_{\vareps \to 0} \mathcal{L}(\bfY_{T-\vareps}) = \pi$. An explicit
dependency of the bound on these parameters is given in
\Cref{coro:control_eta}. More generally the error bound depends on four
variables
\begin{enumerate*}[label=(\alph*)]  
\item $\vareps$ which corresponds to the truncation of the backward process,
\item $T$ the integration time of the forward process,
\item $\delta$ which is related to a condition on the stepsizes of the backward
  discretization, see \rref{assum:step_size}
\item $\Mtt$ which controls the score approximation, see \Cref{assum:score_control}.
\end{enumerate*}
The dependence \wrt $\delta^{1/2}$ and $\Mtt$ is linear, whereas the dependence
\wrt $T$ is of the form $\exp[-T/\bar{\beta}]$. These two terms are multiplied
by a quantity depending on the truncation bound $\vareps$ which is exponential
of the form $\exp[\kappa/\vareps]$.  We conjecture that under additional
assumptions on $\M$ this dependence can be improved to also be polynomial, see
\Cref{thm:convergence_bound_hessian} for an extension of
\Cref{thm:convergence_general} under general Hessian
assumptions. Additional remarks and comments on
  \Cref{thm:convergence_general} and its assumptions are considered in
  \Cref{sec:comment-thm}.

\begin{proof}
  We provide a sketch of the proof. The detailed proof is postponed to
  \Cref{sec:proof_theorem_one}. The distribution of $Y_{K}$ is given by
  $\pi_\infty \Rker_{K}$, where $\Rker_{K}$ is the transition kernel
  associated with $Y_{K}|Y_0$. In order to control
  $\wassersteinD[1](\pi_\infty \Rker_{K}, \pi)$, we consider the following
  inequality
  \begin{equation}
    \label{eq:decomposition}
    \wassersteinD[1](\pi_\infty \Rker_{K}, \pi) \leq \wassersteinD[1](\pi_\infty \Rker_{K}, \pi_\infty \Qker_{t_{K}}) + \wassersteinD[1](\pi_\infty \Qker_{t_{K}}, \pi \Pker_{T - t_{K}}) + \wassersteinD[1](\pi \Pker_{T - t_K}, \pi) \eqsp ,
  \end{equation}
  where $(\Pker_t)_{t \in \ccint{0,T}}$ is the semi-group associated with
  $(\bfX_t)_{t \in \ccint{0,T}}$ and $(\Qker_t)_{t \in \ccint{0,T}}$ is the
  semi-group associated with $(\bfY_t)_{t \in \ccint{0,T}}$.  We then control
  each one of these terms. The first term corresponds to the discretization
  error and the score approximation. It is upper bounded by a term of the form
  $\bigO(\exp[\kappa/\vareps](\Mtt + \delta^{1/2})/\vareps^2)$. The second term
  corresponds to the convergence of the continuous-time exact backward process
  and is of order $\bigO(\exp[\kappa/\vareps]\exp[-T/\bar{\beta}])$. The last
  term corresponds to the error between the data distribution and a slightly
  noisy version of this distribution and is of order $\bigO(\vareps^{1/2})$. 
\end{proof}

As an immediate corollary of \Cref{thm:convergence_general}, we have the following result.

\begin{corollary}
  \label{coro:control_eta}
  Assume \rref{assum:manifold_hyp}, \rref{assum:assumption_beta},
  \rref{assum:score_control}, \rref{assum:step_size}. Let
  $\eta \in \ooint{0, 1/32}$, $T \geq 2\bar{\beta}(1 + \log(1+\diam(\M))$ and
  \begin{equation}
    T \geq \bar{\beta} (\kappa + 1) / \eta^2 \eqsp , \quad \Mtt \leq \exp[-\kappa/\eta^2]\eta^5 \eqsp, \quad \delta \leq \exp[-2\kappa/\eta^2]\eta^{10} \eqsp , \quad \gamma_K = \eta^2 \eqsp .
  \end{equation}
  Then, 
   \begin{equation}
         \wassersteinD[1](\mathcal{L}(Y_K), \pi) \leq 4 \Dtt_0 \eta  \eqsp ,
       \end{equation}
       with $\kappa = \diam(\M)^2(1+\bar{\beta})/2$ and $\Dtt_0$ given in
       \eqref{eq:constant_diam}.
\end{corollary}

The constant $\Dtt_0$ appearing in \Cref{thm:convergence_general} and
\Cref{coro:control_eta} does not depend on $\vareps$, $T$, $\delta$ and $\Mtt$
but only on $\bar{\beta}$, $\diam(\M)$ and $d$. In particular, we highlight that
the dependence of $\Dtt_0$ \wrt the dimension is $\bigO(d)$ and the dependence
\wrt the diameter of $\M$ is $\bigO(\diam(\M)^4)$ up to logarithmic term. Note
that the diameter might only depend on \textit{intrisic} dimension $p$ of $\M$
which satisfies $p \ll d$ in some settings. For example in the case of an
hypercube of dimension $p$ we have $\diam(\M) = \sqrt{p}$.

Contrary to
\citet{debortoli2021diffusion,lee2022convergence,pidstrigach2022score}, our
results are stated \wrt the Wasserstein distance and not the total variation
distance or the Kullback-Leibler divergence. We emphasize that studying the
total variation or Kullback-Leibler divergence between the distribution of
$Y_{K}$ and the one of $\pi$ under \rref{assum:manifold_hyp} with $\M$ lower
dimensional than $\rset^d$ lead to \emph{vacuous bounds} as these quantities are
lower bounded by $1$ in the case of the total variation and $+\infty$ in the
case of the Kullback-Leibler divergence since the densities we are comparing are
not supported on the same set.\footnote{We emphasize however that total
  variation bounds smaller than $1$ and finite Kullback-Leibler divergence have
  strong implications, namely the generative model has same support as the
  target distribution. However, such property is not satisfied in practice, see
  \Cref{sec:comment-thm} or \cite[Figure 2]{jolicoeur2020adversarial} for instance.}  This
is not the case with the Wasserstein distance of order one. To the best of our
knowledge \Cref{thm:convergence_general} is the first convergence result for
diffusion models \wrt $\wassersteinD[1]$. We note that our result could be
extended to $\wassersteinD[p]$ for any $p \geq 1$, since we do
  not rely on any property specific to $\wassersteinD[1]$ among all
  $\wassersteinD[p]$ distances for any $p \geq 1$. In particular, our analysis
  does not use the fact that $\wassersteinD[1]$ is an integral probability
  metric, \citep{sriperumbudur2009ipm}.

We conclude this section, with an improvement upon
\Cref{thm:convergence_general} in the case where tighter bounds on the Hessian
$\nabla^2 \log p_t$ are available.

\begin{theorem}
  \label{thm:convergence_bound_hessian}
  Assume \rref{assum:manifold_hyp}, \rref{assum:assumption_beta},
  \rref{assum:score_control}, \rref{assum:step_size} that
  $T \geq 2\bar{\beta}(1 + \log(1+\diam(\M))$, $\gamma_K = \vareps$ and
  $\vareps, \Mtt, \delta \leq 1/32$. In addition, assume that there exists
  $\Gamma \geq 0$ such that for any $t \in \ocint{0, T}$ and $x_t \in \rset^d$
  \begin{equation}
    \label{eq:bound_tight_nabla2}
    \normLigne{\nabla^2 \log p_t(x_t)} \leq \Gamma / \sigma_t^2 \eqsp . 
  \end{equation}
   Then, there exists $\Dtt_0 \geq 0$ such that
   \begin{equation}
         \wassersteinD[1](\mathcal{L}(Y_K), \pi) \leq \Dtt_0 ((\Mtt + \delta^{1/2})/ \vareps^{\Gamma+2} + \exp[-T/\bar{\beta}]/\vareps^{\Gamma} + \vareps^{1/2}) \eqsp ,
       \end{equation}
       with 
       \begin{equation}
         \label{eq:constant_diam_urg}
         \Dtt_0 = D (1 + d + (1+\diam(\M))^4)\exp[3(1+\bar{\beta})^2(\Gamma+2)(1+\log(1+\diam(\M)))] \eqsp . 
       \end{equation}
       and $D$ is a numerical constant.
\end{theorem}

     \begin{proof}
       The complete proof is postponed to \Cref{sec:proof-convergence-hessian}. The crux
       of the proof is to derive an improved version of
       \Cref{prop:control_gradient} which provides controls on some tangent
       process. Indeed, in \Cref{prop:control_gradient}, we use an upper bound
       of the form $\normLigne{\nabla^2 \log p_t(x_t)} \leq \Gamma / \sigma_t^4$
       which is a loose upper bound derived under \rref{assum:manifold_hyp}.
     \end{proof}

     \Cref{thm:convergence_bound_hessian} improves the bounds of
     \Cref{thm:convergence_general}, since the \emph{exponential} dependency
     \wrt $\vareps$ is replaced by a \emph{polynomial} dependency with exponent
     $\Gamma$. At first sight, it is not clear when
     \eqref{eq:bound_tight_nabla2} is satisfied. However, in special cases we
     can verify this condition explicitly. For example, in
     \Cref{sec:hess-bounds-unif}, we show that this condition is satisfied if
     $\pi$ is the uniform distribution on the hypercube, with
     $p \in \{1, \dots, d\}$. The condition \eqref{eq:bound_tight_nabla2} has
     strong geometrical implications on $\M$. In particular, under appropriate
     smoothness assumptions on $\M$, it implies that $\M$ is convex,
     \Cref{sec:non-conv-count}.

\subsection{Statistical guarantees and empirical measure targets}

We emphasize that the results of \Cref{thm:convergence_general} hold under the
general assumption \Cref{assum:manifold_hyp} which only requires the target
measure to be supported on a compact set. This includes measures which are
supported on a smooth manifold of dimension $p \leq d$ but also all empirical
measures of the form $(1/N) \sum_{i=1}^N \updelta_{X^i}$ with
$\{X^i\}_{i=1}^N \sim \pi^{\otimes N}$. In particular if we assume that the
underlying target measure $\pi$ is supported on a manifold of dimension
$p \leq d$ and that the diffusion models are trained \wrt some empirical
measure associated with $\pi$ then we have the following result.

\begin{proposition}
  \label{prop:gener-results}
  Assume \rref{assum:manifold_hyp}, \rref{assum:assumption_beta},
  \rref{assum:score_control}, \rref{assum:step_size} that
  $T \geq 2\bar{\beta}(1 + \log(1+\diam(\M))$, $\gamma_K = \vareps$ and
  $\vareps, \Mtt, \delta \leq 1/32$.
   Then, for any $\eta >0$ there exist $\Dtt_0, \Dtt_1 \geq 0$ such that
   \begin{equation}
         \expeLigne{\wassersteinD[1](\mathcal{L}(Y_K), \pi)} \leq \Dtt_0 (\exp[\kappa/\vareps] (\Mtt + \delta^{1/2})/ \vareps^2 + \exp[\kappa/\vareps]\exp[-T/\bar{\beta}] + \vareps^{1/2}) + \Dtt_1N^{-1/(\dim_M(\M)+\eta)} \eqsp ,
       \end{equation}
       with $\dim_M(\M)$ the Minkowski dimension of $\M$, see
       \eqref{eq:minkowski}, $\kappa = \diam(\M)^2(1+\bar{\beta})/2$, $\Dtt_1$
       given in \cite[Theorem 1]{weed2019sharp} and
       \begin{equation}
         \Dtt_0 = D (1 + \bar{\beta})^7(1 + d + \diam(\M)^4) (1 + \log(1 + \diam(\M))) \eqsp ,
       \end{equation}
       with $D$ a numerical constant.
     \end{proposition}

     \begin{proof}
       For any $N \in \nset$, we denote $\pi^N = (1/N) \sum_{i=1}^N
       \updelta_{X^i}$. Using \Cref{thm:convergence_general}, we have that for any $N \in \nset$
       \begin{equation}
         \wassersteinD[1](\mathcal{L}(Y_K), \pi^N) \leq \Dtt_0 (\exp[\kappa/\vareps] (\Mtt + \delta^{1/2})/ \vareps^2 + \exp[\kappa/\vareps]\exp[-T/\bar{\beta}] + \vareps) \eqsp ,
       \end{equation}
       with a constant $\Dtt_0$ which does not depend on $\{X^i\}_{i=1}^N$ and
       $N$. Therefore, we have that for any $N \in \nset$
       \begin{equation}
         \label{eq:wasserstein_target_empirique}
         \expeLigne{\wassersteinD[1](\mathcal{L}(Y_K), \pi^N)} \leq \Dtt_0 (\exp[\kappa/\vareps] (\Mtt + \delta^{1/2})/ \vareps^2 + \exp[\kappa/\vareps]\exp[-T/\bar{\beta}] + \vareps) \eqsp .
       \end{equation}
       Using \cite[Theorem 1]{weed2019sharp} and \cite[Proposition
       2]{weed2019sharp}, for any $\eta >0$, there exists $\Dtt_1 \geq 0$ such
       that
       \begin{equation}
         \expeLigne{\wassersteinD[1](\pi^N, \pi)} \leq \Dtt_1 N^{-1/(\dim_M(\M) + \eta)}\eqsp ,
       \end{equation}
       which concludes the proof upon combining this result,
       \eqref{eq:wasserstein_target_empirique} and the triangle inequality. 
     \end{proof}

     The Minkowski dimension $\dim(\M)$ is defined as follows:
     \begin{equation}
       \label{eq:minkowski}
       \textstyle{\dim(\M) = d - \liminf_{\vareps \to 0} \log(\Vol(\M_{\vareps}))/\log(1/\vareps)\eqsp , }
     \end{equation}
     with $\Vol(\msa)$ the volume of a (measurable) set $\msa$ and $\M_\vareps$
     the $\vareps$-fattening of $\M$, \ie for any $ \vareps >0$,
     $\M_\vareps = \ensembleLigne{x \in \rset^d}{d(x, \M) \leq \vareps}$.  For
     example if $\M$ is a topological manifold of dimension $p \leq d$ then its
     Minkowski dimension is $p$, \ie $\dim_M(\M) = p$. Hence, in this case the
     error term in \Cref{prop:gener-results} depends \emph{exponentially} on the
     dimension of $\M$ and on its diameter but depends only \emph{linearly} on
     $d$, the dimension of the ambient space. Note again that $\diam(\M)$ might
     depend on the dimension of $\M$. For example in the case of the hypercube
     $\M = [-1/2,1/2]^p$, we have $\diam(\M) = \sqrt{p}$. Hence, the results of
     \Cref{prop:gener-results} show that diffusion models exploit the
     lower-dimensional structure of the target.  We highlight that this result
     does not quantify the \emph{diversity} of diffusion models, \ie their
     ability to produce samples which are distinct from the ones of the training
     dataset \cite{alaa2022faithful,zhao2018bias}. There is empirical evidence
     that denoising diffusion models yield generative models with good diversity
     properties \cite{xiao2021tackling,nichol2021beatgans} and we leave the
     theoretical study of the diversity of denoising diffusion models for future
     work.
     
\subsection{Related works}
\label{sec:related-works}

To the best of our knowledge, \citep{debortoli2021diffusion} is the first
quantitative convergence results for denoising diffusion models. More precisely,
\citet{debortoli2021diffusion} show a bound in total variation between the
distribution of the diffusion model and the target distribution of the form
\begin{equation}
  \label{eq:earlier_tv}
  \tvnorm{\mathcal{L}(Y_{K+1}) - \pi} \leq A( \exp[-T] + \exp[T] (\Mtt^{1/2} + \delta^{1/2})) \eqsp . 
\end{equation}
This result holds under the assumption that $\pi$ admits a density \wrt Lebesgue
measure which satisfies some dissipativity conditions. Again we emphasize that
such results in total variation are vacuous under the manifold hypothesis. The
upper bound in \eqref{eq:earlier_tv} is obtained using a similar splitting of
the error as in \Cref{thm:convergence_general}. However the control of the
discretization error is handled using Girsanov formula in
\citep{debortoli2021diffusion} and relies on similar techniques as
\citep{dalalyan2017theoretical,durmus2017nonasymp}. In the present work, this
error is controlled using the interpolation formula from \citet{del2019backward}
which, combined with controls on stochastic flows, allows for tighter controls
of the discretization error \wrt $\wassersteinD[1]$.

\citet{lee2022convergence} study the convergence of diffusion models under
(uniform in time) $\mathrm{L}^2$ controls on the score approximation. Their
result is given \wrt the total variation and therefore suffers from the same
shortcoming as the ones of \citet{debortoli2021diffusion}. In particular it is
assumed that the data distribution admits a density \wrt the Lebesgue measure
which satisfies some regularity conditions as well as a logarithmic Sobolev
inequality. Additionally, it is required that $\nabla^2 \log p_t$ is bounded
uniformly in time and in space which is not true under the manifold hypothesis
and is hard to verify in practice even in simple cases.

Closer to our line of work are the results of \citet{pidstrigach2022score} who
proves that the approximate backward process
\eqref{eq:approximate_time_reversal_continuous} converges to a random variable
whose distribution is supported on the manifold of interest. In this work, we
complement these results by studying the discretization scheme and providing
quantitative bounds between the output of the diffusion model and the target
distribution.

Related to the manifold hypothesis and the study of convergence of diffusion
models  \citet{de2022riemannian}  study the convergence of a
Riemannian counterpart of diffusion models. Result are given \wrt the total
variation (defined on the manifold of interest). Even though such diffusion
models directly incorporate the manifold information they require the knowledge
of the geodesics and the Riemannian metric of the manifold. In the case of the
manifold hypothesis these quantities are not known and therefore cannot be used
in practice. In particular, \citet{de2022riemannian} focus on manifolds which
have a well-known structure such as $\mathbb{S}^1$, $\mathbb{T}^2$ or
$\mathrm{SO}_3(\rset)$.

\citet{franzese2022much} show that there exists a trade-off between long and
short time horizons $T$. Their analysis is based on a rearrangement of the
Evidence Lower Bound (ELBO) obtained by \citet{huang2021variational}. This ELBO
can be decomposed in the sum of two terms: one which decreases with $T$
(controlling the bias between $\mathcal{L}(\bfX_T)$ and $\pi_\infty$) and one
which increases with $T$ (corresponding to the loss term
\eqref{eq:dsm_loss}). Their decomposition of the ELBO is in fact equivalent to
\cite[Theorem 1]{durkan2021maximum}. In \Cref{sec:short-proof} we include a
short derivation of this result.

Finally, we highlight the earlier results of \citet{block2020generative}. In
this work, the authors study a version of the Langevin algorithm in which the
score term is approximated. This is different from the diffusion model
\footnote{Even though the authors provide a discussion on an annealed version of
  the algorithm they study which corresponds to the original framework of
  \citet{song2019generative}.} setting and is closer to the setting of
Plug-and-Play (PnP) approaches
\citep{venkatakrishnan2013plug,arridge2019solving,zhang2017learning}. Related to
our manifold assumptions, \cite{block2020fast} show that in a setting similar to
PnP approaches, the corresponding Langevin dynamics enjoys fast convergence
rates if the target distribution is supported on a manifold with curvature
assumptions. In particular, they show that a noisy version of the target
distribution satisfies a logarithmic Sobolev inequality with constant which only
depends on the intrisic dimension of the manifold.


\section{Proof of \Cref{thm:convergence_general}}
\label{sec:proof_theorem_one}

In this section, we present a proof of \Cref{thm:convergence_general}. More
precisely, we control each term on the right hand side of
\eqref{eq:decomposition}. The bottleneck of the proof resides in the control of
the discretization and approximation error
$\wassersteinD[1](\pi_\infty \Rker_{K}, \pi_\infty \Qker_{t_K})$ which is
dealt with in \Cref{sec:discretization_approximation_error}. Then, we turn to
the convergence of the backward process
$\wassersteinD[1](\pi_\infty \Qker_{t_K}, \pi \Pker_{T - t_K})$ in
\Cref{sec:convergence_backward}. Finally, we control the noising error
$\wassersteinD[1](\pi \Pker_{T - t_K}, \pi)$ and conclude in
\Cref{sec:noising_error_conclusion}. Technical results are postponed to the
appendix.

\subsection{Control of $\wassersteinD[1](\pi_\infty \Rker_{K}, \pi_\infty \Qker_{t_K})$}
\label{sec:discretization_approximation_error}

In this section, we control
$\wassersteinD[1](\pi_\infty \Rker_{K}, \pi_\infty \Qker_{t_K})$. To do so we
are going to use the backward formula introduced in \citet{del2019backward}.
First, we recall the definition of the stochastic flows
$(\bfY_{s,t}^x)_{s, t \in \ccint{0,T}}$ and the interpolation of its
discretization $(\bbfY_{s,t}^x)_{s, t \in \ccint{0,T}}$, for any $x \in \rset^d$
and $s,t \in \ccint{0,T}$ with $t \geq s$
\begin{equation}
  \textstyle{
    \rmd \bfY_{s,t}^x = \beta_{T-t} \{ \bfY_{s,t}^x + 2 \nabla \log p_{T-t}(\bfY_{s,t}^x) \} \rmd t + \sqrt{2 \beta_{T-t}} \rmd \bfB_t \eqsp , \qquad \bfY_{s,s}^x = x \eqsp ,
    }
\end{equation}
and for any $k \in \{0, \dots, K\}$ and $t \in \coint{s_k, t_{k+1}}$
\begin{equation}
  \textstyle{
    \rmd \bbfY_{s,t}^x = \beta_{T-t} \{ \bbfY_{s, t}^x + 2 \bm{s}(T-s_k, \bbfY_{s, s_k}^x) \} \rmd t + \sqrt{2 \beta_{T-t}} \rmd \bfB_t  \eqsp , \qquad \bbfY_{s,s}^x = x \eqsp ,
    }
  \end{equation}
  where $s_k = \max(s, t_k)$.  We also introduce the tangent process $(\bfY_{s,t}^x)_{t \in \ccint{s, T}}$
  \begin{equation}
      \label{eq:tangent_process}
  \textstyle{
    \rmd \nabla \bfY_{s,t}^x = \beta_{T-t} \{ \Id + 2 \nabla^2 \log p_{T-t}(\bfY_{s,t}^x) \} \nabla \bfY_{s,t}^x \rmd t  \eqsp , \qquad \nabla \bfY_{s,s}^x = \Id \eqsp .
  }
\end{equation}
Note that $(\bfY_{s,t}^x)_{t \in \ccint{s, T}}$ is a $d \times d$ stochastic
process.  The tangent process $(\nabla \bfY_{s,t}^x)_{s,t \in \ccint{0, T}}$ can
also be defined as follows. Under mild regularity assumption, for any
$s, t \in \ccint{0,T}$ with $t\geq s$, $x \mapsto \bfY_{s,t}^x$ is a
diffeomorphism, see \citep{kunita1981decomposition}, and we denote
$x \mapsto \nabla \bfY_{s,t}^x$ its differential. Then, \cite[Section
2]{kunita1981decomposition} shows under mild assumptions that
$(\nabla \bfY_{s,t}^x)_{s,t \in \ccint{0, T}}$ satisfies
\eqref{eq:tangent_process}. Hence,
$(\nabla \bfY_{s,t}^x)_{s,t \in \ccint{0, T}}$ encodes the local variation of
the process $(\bfY_{s,t}^x)_{s,t \in \ccint{0, T}}$ \wrt its initial condition.
Our bound on the approximation/discretization error relies on the following
proposition which was first proven by \citet{del2019backward}.
\begin{proposition}
  \label{prop:pierre-extended}
  Assume \rref{assum:manifold_hyp}. Then, for any $s, t \in \coint{0,T}$ with $s<t$ and $x \in \rset^d$
  \begin{equation}
\textstyle{    \bfY_{s,t}^x - \bbfY_{s,t}^x = \int_s^t (\nabla \bfY_{u,t}^{\bbfY_{s,u}^x})^\top \Delta b_u((\bbfY_{s,v}^x)_{v\in \ccint{s,T}}) \rmd u } \eqsp , 
\end{equation}
where for any $u \in \coint{0,T}$ with $u \in \coint{s_k, t_{k+1}}$ for
some $k \in \{0, \dots, K\}$ and $(\omega_v)_{v \in \ccint{s,T}} \in \rmc(\ccint{s,T}, \rset^d)$ we have 
  \begin{align}
    &b_u(\omega) = \beta_{T-u}( \omega_u + 2 \nabla \log p_{T-u}(\omega_u)) \eqsp , \quad \bar{b}_u(\omega) = \beta_{T-u}(\omega_{u} + 2 \bm{s}(T -s_k, \omega_{s_k})) \eqsp , \quad \Delta b_u(\omega) = b_u(\omega) - \bar{b}_u(\omega) \eqsp .
  \end{align}
  where $s_k = \max(s, t_k)$.
\end{proposition}

\begin{proof}
  The proof of this proposition is postponed to
  \Cref{sec:stoch-interp-form}.
\end{proof}

Using \Cref{prop:pierre-extended} our goal is now to control
$\normLigne{\nabla \bfY_{s,t}^x}$ and
$\normLigne{\Delta b_s((\bbfY_{s,t}^x)_{t \in \ccint{s,T}})}$ for any
$s, t \in \ccint{0,T}$ and $x \in \rset^d$.  To do so, we introduce the time
$t^\star$ which is a lower bound on the supremum time so that the backward
process is contractive on $\ccint{0, t^\star}$, 
\begin{equation}
  \label{eq:def_t_star}
  t^\star =  T - 2\bar{\beta}(1 + \log(1+\diam(\M))) \eqsp . 
\end{equation}
We then obtain the following bound.

\begin{proposition}
  \label{prop:control_gradient}
  Assume \rref{assum:manifold_hyp} and
  $T \geq 2\bar{\beta}(1 + \log(1+\diam(\M))$.  Let
  $t_K \in \coint{0, T}$. Then, for any $s\in \ccint{0,t_K}$ and
  $x \in \rset^d$ we have
  \begin{equation}
    \textstyle{
      \normLigne{\nabla \bfY_{s,t_K}^x} \leq \exp[-(1/2) \int_{T-t^\star}^{T-s} \beta_u \rmd u \1_{\coint{0,t^\star}}(s)] \exp[(\diam(\M)^2/2)\sigma_{T-t_K}^{-2}] \eqsp .
      }
    \end{equation}
  \end{proposition}

  \begin{proof}
    Let $x\in \rset^d$. First, using \eqref{eq:tangent_process} and \Cref{lemma:control_hessian} we have that for any $s, t\in \ccint{0,T}$ with $s \leq t$
    \begin{equation}
      \rmd \normLigne{\nabla \bfY_{s,t}^x}^2 \leq 2\beta_{T-t}(\normLigne{\nabla \bfY_{s,t}^x }^2 - 2(1  -  m_{T-t}^2 \diam(\M)^2 / (2\sigma_{T-t}^2) )/\sigma_{T-t}^2 \normLigne{\nabla \bfY_{s,t}^x}^2) \rmd t  \eqsp . 
    \end{equation}
    First, assume that $s \leq t^\star$ and that $t \geq t^\star$. In that case, using
    \Cref{lemma:growth_tangent_process} we have that
      \begin{equation}
\textstyle{    \int_s^{t^\star}\beta_{T-u} (1 -2/\sigma_{T-u}^2 + m_{T-u}^2\diam(\M)^2/\sigma_{T-u}^4) \rmd u  \leq -(1/2) \int_s^{t^\star} \beta_{T-u} \rmd u  \eqsp . }
\end{equation}
Therefore, using that result and the fact that $\nabla \bfY_{s,s}^x = \Id$, we get that
\begin{equation}\label{eq:inter_s_lower}
  \textstyle{\normLigne{\nabla \bfY_{s,t^\star}^x} \leq \exp[-(1/2)\int_{T-t^\star}^{T-s} \beta_u \rmd u] \eqsp . }
\end{equation}
In addition, using \Cref{lemma:growth_tangent_process} we have that
\begin{equation}
    \textstyle{    \int_{t^\star}^{t}\beta_{T-u} (1 -2/\sigma_{T-u}^2 + m_{T-u}^2\diam(\M)^2/\sigma_{T-u}^4) \rmd u} 
    \leq (\diam(\M)^2/2)(\sigma_{T-t}^{-2} - \sigma_{T-t^\star}^{-2}) \eqsp . 
  \end{equation}
  Therefore, we get that
\begin{equation}
  \textstyle{\normLigne{\nabla \bfY_{s,t}^x} \leq \exp[(\diam(\M)^2/2)\sigma_{T-t}^{-2}] \normLigne{\nabla \bfY_{s,t^\star}^x} \eqsp . }
\end{equation}
  Hence, combining this result and \eqref{eq:inter_s_lower}, in the case where $s \leq t^\star$ we have
  \begin{equation}
          \textstyle{\normLigne{\nabla \bfY_{s,t}^x} \leq \exp[-(1/2) \int_{T-t^\star}^{T-s} \beta_s \rmd s ] \exp[(\diam(\M)^2/2)\sigma_{T-t}^{-2}] \eqsp .}
        \end{equation}
        The proof in the cases where $s \geq t^\star$, $t \geq t^\star$ and
        $s \leq t^\star$, $t \leq t^\star$ are similar and left to the reader.
  \end{proof}

  Our next goal is to control $\normLigne{\Delta b}$. We recall that
  $b, \bar{b}: \ \ccint{0,T} \times \rmc(\ccint{0,T}, \rset^d) \to \rset^d$
  where for any $u \in \coint{0,T}$ such that $u \in \coint{s_k, t_{k+1}}$ for
  some $k \in \{0, \dots, K\}$ and
  $\omega = (\omega_v)_{v \in \ccint{s,T}} \in \rmc(\ccint{s,T}, \rset^d)$
  \footnote{With a slight abuse of notation we assume that each process on
    $\rmc(\ccint{s, T})$ is extended on $\rmc(\ccint{0, T})$ by setting
    $\omega_u = \omega_s$ for any $u \in \ccint{0,s}$. } we have
  \begin{align}
    &b_u(\omega) = \beta_{T-u}( \omega_u + 2 \nabla \log p_{T-u}(\omega_u)) \eqsp , \qquad \bar{b}_u(\omega) = \beta_{T-u}(\omega_{u} + 2 \bm{s}(T -s_k, \omega_{s_k})) \eqsp , \\
    &\Delta b_u(\omega) = b_u(\omega) - \bar{b}_u(\omega) \eqsp ,
  \end{align}
  where $s_k = \max(s, t_k)$.
  We now provide upper bounds on $\Delta b$.  We introduce the intermediate
  drift functions $\drifta, \driftb, \driftc, \driftd$ such that $\drifta = b$
  and $\driftd = \bar{b}$. In addition, for any $s, u \in \coint{0,T}$ such that
  $u \geq s$, $u \in \coint{s_k, t_{k+1}}$ for some $k \in \{0, \dots, K\}$ and
  for any
  $\omega = (\omega_v)_{v \in \ccint{s,T}} \in \rmc(\ccint{s,T}, \rset^d)$ we
  have
  \begin{align}
    &\driftb_u(\omega) = \beta_{T-u}( \omega_u + 2 \nabla \log p_{T-s_k}(\omega_{u})) \eqsp , \qquad \driftc_u(\omega) = \beta_{T-u}(\omega_{u} + 2 \nabla \log p_{T-s_k}(\omega_{s_k})) \eqsp , \\
    &\Deltaab = \drifta - \driftb \eqsp , \qquad \Deltabc = \driftb - \driftc \eqsp , \qquad \Deltacd = \driftc - \driftd  \eqsp ,
  \end{align}
  where $s_k = \max(s, t_k)$.  We have that
\begin{equation}
  \label{eq:ineq_delta}
  \normLigne{\Delta b} \leq \normLigne{\Deltaab} + \normLigne{\Deltabc} + \normLigne{\Deltacd} \eqsp .
\end{equation}
In the rest of this section, we control each term on the right hand side of
\eqref{eq:ineq_delta}.

\begin{lemma}
  \label{lemma:controlab}
  For any $s, u \in \coint{0,T}$ such that $u \geq s$,
  $u \in \coint{s_k, t_{k+1}}$ for some $k \in \{0, \dots, K\}$ and
  $\omega = (\omega_v)_{v \in \ccint{s,T}} \in \rmc(\ccint{s,T}, \rset^d)$ we
  have
\begin{equation}
  \textstyle{
    \normLigne{\Deltaab_u(\omega)}  \leq 2  \sup_{v \in \ccint{T-u, T-t_k}} (\beta_{v}^2 /\sigma_{v}^6) (2 + \diam(\M)^2)(\diam(\M) + \normLigne{\omega_u}) \gamma_k  \eqsp .
    }
\end{equation}
\end{lemma}

\begin{proof}
  Assume that $s \leq t_k$. Then, we have
  \begin{align}
    \normLigne{\Deltaab_u(\omega)} &\leq 2 \beta_{T-u} \normLigne{\nabla \log p_{T-u}(\omega_{u}) - \nabla \log p_{T-t_k}(\omega_{u})} \\
   &\textstyle{\leq 2 \beta_{T-u} \gamma_k \sup_{v \in \ccint{T-u, T-t_k}} \normLigne{\partial_v \nabla \log p_{T-v}(\omega_u)}}  \eqsp . 
\end{align}
Using \Cref{lemma:control_time_space_der}, we have that
\begin{equation}
\textstyle{  \normLigne{\Deltaab_u(\omega)} \leq 2  \beta_{T-u} \sup_{v \in \ccint{T-u, T-t_k}} (\beta_{v} /\sigma_{v}^6) (2 + \diam(\M)^2)(\diam(\M) + \normLigne{\omega_{u}}) \gamma_k \eqsp ,}
\end{equation}
which concludes the proof in the case where $s \leq t_k$. The case where
$s \geq t_k$ is similar and left to the reader.
\end{proof}

\begin{lemma}
  \label{lemma:controlbc}
  For any $s, u \in \coint{0,T}$ such that $u \geq s$,
  $u \in \coint{s_k, t_{k+1}}$ for some $k \in \{0, \dots, K\}$ and
  $\omega = (\omega_v)_{v \in \ccint{s,T}} \in \rmc(\ccint{s,T}, \rset^d)$ we have
  \begin{equation}
    \normLigne{\Deltabc_u(\omega)} \leq 2 (\beta_{T-u} /\sigma_{T-u}^4)(1 + \diam(\M)^2) \normLigne{\omega_u - \omega_{s_k}} \eqsp ,
  \end{equation}
  where $s_k = \max(s, t_k)$.
\end{lemma}

\begin{proof}
  Assume that $s \leq t_k$. We have
  \begin{align}
    \textstyle{
      \normLigne{\Deltabc_u(\omega)}} &\leq \textstyle{2 \beta_{T-u} \normLigne{\nabla \log p_{T-t_k}(\omega_{t_k}) - \nabla \log p_{T-t_k}(\omega_{u})}} \\
    &\leq \textstyle{ 2 \beta_{T-u} \sup_{v \in \ccint{u, T-t_k}} \normLigne{\nabla^2 \log p_{T-t_k}(\omega_v)} \normLigne{\omega_{u} - \omega_{t_k}} \eqsp .
      }
  \end{align}
  Using \Cref{lemma:control_hessian} we have that
  \begin{equation}
    \normLigne{\Deltabc_u(\omega)} \leq 2 (\beta_{T-u} /\sigma_{T-u}^4)(1 + \diam(\M)^2) \normLigne{\omega_u - \omega_{t_k}} \eqsp ,
  \end{equation}
  which concludes the proof in the case where $s \leq t_k$. The case where
$s \geq t_k$ is similar and left to the reader.
\end{proof}

Finally, combining \Cref{lemma:controlab}, \Cref{lemma:controlbc} and
\Cref{assum:score_control} in \eqref{eq:ineq_delta}, we get that for any
$s, u \in \coint{0,T}$ such that $u \geq s$, $u \in \coint{s_k, t_{k+1}}$ for
some $k \in \{0, \dots, K\}$ and
$(\omega_v)_{v \in \ccint{s,T}} \in \rmc(\ccint{s,T}, \rset^d)$ we have
\begin{align}
  \label{eq:control_b}
  \normLigne{\Delta b_u(\omega)} &\leq \textstyle{ 2  \sup_{v \in \ccint{T-t_{k+1}, T-t_k}} (\beta_{v}^2 /\sigma_{v}^6) (2 + \diam(\M)^2)(\diam(\M) + \normLigne{\omega_{u}})\gamma_k } \\
                                 & \qquad \qquad + 2 (\beta_{T-u} /\sigma_{T-u}^4)(1 + \diam(\M)^2) \normLigne{\omega_u - \omega_{s_k}} \\
  & \qquad\qquad +2 \beta_{T-u} \Mtt (1 + \normLigne{\omega_{s_k}})/\sigma_{T-u}^2 \eqsp ,
\end{align}
where $s_k = \max(s, t_k)$.

The following proposition controls the local error
between the continuous-time backward process and the interpolation of the
discretized one where the true score is replaced by the approximation $\bm{s}$.

\begin{proposition}
  \label{prop:local-error-control}
  Assume \rref{assum:manifold_hyp}, \rref{assum:assumption_beta},
  \rref{assum:score_control}, \rref{assum:step_size}. In addition, assume that
  $\delta, \Mtt, \gamma_K \leq 1/32$.  Then, we have for any $s, u \in \ccint{0,t_K}$ with
  $u \geq s$
  \begin{equation}
    \expeLigne{\normLigne{\Delta b_u((\bbfY_{s,v})_{v \in \ccint{s,T}})}} \leq \Ctt_0 (T-t_K+\bar{\beta})^2 (\Mtt + \delta^{1/2}) /(T-t_K)^2  \eqsp ,
\end{equation}
where $\bbfY_{s,s} \sim \mathrm{N}(0, \Id)$ and
\begin{equation}
  \label{eq:Ctt0_def}
  \Ctt_0 = (1 + \bar{\beta})^{7/2}(4 + 256 d + 43664 (1+\diam(\M))^4) \eqsp . 
\end{equation}
\end{proposition}

\begin{proof}
  Let $s, u \in \ccint{0,t_K}$ with $u \geq s$. In what follows, for ease of
  notation, we denote for any $k \in \{0, \dots, K\}$
\begin{equation}
 \textstyle{\kappa_k = \sup_{v \in \ccint{T-t_{k+1}, T-t_k}} \beta_{v} /
\sigma_{v}^2 \eqsp . } 
\end{equation}
  There exists $k \in \{0, \dots, K-1\}$ such that $u \in \ccint{t_k,
    t_{k+1}}$. Assume that $s \leq t_k$.  Recall that using
  \eqref{eq:control_b}, we have that for any
  $\omega = (\omega_v)_{v \in \ccint{s,T}} \in \rmc(\ccint{s,T}, \rset^d)$
\begin{align}
  \label{eq:control_b_proof}
  \normLigne{\Delta b_u(\omega)} &\leq \textstyle{ 2   \sup_{v \in \ccint{T-t_{k+1}, T-t_k}} (\beta_{v}^2 /\sigma_{v}^6) (2 + \diam(\M)^2)(\diam(\M) + \normLigne{\omega_{u}})\gamma_k } \\
                                 & \qquad \qquad + 2 (\beta_{T-u} /\sigma_{T-u}^4)(1 + \diam(\M)^2) \normLigne{\omega_u - \omega_{t_k}} \\
                                 & \qquad \qquad +2 \beta_{T-u} \Mtt (1 + \normLigne{\omega_{s_k}})/\sigma_{T-u}^2 \\
&\leq \textstyle{ 2  (\kappa_k^2/\sigma_{T-t_{k+1}}^2)\gamma_k (2 + \diam(\M)^2)(\diam(\M) + \normLigne{\omega_{u}}) } \\
                                 & \qquad \qquad + 2 \kappa_k^2 (1 + \diam(\M)^2)  \normLigne{\omega_u - \omega_{t_k}}/\beta_{T-u}  +2 \kappa_k \Mtt (1 + \normLigne{\omega_{s_k}})\eqsp .   
\end{align}
Combining this result with \Cref{lemma:control_growth_discrete_process}, its
following remark and \Cref{lemma:lipschitz-behavior}, we get that
\begin{align}
  \expeLigne{\normLigne{\Delta b_u((\bbfY_{s,v})_{v \in \ccint{s,T}})}} &\leq  2  (\kappa_k^2/\sigma_{T-t_{k+1}}^2)\gamma_k (2 + \diam(\M)^2)(\diam(\M) + \Ktt_0^{1/2}) \\
  &+ 2 \kappa_k^2 (1 + \diam(\M)^2) \Ltt_0^{1/2} \bar{\beta}^{3/2} \gamma_k^{1/2} + 2\kappa_k \Mtt (1+\Ktt_0^{1/2}) \eqsp . 
\end{align}
Denoting
$\Ctt = 2 (2 + \diam(\M)^2)(\diam(\M) + \Ktt_0^{1/2}) + 2 \Ltt_0^{1/2} \bar{\beta}^{3/2} (1 + \diam(\M)^2) + 2 (1 +
\Ktt_0^{1/2})$, we get that
\begin{equation}
  \expeLigne{\normLigne{\Delta b_u((\bbfY_{s,v})_{v \in \ccint{s,T}})}} \leq \Ctt ( (\kappa_k^2/\sigma_{T-t_{k+1}}^2)\gamma_k + \kappa_k^2 \gamma_k^{1/2} + \Mtt \kappa_k) \eqsp .
\end{equation}
Combining this result, \rref{assum:step_size} and \Cref{lemma:kappa_control} we have
  \begin{align}
    &\expeLigne{\normLigne{\Delta b_u((\bbfY_{s,v})_{v \in \ccint{s,T}})}} \leq \Ctt (1 + \bar{\beta})^2 (1+\bar{\beta}/(T-t_K))^2 (\delta^{1/2} + \Mtt ) \\
    & \qquad \qquad  + \Ctt (1 + \bar{\beta}) (1+\bar{\beta}/(T-t_K)) (\delta/\sigma_{T-t_K}^2) \\
                                                                          &\qquad \leq \Ctt (1 + \bar{\beta})^2 (T-t_K+\bar{\beta})^2 (  \delta^{1/2} + \Mtt )/(T-t_K)^2 + \Ctt (1 + \bar{\beta}) (T-t_K+\bar{\beta}) (\delta/\sigma_{T-t_K}^2)/(T-t_K) \eqsp . 
\end{align}
Finally, using \Cref{lemma:bound_sigma_t}, we have
$\sigma_{T-t_K}^{-2} \textstyle{= (1 - \exp[-2\int_0^{T-t_K}\beta_s \rmd s
  ])^{-1}} \leq 1 + \bar{\beta}/(2(T-t_K))$ Therefore, using that
$\gamma_K = T - t_K < 1$ we get that
  \begin{align}
    \expeLigne{\normLigne{\Delta b_u((\bbfY_{s,v})_{v \in \ccint{s,T}})}}
                                    &\leq \Ctt (1 + \bar{\beta})^2  (T-t_K+\bar{\beta})^2 ( \delta  + \delta^{1/2} + \Mtt )/(T-t_K)^2 \\
    &\leq 2 \Ctt (1 + \bar{\beta})^2  (T-t_K+\bar{\beta})^2 ( \delta^{1/2} + \Mtt )/(T-t_K)^2
\end{align}
which concludes the first part of the proof in the case where $s \leq t_k$.  The
same bound holds in the case where $s \geq t_k$. Finally, we conclude upon
noticing that $2 \Ctt (1 + \bar{\beta})^2 \leq \Ctt_0$ with $\Ctt_0$ given by
\eqref{eq:Ctt0_def}.
\end{proof}

We are now ready to control the global error between the backward process and
the interpolation of the associated discrete-time process where the true score
has been replaced by its approximation $\bm{s}$.

\begin{proposition}
  \label{prop:discretization_bound_final}
  Assume \rref{assum:manifold_hyp}, \rref{assum:assumption_beta},
  \rref{assum:score_control}, \rref{assum:step_size} and $\gamma_K =
  \vareps$. In addition, assume that $\vareps, \delta, \Mtt \leq 1/32$.  Then
  \begin{equation}
      \wassersteinD[1](\pi_\infty \rmQ_{t_K}, \pi_\infty \rmR_{K})  \leq  \Dtt_0 \exp[\diam(\M)^2(1+\bar{\beta})/(2\vareps)] (\Mtt + \delta^{1/2})/ \vareps^2  \eqsp ,
  \end{equation}
  where $\Dtt_0 = (1 + \bar{\beta})^7(8 + 512 d + 87328 (1 + \diam(\M))^4) (1 + \log(1 + \diam(\M)))$.
\end{proposition}

\begin{proof}
  Using \Cref{prop:pierre-extended}, we have
    \begin{equation}
\textstyle{    \normLigne{\bfY_{t_K} - Y_K} = \normLigne{\bfY_{t_K} - \bbfY_{t_K}} \leq \int_0^{t_K} \normLigne{\nabla \bfY_{u,t_K}^{\bbfY_{0,u}}} \normLigne{ \Delta b_u((\bbfY_{0,v})_{v\in \ccint{0,T}})} \rmd u } \eqsp .
\end{equation}
Combining this result, recalling that $t^\star$ is defined in
\eqref{eq:def_t_star} and \Cref{prop:control_gradient}, we get
    \begin{align}
      \textstyle{    \normLigne{\bfY_{t_K} - Y_K}} \leq & \textstyle{ \int_0^{t_K} \exp[-(1/2) \int_{T-t^\star}^{T-u} \beta_s \rmd s \1_{\coint{0,t^\star}}(u)] \exp[(\diam(\M)^2/2)\sigma_{T-t_K}^{-2}]  \normLigne{ \Delta b_u((\bbfY_{0,v})_{v\in \ccint{0,T}})} \rmd u } \\
                                                        &\leq \textstyle{ \exp[(\diam(\M)^2/2)\sigma_{T-t_K}^{-2}] (\int_{0}^{t^\star} \exp[-(1/2) \int_{T-t^\star}^{T-u} \beta_s \rmd s] \normLigne{ \Delta b_u((\bbfY_{0,v})_{v\in \ccint{0,T}})} \rmd u} \\
      & \qquad \qquad \textstyle{+ \int_{t^\star}^{t_K} \normLigne{ \Delta b_u((\bbfY_{0,v})_{v\in \ccint{0,T}})} \rmd u ) \eqsp .  }
\end{align}
Using this result and \Cref{prop:local-error-control} we get
\begin{align}
  &\wassersteinD[1](\pi_\infty \rmQ_{t_K}, \pi_\infty \rmR_{K}) \leq \expeLigne{\normLigne{\bfY_{t_K} - Y_K}} \\
  &\qquad \leq \textstyle{ \exp[(\diam(\M)^2/2)\sigma_{T-t_K}^{-2}] (\int_{0}^{t^\star} \exp[-(1/2) \int_{T-t^\star}^{T-u} \beta_s \rmd s] \expeLigne{\normLigne{ \Delta b_u((\bbfY_{0,v})_{v\in \ccint{0,T}})}} \rmd u} \\
      & \qquad \qquad \textstyle{+ \int_{t^\star}^{t_K} \expeLigne{\normLigne{ \Delta b_u((\bbfY_{0,v})_{v\in \ccint{0,T}})}} \rmd u ) \eqsp .  } \\
  &\qquad \leq \textstyle{ \exp[(\diam(\M)^2/2)\sigma_{T-t_K}^{-2}] \Ctt_0 (T - t_K + \bar{\beta})^2 (\Mtt + \delta^{1/2}) /(T-t_K)^2 } \\
  & \qquad \qquad \textstyle{ \times (\int_0^{t^\star} \exp[-(1/2) \int_{T-t^\star}^{T-u} \beta_s \rmd s]  \rmd u + t_K - t^\star)}  \eqsp . \label{eq:bound_wasserstein_intermediate}
\end{align}
We have that
\begin{equation}
  \textstyle{ \int_0^{t^\star} \exp[-(1/2) \int_{T-t^\star}^{T-u} \beta_s \rmd s]  \rmd u \leq \int_0^{t^\star} \exp[-(t^\star -u)/(2\bar{\beta})]  \rmd u \leq 2 \bar{\beta} \eqsp . } \label{eq:bound_integration}
\end{equation}
In addition, using \eqref{eq:def_t_star}  we have
\begin{equation}
  t_K - t^\star = T - \vareps - T +2 \bar{\beta} (1 + \log(1 + \diam(\M))) \leq 2 \bar{\beta} (1 + \log(1 + \diam(\M)))  \eqsp . \label{eq:bound_tk_tstar}
\end{equation}
Using \Cref{lemma:bound_sigma_t}, we have that
$\sigma_{T-t_K}^{-2} \leq (1 + \bar{\beta})/\vareps$. Combining this result,
\eqref{eq:bound_integration} and \eqref{eq:bound_tk_tstar} in
\eqref{eq:bound_wasserstein_intermediate} we get
\begin{equation}
  \wassersteinD[1](\pi_\infty \rmQ_{t_K}, \pi_\infty \rmR_{K}) \leq 2 \Ctt_0 \exp[\diam(\M)^2(1+\bar{\beta})/(2\vareps)](1 + \bar{\beta})^3 (1 + \log(1 + \diam(\M))) (\Mtt + \delta^{1/2})/ \vareps^2 \eqsp ,
\end{equation}
which concludes the proof.
\end{proof}

\subsection{Control of $\wassersteinD[1](\pi_\infty \Qker_{t_K}, \pi \Pker_{T - t_K})$}
\label{sec:convergence_backward}
In this section, we focus on the error
$\wassersteinD[1](\pi_\infty \Qker_{t_K}, \pi \Pker_{T-t_K})$. First, note that
$\pi \Pker_{T-t_K} = \pi \Pker_{T} \Qker_{t_K}$. Therefore, using
\Cref{prop:wasserstein_forward}, we have
\begin{equation}
\label{eq:convergence_bound_backward_inter}  
  \wassersteinD[1](\pi_\infty \Qker_{t_K}, \pi \Pker_{T-t_K}) = \wassersteinD[1](\pi_\infty \Qker_{t_K}, \pi \Pker_{T} \Qker_{t_K}) \leq \exp[(1/2)\sigma_{T-t_K}^{-2}] \wassersteinD[1](\pi \Pker_T, \pi_\infty) \eqsp . 
\end{equation}
To control $\wassersteinD[1](\pi \Pker_T, \pi_\infty)$, we use a
synchronous coupling, i.e. we set $(\bfY_t, \bfZ_t)_{t \in \ccint{0,T}}$ such that
\begin{equation}
  \rmd \bfY_t = - \beta_t \bfY_t \rmd t + \sqrt{2 \beta_t} \rmd \bfB_t \eqsp , \qquad 
  \rmd \bfZ_t = - \beta_t \bfZ_t \rmd t + \sqrt{2 \beta_t} \rmd \bfB_t \eqsp ,
\end{equation}
where $(\bfB_t)_{t \in \ccint{0,T}}$ is a $d$-dimensional Brownian motion and
$\bfY_0 \sim \pi$, $\bfZ_0 \sim \pi_\infty$. We have that for any
$t \in \ccint{0,T}$, $\bfZ_t \sim \pi_\infty$. In addition, denoting
$u_t = \expeLigne{\normLigne{\bfY_t - \bfZ_t}}$ for any $t \in \ccint{0,T}$, we
have that $u_t \leq u_0 \exp[-\int_0^t \beta_s \rmd s]$. Therefore, combining
this result and \eqref{eq:convergence_bound_backward_inter}, we get that
\begin{equation}
  \label{eq:convergence_bound_backward}
  \textstyle{
    \wassersteinD[1](\pi_\infty \Qker_{t_K}, \pi \Pker_{T-t_K}) \leq \exp[(1/2)\sigma_{T-t_K}^{-2}] \exp[-\int_0^T \beta_t \rmd t] \wassersteinD[1](\pi, \pi_\infty) \eqsp .
    }
  \end{equation}
Therefore, using \Cref{lemma:bound_sigma_t}, we have
\begin{equation}
  \textstyle{
    \wassersteinD[1](\pi_\infty \Qker_{t_K}, \pi \Pker_{T-t_K}) \leq \exp[(1 + \bar{\beta})\diam(\M)^2/(2\vareps)] \exp[-T/\bar{\beta}](\sqrt{d} + \diam(\M)) \eqsp . }
\end{equation}

\subsection{Control of $\wassersteinD[1](\pi \Pker_{T - t_K}, \pi)$ and conclusion}
\label{sec:noising_error_conclusion}

In this section, we focus on the error $\wassersteinD[1](\pi, \pi\Pker_{T-t_K})$
and conclude the proof.  We have that
$\wassersteinD[1](\pi, \pi\Pker_{T-t_K}) \leq \expeLigne{\normLigne{X -
    m_{T-t_K} X + \sigma_{T-t_K} Z}}$,
with $X \sim \pi$ and $Z \sim \mathrm{N}(0, \Id)$. Hence, using $1 - m_{T-t_K} \leq \sigma_{T-t_K}$, we have 
\begin{equation}
  \textstyle{
    \wassersteinD[1](\pi, \pi\Pker_{T-t_K}) \leq \diam(\M) (1 - m_{T-t_K}) + \sigma_{T-t_K} \sqrt{d} \leq (\diam(\M) + \sqrt{d}) \sigma_{T-t_K} \eqsp .
    }
  \end{equation}
  Using \Cref{lemma:bound_sigma_t} and this result we have
  \begin{equation}
    \label{eq:control_noising}
  \textstyle{
    \wassersteinD[1](\pi, \pi\Pker_{T-t_K}) \leq (2 \bar{\beta})^{1/2} (\diam(\M) + \sqrt{d}) \vareps^{1/2} \eqsp .
  }
  \end{equation}
  We conclude the proof upon combining this result,
  \eqref{eq:convergence_bound_backward} and
  \Cref{prop:discretization_bound_final} in \eqref{eq:decomposition}
 

\section{Conclusion}
\label{sec:conclusion}

In this work, we have studied the convergence of diffusion models under the
manifold hypothesis and provided convergence guarantees w.r.t. the Wasserstein
distance of order one. Our theoretical results show that diffusion models are
able to recover target distributions defined on low-dimensional manifolds. One
current limitation of our results lies in the dependency w.r.t. $1/\vareps$
which is exponential in the general case and might be overly pessimistic. This
dependency can be improved at the cost of imposing conditions on the Hessian of
$\log p_t$ but further investigations are needed to establish similar results in
realistic settings.

Our results can be extended in several directions. First, in this work we
focused on the Ornstein--Uhlenbeck process as a forward noising process. It would
be interesting to analyze other forward diffusions such as the critically-damped
one \citep{dockhorn2021score}. Another extension would be to study other
discretization frameworks such as predictor-corrector schemes
\citep{durkan2021maximum} and to extend our analysis to more realistic
statistical settings. Finally, it is a challenge to derive similar bounds for
target distributions with $\rset^d$ support and tail constraints. 

Finally, we would like to deepen our study of the relationship between the
geometry of the manifold $\M$ and the properties of the score
function. Preliminary results from \Cref{sec:non-conv-count} indicate that the
convexity of $\M$ can be recovered from the properties of the score but it
remains unclear if more can be said on the geometry of the manifold.



\section*{Acknowledgements}
\label{sec:acknowledgement}

We thank Arnaud Doucet, \'Emile Mathieu and James Thornton for providing
feedback on an early version of the paper. We thank George Deligiannidis, Alain
Durmus and \'Eric Moulines for useful discussions. Finally, we are indebted to
Pierre Del Moral who pointed us toward his work on stochastic interpolation
formulae. This work has been supported by The Alan Turing Institute through the
Theory and Methods Challenge Fortnights event “Accelerating generative models
and nonconvex optimisation”, which took place on 6-10 June 2022 and 5-9 Sep 2022
at The Alan Turing Institute headquarters.


\bibliographystyle{apalike}
\bibliography{./bibliography.bbl}

\begin{thebibliography}{}

\bibitem[Absil et~al., 2013]{absil2013extrinsic}
Absil, P.-A., Mahony, R., and Trumpf, J. (2013).
\newblock An extrinsic look at the {R}iemannian {H}essian.
\newblock In {\em International conference on geometric science of
  information}, pages 361--368. Springer.

\bibitem[Alaa et~al., 2022]{alaa2022faithful}
Alaa, A., van Breugel, B., Saveliev, E.~S., and van~der Schaar, M. (2022).
\newblock How faithful is your synthetic data? sample-level metrics for
  evaluating and auditing generative models.
\newblock In {\em International Conference on Machine Learning, {ICML} 2022},
  volume 162 of {\em {PMLR}}, pages 290--306. {PMLR}.

\bibitem[Alekseev, 1961]{alekseev1961estimate}
Alekseev, V.~M. (1961).
\newblock An estimate for the perturbations of the solutions of ordinary
  differential equations.
\newblock {\em Vestn. Mosk. Univ. Ser. I. Math. Mekh}, 2:28--36.

\bibitem[Arridge et~al., 2019]{arridge2019solving}
Arridge, S., Maass, P., {\"O}ktem, O., and Sch{\"o}nlieb, C.-B. (2019).
\newblock Solving inverse problems using data-driven models.
\newblock {\em Acta Numerica}, 28:1--174.

\bibitem[Bishop, 1974]{bishop1974infinitesimal}
Bishop, R.~L. (1974).
\newblock Infinitesimal convexity implies local convexity.
\newblock {\em Indiana Univ. Math. J}, 24(169-172):75.

\bibitem[Bishop and Crittenden, 2011]{bishop2011geometry}
Bishop, R.~L. and Crittenden, R.~J. (2011).
\newblock {\em Geometry of manifolds}.
\newblock Academic press.

\bibitem[Block et~al., 2020a]{block2020generative}
Block, A., Mroueh, Y., and Rakhlin, A. (2020a).
\newblock Generative modeling with denoising auto-encoders and {L}angevin
  sampling.
\newblock {\em arXiv preprint arXiv:2002.00107}.

\bibitem[Block et~al., 2020b]{block2020fast}
Block, A., Mroueh, Y., Rakhlin, A., and Ross, J. (2020b).
\newblock Fast mixing of multi-scale {L}angevin dynamics under the manifold
  hypothesis.
\newblock {\em arXiv preprint arXiv:2006.11166}.

\bibitem[Brown et~al., 2022]{brown2022union}
Brown, B.~C., Caterini, A.~L., Ross, B.~L., Cresswell, J.~C., and Loaiza-Ganem,
  G. (2022).
\newblock The union of manifolds hypothesis and its implications for deep
  generative modelling.
\newblock {\em arXiv preprint arXiv:2207.02862}.

\bibitem[Bundt, 1934]{bundt1934chebyshev}
Bundt, L. (1934).
\newblock {\em Bijdrage tot de theorie der konvekse puntverzamelingen}.
\newblock PhD thesis, University of Groningen, Amsterdam.

\bibitem[Cattiaux et~al., 2021]{cattiaux2021time}
Cattiaux, P., Conforti, G., Gentil, I., and L{\'e}onard, C. (2021).
\newblock Time reversal of diffusion processes under a finite entropy
  condition.
\newblock {\em arXiv preprint arXiv:2104.07708}.

\bibitem[Combet, 2006]{combet2006integrales}
Combet, E. (2006).
\newblock {\em Int{\'e}grales exponentielles: d{\'e}veloppements asymptotiques,
  propri{\'e}t{\'e}s lagrangiennes}, volume 937.
\newblock Springer.

\bibitem[Dalalyan, 2017]{dalalyan2017theoretical}
Dalalyan, A.~S. (2017).
\newblock Theoretical guarantees for approximate sampling from smooth and
  log-concave densities.
\newblock {\em J. R. Stat. Soc. Ser. B. Stat. Methodol.}, 79(3):651--676.

\bibitem[De~Bortoli et~al., 2021a]{de2021simulating}
De~Bortoli, V., Doucet, A., Heng, J., and Thornton, J. (2021a).
\newblock Simulating diffusion bridges with score matching.
\newblock {\em arXiv preprint arXiv:2111.07243}.

\bibitem[De~Bortoli et~al., 2022]{de2022riemannian}
De~Bortoli, V., Mathieu, E., Hutchinson, M., Thornton, J., Teh, Y.~W., and
  Doucet, A. (2022).
\newblock Riemannian score-based generative modeling.
\newblock {\em arXiv preprint arXiv:2202.02763}.

\bibitem[De~Bortoli et~al., 2021b]{debortoli2021diffusion}
De~Bortoli, V., Thornton, J., Heng, J., and Doucet, A. (2021b).
\newblock Diffusion {S}chr{\"o}dinger bridge with applications to score-based
  generative modeling.
\newblock {\em Advances in Neural Information Processing Systems}, 34.

\bibitem[Del~Moral and Singh, 2019]{del2019backward}
Del~Moral, P. and Singh, S.~S. (2019).
\newblock Backward {I}to-{V}entzell and stochastic interpolation formulae.
\newblock {\em arXiv preprint arXiv:1906.09145}.

\bibitem[Dhariwal and Nichol, 2021]{nichol2021beatgans}
Dhariwal, P. and Nichol, A.~Q. (2021).
\newblock Diffusion models beat {GANs} on image synthesis.
\newblock In {\em Advances in Neural Information Processing Systems}, pages
  8780--8794.

\bibitem[Dockhorn et~al., 2021]{dockhorn2021score}
Dockhorn, T., Vahdat, A., and Kreis, K. (2021).
\newblock Score-based generative modeling with critically-damped {L}angevin
  diffusion.
\newblock {\em arXiv preprint arXiv:2112.07068}.

\bibitem[Durmus and Moulines, 2017]{durmus2017nonasymp}
Durmus, A. and Moulines, E. (2017).
\newblock Nonasymptotic convergence analysis for the unadjusted {L}angevin
  algorithm.
\newblock {\em The Annals of Applied Probability}, 27(3):1551--1587.

\bibitem[Fefferman et~al., 2016]{fefferman2016testing}
Fefferman, C., Mitter, S., and Narayanan, H. (2016).
\newblock Testing the manifold hypothesis.
\newblock {\em Journal of the American Mathematical Society}, 29(4):983--1049.

\bibitem[Fontaine et~al., 2021]{fontaine2021convergence}
Fontaine, X., De~Bortoli, V., and Durmus, A. (2021).
\newblock Convergence rates and approximation results for sgd and its
  continuous-time counterpart.
\newblock In {\em Conference on Learning Theory}, pages 1965--2058. PMLR.

\bibitem[Franzese et~al., 2022]{franzese2022much}
Franzese, G., Rossi, S., Yang, L., Finamore, A., Rossi, D., Filippone, M., and
  Michiardi, P. (2022).
\newblock How much is enough? a study on diffusion times in score-based
  generative models.
\newblock {\em arXiv preprint arXiv:2206.05173}.

\bibitem[Goodfellow et~al., 2016]{goodfellow2016deep}
Goodfellow, I., Bengio, Y., and Courville, A. (2016).
\newblock {\em Deep Learning}.
\newblock MIT press.

\bibitem[Haussmann and Pardoux, 1986]{haussmann1986time}
Haussmann, U.~G. and Pardoux, E. (1986).
\newblock Time reversal of diffusions.
\newblock {\em The Annals of Probability}, 14(4):1188--1205.

\bibitem[Ho et~al., 2020]{ho2020denoising}
Ho, J., Jain, A., and Abbeel, P. (2020).
\newblock Denoising diffusion probabilistic models.
\newblock In {\em Advances in Neural Information Processing Systems}.

\bibitem[Huang et~al., 2021]{huang2021variational}
Huang, C.-W., Lim, J.~H., and Courville, A.~C. (2021).
\newblock A variational perspective on diffusion-based generative models and
  score matching.
\newblock {\em Advances in Neural Information Processing Systems},
  34:22863--22876.

\bibitem[Hyv{\"a}rinen, 2005]{hyvarinen2005estimation}
Hyv{\"a}rinen, A. (2005).
\newblock Estimation of non-normalized statistical models by score matching.
\newblock {\em Journal of Machine Learning Research}, 6(4).

\bibitem[Jolicoeur-Martineau et~al., 2021]{jolicoeur2020adversarial}
Jolicoeur-Martineau, A., Pich{\'e}-Taillefer, R., Tachet~des Combes, R., and
  Mitliagkas, I. (2021).
\newblock Adversarial score matching and improved sampling for image
  generation.
\newblock {\em International Conference on Learning Representations}.

\bibitem[Kim et~al., 2022]{kim2021soft}
Kim, D., Shin, S., Song, K., Kang, W., and Moon, I. (2022).
\newblock Soft truncation: {A} universal training technique of score-based
  diffusion model for high precision score estimation.
\newblock In {\em International Conference on Machine Learning}, volume 162 of
  {\em PMLR}, pages 11201--11228.

\bibitem[Kritikos, 1938]{kritikos1938quelques}
Kritikos, N. (1938).
\newblock Sur quelques propri{\'e}t{\'e}s des ensembles convexes.
\newblock {\em Bulletin math{\'e}matique de la Soci{\'e}t{\'e} Roumaine des
  Sciences}, 40(1/2):87--92.

\bibitem[Kunita, 1981]{kunita1981decomposition}
Kunita, H. (1981).
\newblock On the decomposition of solutions of stochastic differential
  equations.
\newblock In {\em Stochastic integrals ({P}roc. {S}ympos., {U}niv. {D}urham,
  {D}urham, 1980)}, volume 851 of {\em Lecture Notes in Math.}, pages 213--255.
  Springer, Berlin-New York.

\bibitem[Le~Loi and Phien, 2014]{le2014numerical}
Le~Loi, T. and Phien, P. (2014).
\newblock A numerical approach to some basic theorems in singularity theory.
\newblock {\em Mathematische Nachrichten}, 287(7):764--781.

\bibitem[Lee et~al., 2022]{lee2022convergence}
Lee, H., Lu, J., and Tan, Y. (2022).
\newblock Convergence for score-based generative modeling with polynomial
  complexity.
\newblock {\em arXiv preprint arXiv:2206.06227}.

\bibitem[Matsumoto, 2002]{matsumoto2002introduction}
Matsumoto, Y. (2002).
\newblock {\em An introduction to Morse theory}, volume 208.
\newblock American Mathematical Soc.

\bibitem[Motzkin, 1935]{motzkin1935quelques}
Motzkin, T.~S. (1935).
\newblock {\em Sur quelques propri{\'e}t{\'e}s caract{\'e}ristiques des
  ensembles born{\'e}s non convexes}.
\newblock Bardi.

\bibitem[Nichol and Dhariwal, 2021]{nichol2021improved}
Nichol, A.~Q. and Dhariwal, P. (2021).
\newblock Improved denoising diffusion probabilistic models.
\newblock In Meila, M. and Zhang, T., editors, {\em International Conference on
  Machine Learning}, volume 139 of {\em PMLR}, pages 8162--8171.

\bibitem[Pidstrigach, 2022]{pidstrigach2022score}
Pidstrigach, J. (2022).
\newblock Score-based generative models detect manifolds.
\newblock {\em arXiv preprint arXiv:2206.01018}.

\bibitem[Song et~al., 2021a]{durkan2021maximum}
Song, Y., Durkan, C., Murray, I., and Ermon, S. (2021a).
\newblock Maximum likelihood training of score-based diffusion models.
\newblock In {\em Advances in Neural Information Processing Systems},
  volume~34, pages 1415--1428. Curran Associates, Inc.

\bibitem[Song and Ermon, 2019]{song2019generative}
Song, Y. and Ermon, S. (2019).
\newblock Generative modeling by estimating gradients of the data distribution.
\newblock In {\em Advances in Neural Information Processing Systems}, pages
  11895--11907.

\bibitem[Song and Ermon, 2020]{song2020improved}
Song, Y. and Ermon, S. (2020).
\newblock Improved techniques for training score-based generative models.
\newblock In {\em Advances in Neural Information Processing Systems}.

\bibitem[Song et~al., 2021b]{song2020score}
Song, Y., Sohl{-}Dickstein, J., Kingma, D.~P., Kumar, A., Ermon, S., and Poole,
  B. (2021b).
\newblock Score-based generative modeling through stochastic differential
  equations.
\newblock In {\em International Conference on Learning Representations}.

\bibitem[Sriperumbudur et~al., 2009]{sriperumbudur2009ipm}
Sriperumbudur, B.~K., Gretton, A., Fukumizu, K., Lanckriet, G. R.~G., and
  Sch{\"{o}}lkopf, B. (2009).
\newblock A note on integral probability metrics and {$\phi$}-divergences.

\bibitem[Tenenbaum et~al., 2000]{tenenbaum2000global}
Tenenbaum, J.~B., Silva, V.~d., and Langford, J.~C. (2000).
\newblock A global geometric framework for nonlinear dimensionality reduction.
\newblock {\em {S}cience}, 290(5500):2319--2323.

\bibitem[Vahdat et~al., 2021]{vahdat2021score}
Vahdat, A., Kreis, K., and Kautz, J. (2021).
\newblock Score-based generative modeling in latent space.
\newblock {\em Advances in Neural Information Processing Systems},
  34:11287--11302.

\bibitem[Venkatakrishnan et~al., 2013]{venkatakrishnan2013plug}
Venkatakrishnan, S.~V., Bouman, C.~A., and Wohlberg, B. (2013).
\newblock Plug-and-play priors for model based reconstruction.
\newblock In {\em 2013 IEEE Global Conference on Signal and Information
  Processing}, pages 945--948. IEEE.

\bibitem[Vincent, 2011]{vincent2011connection}
Vincent, P. (2011).
\newblock A connection between score matching and denoising autoencoders.
\newblock {\em Neural Computation}, 23(7):1661--1674.

\bibitem[Weed and Bach, 2019]{weed2019sharp}
Weed, J. and Bach, F. (2019).
\newblock Sharp asymptotic and finite-sample rates of convergence of empirical
  measures in wasserstein distance.
\newblock {\em Bernoulli}, 25(4A):2620--2648.

\bibitem[Weinberger and Saul, 2006]{weinberger2006unsupervised}
Weinberger, K.~Q. and Saul, L.~K. (2006).
\newblock Unsupervised learning of image manifolds by semidefinite programming.
\newblock {\em International journal of computer vision}, 70(1):77--90.

\bibitem[Xiao et~al., 2021]{xiao2021tackling}
Xiao, Z., Kreis, K., and Vahdat, A. (2021).
\newblock Tackling the generative learning trilemma with denoising diffusion
  gans.
\newblock {\em arXiv preprint arXiv:2112.07804}.

\bibitem[Zhang et~al., 2017]{zhang2017learning}
Zhang, K., Zuo, W., Gu, S., and Zhang, L. (2017).
\newblock Learning deep cnn denoiser prior for image restoration.
\newblock In {\em Proceedings of the IEEE conference on computer vision and
  pattern recognition}, pages 3929--3938.

\bibitem[Zhang and Chen, 2022]{zhang2022exponential}
Zhang, Q. and Chen, Y. (2022).
\newblock Fast sampling of diffusion models with exponential integrator.
\newblock {\em arXiv preprint arXiv:2204.13902}.

\bibitem[Zhao et~al., 2018]{zhao2018bias}
Zhao, S., Ren, H., Yuan, A., Song, J., Goodman, N.~D., and Ermon, S. (2018).
\newblock Bias and generalization in deep generative models: An empirical
  study.
\newblock In Bengio, S., Wallach, H.~M., Larochelle, H., Grauman, K.,
  Cesa{-}Bianchi, N., and Garnett, R., editors, {\em Advances in Neural
  Information Processing Systems}, pages 10815--10824.

\end{thebibliography}

\appendix

\counterwithin{theorem}{section}
\counterwithin{lemma}{section}
\counterwithin{corollary}{section}
\counterwithin{proposition}{section}
\counterwithin{definition}{section}

\section{Organization of the appendix}
\label{sec:intro_appendix}

The appendix is organized as follows.  We start by discussing our discretization
scheme in \Cref{sec:discr-backw-proc}. In \Cref{sec:grad-hess-contr}, we provide
upper bounds on the gradient and Hessian of the logarithmic gradient of the
density of the forward process under the manifold assumption. In
\Cref{sec:contr-backw-proc}, we control the stability of several backward
processes. In \Cref{sec:stoch-interp-form}, we recall and adapt a stochastic
interpolation formula from \citet{del2019backward}. We check the different
assumptions on the noise schedule in \Cref{sec:assumptions-schedule}. A short
proof of the results of \citet{franzese2022much} is presented in
\Cref{sec:short-proof}. In \Cref{sec:wass-contr-under}, we present an extension
of our results in the case where error is controlled \wrt the $\mathrm{L}^2$
norm, following the work of \citet{lee2022convergence}. We improve on
\Cref{thm:convergence_general} in \Cref{sec:impr-bounds-under} under some
Hessian conditions.


\section{Discretization of backward processes}
\label{sec:discr-backw-proc}

In \Cref{sec:link-with-euler}, we briefly describe the links between our
proposed discretization and the classical Euler-Maruyama discretization. In
\Cref{sec:equivalence-with-}, we show that the discretization
\eqref{eq:discretization_improved} is associated to the one of
\cite{ho2020denoising} under specific settings

\subsection{Link with Euler-Maruyama discretization}
\label{sec:link-with-euler}

First, we recall the Euler-Maruyama discretization. Given a sequence of
stepsizes a discretization of \eqref{eq:approximate_time_reversal_continuous} is
given by the so-called Euler-Maruyama approximation, \ie we define for any
$k \in \{0, \dots, K\}$ and $t \in \ccint{t_k, t_{k+1}}$
  \begin{equation}
    \label{eq:classical_EM}
    \textstyle{
      \rmd \bbfY_t^{\EM} = \beta_{T-t_k} \{ \bbfY_{t_k}^{\EM} + 2 \bm{s}(T-t_k, \bbfY_{t_k}^{\EM}) \} \rmd t + \sqrt{2\beta_{T-t_k}} \rmd \bfB_t \eqsp , \qquad \bbfY_0^{\EM} \sim \pi_\infty \eqsp .
      }
\end{equation}
The associated discrete process $(Y_k^\EM)_{k \in \{0, \dots, K+1\}}$ is given for any
$k \in \{0, \dots, K+1\}$ by $Y_k^\EM = \bbfY_{t_k}^{\EM}$ and we have for any $k \in \{0, \dots, K\}$
\begin{equation}
  \label{eq:classical_EM_disc}
  \textstyle{
    Y_{k+1}^\EM = Y_k^\EM + \gamma_k  \beta_{T-t_k} \{ Y_k^{\EM} + 2 \bm{s}(T-t_k, Y_k^{\EM}) \} + \sqrt{2 \beta_{T-t_k} \gamma_k} Z_k \eqsp ,
    }
\end{equation}
where $\{Z_k\}_{k \in \nset}$ is a sequence of independent $d$-dimensional
Gaussian random variables with zero mean and identity covariance matrix.

Note that \eqref{eq:discretization_improved} describes the same update as
\eqref{eq:classical_EM} up to the first order \wrt $\gamma_k$. In practice,
there is no additional cost to replace the classical Euler-Maruyama
discretization with the discretization defined in
\eqref{eq:discretization_improved}, provided that the integral of the
temperature schedule $t \mapsto \beta_t$ can be computed in close form, which is
the case in all the cases considered experimentally, see
\Cref{sec:assumptions-schedule}.

However, in our theoretical analysis we found out that
\eqref{eq:discretization_improved} introduces less error than
\eqref{eq:classical_EM_disc} when compared to the approximate backward process
\eqref{eq:approximate_time_reversal_continuous}. In our study we only consider
the discretization scheme $(Y_k)_{k \in \{0, \dots, K+1\}}$ but emphasize that
our analysis could be readily extended to derive similar discretization errors
for the process $(Y_k^\EM)_{k \in \{0, \dots, K+1\}}$.


\subsection{Equivalence with \cite{ho2020denoising}}
\label{sec:equivalence-with-}

In this section, we show that the discretization scheme introduced in
\citet{ho2020denoising} and the one of \eqref{eq:discretization_improved} are
equivalent up to the first order in some parameter.

\paragraph{Setting of \cite{ho2020denoising}}
We start by recalling the setting of \citet{ho2020denoising}. Since, there is a
conflict between our notations and the ones of \citet{ho2020denoising}, we write
our constants in \rl{red} and the constants of \citet{ho2020denoising} in
\bl{blue}.  The forward process in \citet{ho2020denoising} is given for any
$\bl{t} \in \{1, \dots, \bl{T}\}$\footnote{Note that in \citet{ho2020denoising},
  $\bl{T}$ is a number of steps and not the total time of the forward.}
\begin{equation}
  \label{eq:reference_process}
  q(x_{\bl{t}}|x_0) = \mathrm{N}(x_{\bl{t}}; \sqrt{\bl{\bar{\alpha_t}}} x_0, (1 - \bl{\bar{\alpha}_t}) \Id) \eqsp ,
\end{equation}
and we define
\begin{equation}
  \label{eq:def_bar_alpha}
  \bl{\beta_t} = 1 - \bl{\alpha_t} \eqsp , \qquad \textstyle{\bl{\bar{\alpha}_t} = \prod_{\bl{s}=1}^{\bl{t}} \bl{\alpha_s} \eqsp .}
\end{equation}
In that case the loss function is given by
\begin{equation}
  \label{eq:loss_ddpm}
  \textstyle{
    \ell(\theta) = \sum_{\bl{t}=1}^{\bl{T}} \expeLigne{\normLigne{\eps_{\bl{t}} - \bm{\eps}_\theta(\sqrt{\bl{\bar{\alpha}_t}} x_0 + \sqrt{1 - \bl{\bar{\alpha}_t}} \eps_{\bl{t}}, \bl{t})}^2} \eqsp ,
    }
  \end{equation}  
  with $\{\eps_{\bl{t}}\}_{\bl{t}=1}^{\bl{T}}$ a collection of independent Gaussian random
  variables with zero mean and identity covariance matrix.  The backward
  sampling is given by the following recursion
  \begin{equation}
    \label{eq:backward_recursion}
    x_{\bl{t-1}} = \bl{\alpha_t}^{-1/2} (x_{\bl{t}} - (\bl{\beta_t}/\sqrt{1 - \bl{\bar{\alpha}_t}}) \bm{\eps}_\theta(x_{\bl{t}}, \bl{t})) + \bl{\beta_t} z_{\bl{t}} \eqsp ,
  \end{equation}
  with$\{z_{\bl{t}}\}_{\bl{t}=1}^{\bl{T}}$ a collection of independent Gaussian
  random variables with zero mean and identity covariance matrix\footnote{We
    consider the case where $\bl{\sigma_t} = \bl{\beta_t}$.}. Note that using
  these notations, there is a conflict of notation between the forward process
  and the backward process. To clarify our identification, we denote
  $y_{\bl{t}} = x_{\bl{t}}$ for any $t \in \{0, \dots, T\}$, with $x_{\bl{t}}$
  given by \eqref{eq:backward_recursion}, in what follows.

  \paragraph{Identification}
  In what follows, we set $\bl{t} = k+1$, $\bl{T} = K$ and for any
  $\bl{t} \in \{1, \dots, \bl{T}\}$
  \begin{equation}
    \textstyle{\bl{\alpha_t} = \exp[-2 \int_{\rl{T} - t_{K+1-k}}^{\rl{T} - t_{K+1-(k+1)}} \rl{\beta_s} \rmd s] = \exp[-2 \int_{\rl{T} - t_{K+1-k}}^{\rl{T} - t_{K-k}} \rl{\beta_s} \rmd s] \eqsp .}
  \end{equation}
  For instance, we have $\bl{\alpha_1} = \exp[-2\int_0^{\rl{T}-t_K} \rl{\beta_s} \rmd s ]$ and $\bl{\alpha_T} = \exp[-2\int_{\rl{T}-t_1}^{\rl{T}} \rl{\beta_s} \rmd s ]$. Note that in this case, using \eqref{eq:def_bar_alpha}, we have
  \begin{equation}
        \textstyle{\bl{\bar{\alpha}_t}^{1/2} = \exp[-\int_{0}^{\rl{T} - t_{K-k}} \rl{\beta_s} \rmd s] = \rl{m_{\rl{T}-t_{K-k}}} \eqsp .}
      \end{equation}
      Similarly, $\sqrt{1 - \bl{\bar{\alpha}_t}} = \rl{\sigma_{\rl{T} -t_{K-k}}}$.
      In what follows, we identify the distribution of the forward process
      \eqref{eq:reference_process} with the one of
      \eqref{eq:ornstein_ulhenbeck}, the loss function \eqref{eq:loss_ddpm} with
      the one of \eqref{eq:dsm_loss} and the time reversal
      \eqref{eq:backward_recursion} with the one of
      \eqref{eq:discretization_improved}.

      \begin{enumerate}[label=(\alph*)]
      \item The distribution $q(x_{\bl{t}}|x_0)$ given in
        \eqref{eq:reference_process} is the distribution of
        $\bfX_{\rl{T} - t_{K-k}} | \bfX_0$ where
        $(\bfX_{\rl{t}})_{\rl{t} \in \ccint{0,\rl{T}}}$ is given in
        \eqref{eq:ornstein_ulhenbeck}, since
        $\bfX_{\rl{T}-t_{K-k}} = \rl{m_{T-t_{K-k}}} \bfX_0 +
        \rl{\sigma_{T-t_{K-k}}} Z$ with $Z \sim \mathrm{N}(0, \Id)$. Therefore,
        we identify $x_{\bl{t}}$ and $\bfX_{\rl{T} - t_{K-k}}$ for any
        $\bl{t} \in \{1, \dots, \bl{T}\}$. Similarly, for any
        $\bl{t} \in \{1, \dots, \bl{T}\}$, we identify $\bl{t}$ and $\rl{T} - t_{K-k}$.
      \item Using that $\bfX_{\rl{t}} = \rl{m_t} \bfX_0 + \bfB_{\rl{\sigma_t}}$
        for any $\rl{t} \in \{0, \dots, \rl{T}\}$, the loss is given by
      \begin{align}
        \ell(\bm{s}) &= \textstyle{\int_0^{\rl{T}} \kappa(\rl{t}) \expeLigne{\normLigne{\bm{s}(\rl{t}, \bfX_{\rl{t}}) - \nabla \log p_{\rl{t}|0}(\bfX_{\rl{t}}|\bfX_0)}^2} \rmd t } \\
                     & = \textstyle{\int_0^{\rl{T}} \kappa(\rl{t}) \expeLigne{\normLigne{\bm{s}(\rl{t}, \bfX_{\rl{t}}) + \bfB_{\rl{\sigma_t}} /\rl{\sigma_{\rl{t}}}^2}^2} \rmd t } \\ 
                     & = \textstyle{\int_0^{\rl{T}} \kappa(\rl{t})/\rl{\sigma_{\rl{t}}}^2 \expeLigne{\normLigne{- \rl{\sigma_{\rl{t}}} \bm{s}(\rl{t}, \bfX_{\rl{t}}) - \bfB_{\rl{\sigma_t}} /\rl{\sigma_{\rl{t}}}}^2} \rmd t } \\
                                       & = \textstyle{\int_0^{\rl{T}} \kappa(\rl{T}-\rl{t})/\rl{\sigma_{\rl{T}-\rl{t}}}^2 \expeLigne{\normLigne{- \rl{\sigma_{\rl{T}-\rl{t}}} \bm{s}(\rl{T}-\rl{t}, \bfX_{\rl{T}-\rl{t}}) - \bfB_{\rl{\sigma_{T-t}}} /\rl{\sigma_{\rl{T}-\rl{t}}}}^2}} \rmd t \eqsp .
  \end{align}
  With a slight abuse of notation we assume that
  $\kappa(\rl{T}-\rl{t}) = \sum_{k=0}^{K} \updelta_{t_{K-k}}(\rl{t})
  \rl{\sigma_{T-t}}^2$ for any $\rl{t} \in \ccint{0,\rl{T}}$.  Hence, we get
  that
\begin{align}
  \ell(\bm{s}) &= \textstyle{ \sum_{k=0}^K \expeLigne{\normLigne{\eps_{\bl{t}} - (-\rl{\sigma_{\rl{T}-t_{K-k}}}\bm{s}(\rl{T}- t_{K-k}, \bfX_{\rl{T}-t_{K-k}}))}^2}} \\
               &= \textstyle{ \sum_{k=0}^K \expeLigne{\normLigne{\eps_{\bl{t}} - (-\sqrt{1 - \bl{\bar{\alpha_t}}} \bm{s}(\rl{T}- t_{K-k}, x_t))}^2}}\\
  &= \textstyle{ \sum_{\bl{t}=1}^{\bl{T}} \expeLigne{\normLigne{\eps_{\bl{t}} - (-\sqrt{1 - \bl{\bar{\alpha_t}}} \bm{s}(\bl{t}, \sqrt{\bl{\bar{\alpha}_t}} x_0 + \sqrt{1 - \bl{\bar{\alpha}_t}} \eps_{\bl{t}}))}^2}}
\end{align}
Hence, identifying $\bm{\eps}_\theta(\cdot, \bl{t})$ and
$-\sqrt{1 - \bl{\bar{\alpha_t}}} \bm{s}(\bl{t}, \cdot)$ for any
$\bl{t} \in \{1, \dots, \bl{T}\}$, we recover \eqref{eq:loss_ddpm}.
\item We now aim at recovering \eqref{eq:backward_recursion} from
  \eqref{eq:discretization_improved}. Using the change of variable $k \to K-k$
  and noting that for any $\rl{t} \in \ocint{0, \rl{T}}$,
  $\bfB_{\rl{\sigma_t}} / \rl{\sigma_t}$ is a Gaussian random variable with zero
  mean and identity covariance matrix we have
\begin{align}
  Y_{K-k+1} &=     \textstyle{Y_{K-k} + (\exp[\int_{\rl{T}-t_{K-k+1}}^{\rl{T}-t_{K-k}} \rl{\beta_s} \rmd s] - 1) (Y_{K-k} + 2 \bm{s}(\rl{T}-t_{K-k}, Y_{K-k}))} \\
            & \qquad \textstyle{+ (\exp[2 \int_{\rl{T}-t_{K-k+1}}^{\rl{T}-t_{K-k}} \rl{\beta_s} \rmd s] - 1)^{1/2} Z_{K-k} 
              } \\
            &=     \textstyle{Y_{K-k} + (\bl{\alpha_t}^{-1/2} - 1) (Y_{K-k} + 2 \bm{s}(\rl{T}-t_{K-k}, Y_{K-k})) + \sqrt{\bl{\beta_t}} Z_{K-k} 
              } \\
            &=     \textstyle{\bl{\alpha_t}^{-1/2} Y_{K-k} + 2 (\bl{\alpha_t}^{-1/2} - 1) \bm{s}(\rl{T}-t_{K-k}, Y_{K-k}) + \sqrt{\bl{\beta_t}} Z_{K-k} 
              }\\
            &=     \textstyle{\bl{\alpha_t}^{-1/2} Y_{K-k} - 2 ((\bl{\alpha_t}^{-1/2} - 1)/\sqrt{1-\bl{\bar{\alpha}_t}}) \bm{\eps}_\theta(Y_{K-k}, \bl{t}) + \sqrt{\bl{\beta_t}} Z_{K-k} \eqsp . 
              } \label{eq:almost_discrete}
\end{align}
Finally, since $\bl{\alpha_t} = 1 - \bl{\beta_t}$ we have
$\bl{\alpha_t}^{1/2} = 1 - \bl{\beta_t}/2 + o(\bl{\beta_t})$. This implies that
$2(\bl{\alpha_t}^{-1/2} -1) = -\bl{\beta_t}/\sqrt{\bl{\alpha_t}} +
o(\bl{\beta_t}/\sqrt{\bl{\alpha_t}})$. Therefore, combining this result and
\eqref{eq:almost_discrete}, we get that
\begin{equation}
  Y_{K-k+1} =     \textstyle{\bl{\alpha_t}^{-1/2} (Y_{K-k} + (\bl{\beta_t}/\sqrt{1-\bl{\bar{\alpha}_t}}) \bm{\eps}_\theta(\bl{t}, Y_{K-k})) + \sqrt{\bl{\beta_t}} Z_{K-k} + o(\bl{\beta_t}/\sqrt{\bl{\alpha_t}})\eqsp . }
\end{equation}
This corresponds to \eqref{eq:backward_recursion} up to a term of the form
$o(\bl{\beta_t}/\sqrt{\bl{\alpha_t}})$. 
\end{enumerate}


\section{Gradient and Hessian controls on the logarithmic density}
\label{sec:grad-hess-contr}

Let $\pi \in \Pens(\rset^d)$ be the target probability measure. We denote
$(p_t)_{t > 0}$ such that for any $t > 0$ the density \wrt the Lebesgue
measure of the distribution of $\bfX_t$ (with initialization
$\bfX_0^N \sim \pi$) is given by $p_t$. Similarly, $\pi^N \in \Pens(\rset^d)$ be
an empirical version of $\pi$, \ie  $\pi^N = (1/N) \sum_{k=1}^N X^k$, with
$\{X^k\}_{k=1}^N \sim \pi^{\otimes N}$. We denote $(p_t^N)_{t > 0}$ such that
for any $t > 0$ the density \wrt the Lebesgue measure of the distribution of
$\bfX_t^N$ (with initialization $\bfX_0^N \sim \pi^N$) is given by $p_t$.  In
order to show the stability and growth of the processes at hand we need to
control quantities related to the gradient and Hessian of $\log q_t$ where
$q_t = p_t$ or $p_t^N$. We first show a dissipativity condition on the
gradient. We recall that for any $t \in \ccint{0,T}$
\begin{equation}
  \textstyle{
    m_t = \exp[-\int_0^t\beta_s \rmd s] \eqsp , \qquad \sigma_t^2 = 1 - \exp[-2\int_0^t\beta_s \rmd s] \eqsp .
    }
  \end{equation}
  Such dissipativity conditions will allow us to control the moments of the
  introduced backward processes.
\begin{lemma}
  \label{lemma:control_grad}
    Assume \rref{assum:manifold_hyp}. Then for any $t \in \ocint{0,T}$ and $x_t \in \rset^d$ we have that
    \begin{equation}
      \langle \nabla \log q_t(x_t), x_t \rangle \leq -\norm{x_t}^2/\sigma_t^2 + m_t \diam(\M)\norm{x_t}/\sigma_t^2 \eqsp ,    \end{equation}
    with $q_t = p_t^N$ or $p_t$. In addition, we have
    \begin{equation}
      \label{eq:bound_norm}
      \normLigne{\nabla \log q_t(x_t)}^2 \leq 2\norm{x_t}^2/\sigma_t^4 + 2m_t^2 \diam(\M)^2/\sigma_t^4 \eqsp .
    \end{equation}
  \end{lemma}

  \begin{proof}
    Let $N \in \nset$. We have that for any $t \in \ccint{0,T}$ and $x_t \in \rset^d$
    \begin{equation}
      \textstyle{
        p_t^N(x_t) = (1/N) \sum_{k=1}^N \exp[-\normLigne{x_t - m_t X^k}^2/2\sigma_t^2] / (2 \uppi \sigma_t^2)^{d/2} \eqsp ,
        }
      \end{equation}
      Therefore, we get that for any $t \in \ccint{0,T}$ and $x_t \in \rset^d$
      \begin{align}
        \textstyle{
\nabla \log p_t^N(x_t) = (-1/N) \sum_{k=1}^N (x_t - m_t X^k) \exp[-\normLigne{x_t - m_t X^k}^2/2\sigma_t^2] / ((2 \uppi \sigma_t^2)^{d/2} \sigma_t^2 p_t^N(x_t)) \eqsp .} 
\end{align}
Hence, we have
\begin{equation}
  \langle \nabla \log p_t^N(x_t), x_t \rangle \leq -\norm{x_t}^2/\sigma_t^2 + m_t \diam(\M)\norm{x_t}/\sigma_t^2 \eqsp . 
\end{equation}
Therefore taking the limit $N \to +\infty$, the same conclusion holds for
$p_t$. The proof of \eqref{eq:bound_norm} follows the same lines and is left to
the reader.
\end{proof}

We now provide controls on the Hessian $\nabla^2 \log q_t$. Such bounds allow to
control the growth (or contraction) of the tangent process. This will also allow
us to control the growth (or contraction) of the distance between backward
processes \wrt the Wasserstein distance of order one.

  \begin{lemma}
    \label{lemma:control_hessian}
    Assume \rref{assum:manifold_hyp} then we have for any $t \in \ocint{0,T}$,
    $x_t \in \rset^d$ and $\mathrm{M} \in \mathcal{M}_d(\rset^d)$
  \begin{equation}
    \langle \mathrm{M}, \nabla^2 \log q_t(x_t) \mathrm{M} \rangle \leq - (1  -  m_t^2 \diam(\M)^2 / (2\sigma_t^2) )/\sigma_t^2 \normLigne{\mathrm{M}}^2 \eqsp .
  \end{equation}
  In addition, we have
  \begin{equation}    
    \normLigne{\nabla^2 \log q_t(x_t)} \leq (1 + \diam(\M)^2)/\sigma_t^4 \eqsp . 
  \end{equation}
  More generally, we have
  \begin{align}
    &\nabla^2 \log p_t(x_t) = -\Id/\sigma_t^2 \\
    &\qquad \textstyle{+ (2\sigma_t^4)^{-1} \int_{\M \times \M} (x_0 - x_0')^{\otimes 2} \exp[-\normLigne{x_t - m_t x_0}^2/(2\sigma_t^2)] \exp[-\normLigne{x_t - m_t x_0'}^2/(2\sigma_t^2)] \rmd \pi(x_0) \rmd \pi(x_0')} \\ 
    & \qquad \qquad \textstyle{/ (\int_\M \exp[-\normLigne{x_t - m_t x_0}^2/(2\sigma_t^2)] \rmd \pi(x_0))^2 \eqsp.}
  \end{align}
\end{lemma}

\begin{proof}
  Let $N \in \nset$. For any $t \in \ocint{0,T}$ and $x \in \rset^d$, we let
  $\bar{p}_t^N = p_t^N (2 \uppi \sigma_t^2)^{d/2} $ and we have
\begin{equation}
  \textstyle{
    \bar{p}_t^N(x) = (1/N) \sum_{k=1}^N \exp[-\normLigne{x - m_t X^k}^2/2\sigma_t^2] \eqsp ,
    }
  \end{equation}
  Hence, we have
  \begin{align}
    \textstyle{
    \nabla \log \bar{p}_t^N(x) = (-1/N) \sum_{k=1}^N (x - m_t X^k) \exp[-\normLigne{x - m_t X^k}^2/2\sigma_t^2] / (\sigma_t^2 \bar{p}_t^N(x)) \eqsp .
    }
\end{align}
Hence, we get that
\begin{align}
  &\textstyle{\nabla^2 \log \bar{p}_t^N(x) = - \Id / \sigma_t^2 }\\
  &\textstyle{\quad + (1/N) \sum_{k=1}^N (x - m_t X^k) \otimes (x - m_t X^k)  \exp[-\normLigne{x - m_t X^k}^2/2\sigma_t^2] / (\sigma_t^4 \bar{p}_t^N(x)) }\\
  &\quad \textstyle{- (1/N^2) (\sum_{k=1}^N (x - m_t X^k) \exp[-\normLigne{x - m_t X^k}^2/2\sigma_t^2]) }\\
  &\quad \qquad \textstyle{\otimes (\sum_{k=1}^N (x - m_t X^k) \exp[-\normLigne{x - m_t X^k}^2/2\sigma_t^2]) / (\sigma_t^2 \bar{p}_t^N(x))^2 } \eqsp . 
\end{align}
For any $k \in \{0, \dots, N-1\}$, denote $f_t^k = -(x - m_tX^k)/\sigma_t^2$
and $e_t^k = \exp[-\normLigne{f_t^k}^2]$. Using the previous result, we have
\begin{align}
  \nabla^2 \log \bar{p}_t^N(x) &\textstyle{= - \Id / \sigma_t^2 + \sum_{k=1}^N f_t^k \otimes f_t^k  e_t^k /  \sum_{k=1}^N e_t^k } \\
  & \qquad \textstyle{ - (\sum_{k=1}^N f_t^k  e_t^k /  \sum_{k=1}^N e_t^k) \otimes (\sum_{k=1}^N f_t^k  e_t^k /  \sum_{k=1}^N e_t^k) }\\
  &\textstyle{ = -\Id / \sigma_t^2 + (1/2) \sum_{j,k=1}^N (f_t^k - f_t^j) \otimes (f_t^k - f_t^j) e_t^k e_t^j / \sum_{k,j=1}^N e_t^k e_t^j \eqsp . }\label{eq:form_empirical_hessian}
\end{align}
In addition, using that for any $\ell \in \{1, \dots, N\}$, $X^\ell \in \M$ we have that
\begin{equation}
  \normLigne{f_t^k - f_t^j} = m_t \normLigne{X^k - X^j} / \sigma_t^2 \leq m_t \diam(\M) / \sigma_t^2 \eqsp . 
\end{equation}
Therefore, we get that
\begin{equation}
  \langle \mathrm{M}, \nabla^2 \log \bar{p}_t^N(x) \mathrm{M} \rangle \leq - (1  -  m_t^2 \diam(\M)^2 / (2\sigma_t^2) )/\sigma_t^2 \normLigne{\mathrm{M}}^2  \eqsp . 
\end{equation}
Using \eqref{eq:form_empirical_hessian}, the fact that $\M$ is compact and the
strong law of large numbers we have that
\begin{align}
  &\textstyle{ \lim_{N \to +\infty} \nabla^2 \log \bar{p}_t^N(x) = -\Id / \sigma_t^2 }\\
  &\textstyle{ + \int_{\rset^d} (x - m_t x_0) \otimes (x - m_t \bar{x}_0)  \exp[- \normLigne{x - m_t x_0}^2/(2\sigma_t^2) ] \exp[- \normLigne{x - m_t \bar{x}_0}^2/(2\sigma_t^2) ]  \rmd \pi(x_0) \rmd \pi(\bar{x}_0) }\\
  &\textstyle{ / (\int_{\rset^d} \exp[- \normLigne{x - m_t \bar{x}_0}^2/(2\sigma_t^2) ]  \rmd \pi(x_0))^2 \eqsp .}
\end{align}
Hence, we get that
$\lim_{N \to +\infty} \nabla^2 \log p_t^N(x) = \nabla^2 \log p_t$, which concludes the proof.
\end{proof}

Finally, in order to control the local error of the time discretization, we also
need to control the time derivative of the gradient, \ie
$\partial_t \nabla \log q_t$.

\begin{lemma}
  \label{lemma:control_time_space_der}
  Assume \rref{assum:manifold_hyp}. Then for any $t \in \ocint{0,T}$ and $x_t \in \rset^d$ we have that
  \begin{equation}
    \normLigne{\partial_t \nabla \log q_t(x_t)} \leq  (\beta_t /\sigma_t^6) (2 + \diam(\M)^2)(\diam(\M) + \normLigne{x})  \eqsp .
  \end{equation}
\end{lemma}

\begin{proof}
  Let $N \in \nset$ and $t \in \ocint{0,T}$. Recall that for any $x \in \rset^d$, $p_t^N(x) = \bar{p}_t^N(x)/(2\uppi \sigma_t^2)^{d/2}$ with 
  \begin{equation}
    \textstyle{\bar{p}_t^N(x) = (1/N) \sum_{k=1}^N e_t^k(x) \eqsp , \qquad e_t^k(x) = \exp[-\normLigne{x - m_t X^k}^2/(2\sigma_t^2)] \eqsp .}
  \end{equation}
  In what follows, we denote $f_t^k = \log e_t^k$ for any $k \in \{1, \dots,
  N\}$. For any $x \in \rset^d$ we have
  \begin{equation}
    \textstyle{ \partial_t \log \bar{p}_t^N(x) = \sum_{k=1}^N \partial_t f_t^k(x) e_t^k(x) / \sum_{k=1}^N e_t^k(x)  \eqsp .}
  \end{equation}
  Therefore, we have for any $x \in \rset^d$
  \begin{align}
    &\partial_t \nabla \log \bar{p}_t^N(x) = \textstyle{\sum_{k=1}^N \partial_t \nabla f_t^k(x) e_t^k(x) / \sum_{k=1}^N e_t^k(x) } \\
                                    & \qquad  \qquad \textstyle{ + \sum_{k=1}^N \partial_t f_t^k(x) \nabla f_t^k(x) e_t^k(x) / \sum_{k=1}^N e_t^k(x) } \\
                                    & \qquad  \qquad \textstyle{ - \sum_{k,j=1}^N \partial_t f_t^k(x) \nabla  f_t^j(x) e_t^k(x) e_t^j(x)  / \sum_{k,j=1}^N e_t^k(x)e_t^j(x) } \\
                                    & \qquad = \textstyle{\sum_{k=1}^N \partial_t \nabla f_t^k(x) e_t^k(x) / \sum_{k=1}^N e_t^k(x) } \\
                                    & \qquad  \qquad \textstyle{ + (1/2) \sum_{k,j=1}^N (\partial_t f_t^k(x) - \partial_t f_t^j(x)) (\nabla  f_t^k(x) - \nabla f_t^j(x)) e_t^k(x) e_t^j(x)  / \sum_{k,j=1}^N e_t^k(x)e_t^j(x) } \eqsp . \label{eq:general_decomposition}
  \end{align}
  In what follows, we fix $k, j \in \{1, \dots, N\}$ and provide upper bounds
  for $\absLigne{\partial_t f_t^k - \partial_t f_t^j}$,
  $\normLigne{\nabla f_t^k - \nabla f_t^j}$ and $\partial_t \nabla f_t^k$. First, we
  have that for any $x \in \rset^d$, $\nabla f_t^k(x) = -(x-m_t X^k)/\sigma_t^2$.
  Hence, using that $m_t \leq 1$, we get that for any $x\in \rset^d$
  \begin{equation}
    \label{eq:control_diff_grad_log}
    \normLigne{\nabla f_t^k(x) - \nabla f_t^j(x)} \leq m_t \diam(\M)/\sigma_t^2 \leq \diam(\M)/\sigma_t^2 \eqsp .
  \end{equation}
  In addition, we have that for any $x \in \rset^d$
  \begin{equation}
    \partial_t f_t^k(t) = \partial_t \sigma_t^2 /(2\sigma_t^4) \normLigne{x - m_t X^k}^2 + \partial_t m_t/\sigma_t^2 \langle X^k, x - m_t X^k \rangle \eqsp . 
  \end{equation}
  Combining this result, the fact that
  $\partial_t \sigma_t^2 = -2m_t \partial_t m_t$ and that
  $\partial_t m_t = -\beta_t m_t$, we get that
  \begin{align}
    \partial_t f_t^k(t) &= -\beta_t m_t/\sigma_t^2 [-(m_t/\sigma_t^2) \normLigne{x - m_t X^k}^2 + \langle x - m_t X^k, X^k \rangle] \\
                      &= -\beta_t m_t/\sigma_t^2 \langle x - m_t X^k, -(m_t /\sigma_t^2)(x- m_tX^k) + X^k \rangle \\
                        &= -\beta_t m_t/\sigma_t^4 \langle x - m_t X^k, - m_t x + X^k  \rangle \\
    &= \beta_t m_t/\sigma_t^4(m_t \normLigne{x}^2 + m_t \normLigne{X^k}^2 + (1+m_t^2)\langle x, X^k \rangle) \eqsp . \label{eq:partial_time_der}
  \end{align}
  Using this result and that $m_t \leq 1$, we have that for any $x \in \rset^d$
  \begin{align}
    \label{eq:control_diff_partial_log}
    \absLigne{\partial_t f_t^k(x) - \partial_t f_t^j(x)} &\leq 2 \beta_t m_t^2  \diam(\M)^2/\sigma_t^4 + \beta_t m_t (1 + m_t^2) \diam(\M)\normLigne{x}/\sigma_t^4 \\
    &\leq 2 (\beta_t/\sigma_t^4) \diam(\M) (\diam(\M)  + \normLigne{x})  \eqsp . 
  \end{align}
  Using \eqref{eq:partial_time_der}, we have for any $x \in \rset^d$
  \begin{equation}
    \nabla \partial_t f_t^k(x) = 2 \beta_t m_t^2 /\sigma_t^4 x +( \beta_t m_t/\sigma_t^4)(1+m_t^2)X^k \eqsp . 
  \end{equation}
  Therefore, combining this result and the fact that $m_t \leq 1$, we get that for any $x \in \rset^d$
  \begin{equation}
    \label{eq:control_cross_der}
    \normLigne{\partial_t \nabla f_t^k(x)} \leq 2 (\beta_t /\sigma_t^4)(\diam(\M) +  \normLigne{x}) \eqsp . 
  \end{equation}
  Combining \eqref{eq:control_diff_grad_log},
  \eqref{eq:control_diff_partial_log} and \eqref{eq:control_cross_der} in
  \eqref{eq:general_decomposition}, we get that for any $x \in \rset^d$
  \begin{align}
    \normLigne{\partial_t \nabla \log \bar{p}_t^N(x)} &\leq 2 (\beta_t /\sigma_t^4)(\diam(\M) +  \normLigne{x}) + (\beta_t/\sigma_t^6) \diam(\M)^2 (\diam(\M)  + \normLigne{x}) \\
    &\leq (\beta_t/\sigma_t^6) (2 + \diam(\M)^2)(\diam(\M)  + \normLigne{x}) \eqsp ,
  \end{align}
  which concludes the proof using that $\lim_{N \to +\infty} \partial_t \nabla \log p_t^N(x_t) = \partial_t \nabla \log p_t$.
\end{proof}

We conclude this section with bounds on the higher-order differentials of
$\log p_t$. To compute higher derivatives we will use the following lemma.
\begin{lemma}
  \label{sec:formal_derivative}
  Let $E = \{ e_i \}_{i=1}^M$ be a family of functions such that for any
  $i \in \{1, \dots, M\}$, $e_i \in \rmc^\infty(\rset^d, \rset)$. Similarly, let
  $G = \{ g_i \}_{i=1}^M$ be a family of functions such that for any
  $i \in \{1, \dots, M\}$, $g_i \in \rmc^\infty(\rset^d, \rset^p)$. Let $F(E,G)$
  such that for any $x \in \rset^d$
  \begin{equation}
    \textstyle{
      F(E,G) = \sum_{i=1}^M g_i e_i / \sum_{i=1}^M e_i \eqsp . 
      }
    \end{equation}
    Then, we have
    \begin{equation}
      \nabla F(E,G) = F(E, \nabla G) + (1/2) F(E \otimes E, (G \ominus G) \odot (\nabla \log E \ominus \nabla \log E)) \eqsp ,
    \end{equation}
    where $\otimes$ is the tensor product, $\odot$ the pointwise product and
    $\ominus$ the tensor substraction.
  \end{lemma}

  \begin{proof}
    We assume that $p =1$. The proof in the general case is similar and left to
    the reader.  We have that
    \begin{align}
      &\nabla F(E,G) = \textstyle{\sum_{i=1}^M \nabla g_i e_i / \sum_{i=1}^M e_i + \sum_{i,j=1}^M g_i \nabla \log(e_i) e_i e_j / \sum_{i,j=1}^M e_i e_j} \\
                    & \qquad \qquad - \textstyle{\sum_{i,j=1}^M  g_i \nabla \log(e_j) e_i e_j  / \sum_{i,j=1}^M e_i e_j} \\
                    &\qquad = \textstyle{\sum_{i=1}^M \nabla g_i e_i / \sum_{i=1}^M e_i} \\
      & \qquad \qquad \textstyle{+ (1/2) \sum_{i,j=1}^M (g_i \nabla \log(e_i) + g_j \nabla \log(e_j) - g_i \nabla \log(e_j) - g_j \nabla \log(e_i)) e_i e_j / \sum_{i,j=1}^M e_i e_j } \\
&\qquad = \textstyle{\sum_{i=1}^M \nabla g_i e_i / \sum_{i=1}^M e_i} \\
      & \qquad \qquad \textstyle{+ (1/2) \sum_{i,j=1}^M (g_i - g_j) (\nabla \log e_i - \nabla \log e_j) / \sum_{i,j=1}^M e_i e_j \eqsp ,}      
    \end{align}
    which concludes the proof.
  \end{proof}

\begin{lemma}
  \label{lemma:higher_order}
  Assume \rref{assum:manifold_hyp}. Then, there exists $C \geq 0$ such that for
  any $t \in \ocint{0,T}$ we have
  \begin{equation}
    \normLigne{\nabla^2 \log p_t(x)} + \normLigne{\nabla^3 \log p_t(x)} + \normLigne{\nabla^4 \log p_t(x)} \leq C / \sigma_t^8 \eqsp .
  \end{equation}
\end{lemma}

\begin{proof}
  Let $t \in \ocint{0,T}$. First, remark that for any $x \in \rset^d$
  \begin{equation}
    \nabla^2 \log p_t(x) = -\Id /\sigma_t^2 + F(E^{\otimes 2}, (\nabla \log E \ominus \nabla \log E)^{\odot 2}) \eqsp ,
  \end{equation}
  where $E = \{ e_i \}_{i=1}^N$ and for any $i \in \{1, \dots, N\}$,
  $e_i(x) = \exp[-\normLigne{x - m_t X^i}^2/(2\sigma_t^2)]$. Note that
  $\nabla \log E \ominus \nabla \log E$ does not depend on $x$ and there exists
  $C_0 \geq 0$ such that for any $i,j \in \{1, \dots, N\}$,
  $\normLigne{\nabla \log e_i(x) - \nabla \log e_j(x)} \leq C_0
  /\sigma_t^2$. Hence, using \Cref{sec:formal_derivative} we have
  \begin{equation}
    \nabla^3 \log p_t(x) = F(E^{\otimes 4}, (\nabla \log E \ominus \nabla \log E)^{\odot 2} \odot (\nabla \log(E \otimes E) \ominus \nabla \log(E \otimes E))) \eqsp . 
  \end{equation}
  Again, note that
  $G_1 = (\nabla \log E \ominus \nabla \log E)^{\odot 2} \odot (\nabla \log(E \otimes
  E) \ominus \nabla \log(E \otimes E))$ does not depend on $x$, upon remarking that
  \begin{equation}
    \nabla \log (E \otimes E) \ominus \nabla \log (E \otimes E) = (\nabla \log E \ominus \nabla \log E) \oplus (\nabla \log E \ominus \nabla \log E) \eqsp . 
  \end{equation}
  Finally, we have $\nabla^4 \log p_t(x) = F(E^{\otimes 8}, G_2)$, where
  \begin{align}
    G_2 &= [((\nabla \log E \ominus \nabla \log E)^{\odot 2} \odot (\nabla \log(E \otimes E) \ominus \nabla \log(E \otimes E))) \\
    & \qquad \ominus ((\nabla \log E \ominus \nabla \log E)^{\odot 2} \odot (\nabla \log(E \otimes E) \ominus \nabla \log(E \otimes E)))] \odot (\nabla \log(E^{\otimes 4}) \ominus \nabla \log(E^{\otimes 4})) 
        \eqsp . 
  \end{align}
  Therefore, we get that there exists $C \geq 0$ such that for any $x \in \rset^d$
  \begin{equation}
    \normLigne{\nabla \log p_t^3(x)} \leq C / \sigma_t^6 \eqsp , \qquad \normLigne{\nabla \log p_t^4(x)} \leq C / \sigma_t^8 \eqsp .
  \end{equation}
  We conclude the proof upon using that $\sigma_t \leq 1$.
\end{proof}


\section{Control of the backward processes}
\label{sec:contr-backw-proc}

We start by introducing the different processes in \Cref{sec:intr-proc}. We
gather a few technical results in \Cref{sec:some-usef-techn}. Then, we turn to
the stability and Lipschitz properties of the backward processes in
\Cref{sec:stab-lipsch-prop}. Finally, we control the growth of the backward
tangent process in \Cref{sec:contr-tang-backw}.

\subsection{Introduction of the processes}
\label{sec:intr-proc}
In this section, we study the stability of the backward process given by
\begin{equation}
  \label{eq:continuous_exact_app}
  \textstyle{
    \rmd \bfY_t = \beta_{T-t} \{ \bfY_t + 2 \nabla \log q_{T-t}(\bfY_t) \} \rmd t + \sqrt{2 \beta_{T-t}} \rmd \bfB_t \eqsp ,
    }
\end{equation}
where $q_t$ is either $p_t$ or
$p_t^N$.  We are also going to
consider the following approximate
continuous-time process
\begin{equation}
  \label{eq:continuous_approximate_app}
  \textstyle{
    \rmd \bhfY_t = \beta_{T-t} \{ \bhfY_t + 2 \bm{s}(T-t, \bhfY_t) \} \rmd t + \sqrt{2 \beta_{T-t} } \rmd \bfB_t \eqsp ,
    }
\end{equation}
where $\bm{s}(t, \cdot)$ is an approximation of either $p_t$ or $p_t^N$. Note
that since $q_t > 0$ and
$q \in \rmc^\infty(\ocint{0, T} \times \rset^d, \rset^d)$ and that
$\bm{s} \in \rmc^1(\ocint{0,T} \times \rset^d, \rset^d)$ we have that
\eqref{eq:continuous_exact_app} and \eqref{eq:continuous_approximate_app} admit strong
solutions up to an explosion time.  Finally, we also consider the following
interpolating process: for any $t \in \coint{t_k, t_{k+1}}$
\begin{equation}
  \label{eq:OU_disc_app}
  \textstyle{
    \rmd \bbfY_t = \beta_{T-t} \{ \bbfY_{t} + 2 \bm{s}(T-t_k, \bbfY_{t_k}) \} \rmd t + \sqrt{2\beta_{T-t}} \rmd \bfB_t \eqsp .
    }
\end{equation}
This process is an interpolation of a modified Euler--Maruyama discretization of
\eqref{eq:continuous_approximate_app}.  Note that the classical Euler--Maruyama
discretization would be associated with the following interpolation
\begin{equation}
  \label{eq:classical_EM_app}
  \textstyle{
    \rmd \bbfY_t^{\EM} = \beta_{T-t_k} \{ \bbfY_{t_k}^{\EM} + 2 \bm{s}(T-t_k, \bbfY_{t_k}^{\EM}) \} \rmd t + \sqrt{2\beta_{T-t_k}} \rmd \bfB_t \eqsp .
    }
\end{equation}
In \eqref{eq:OU_disc_app}, we take advantage of the linear part of the drift. Indeed
on the interval $\ccint{t_k, t_{k+1}}$, the process \eqref{eq:OU_disc_app} is a
simple Ornstein--Uhlenbeck which can be integrated explicitly. In particular for
any $k \in \{0, \dots, N-1\}$ and $t \in \ccint{t_k, t_{k+1}}$ we have
\begin{equation}
  \textstyle{
    \bbfY_t = \bbfY_{t_k} + (\exp[\int_{T-t}^{T-t_k} \beta_s \rmd s] - 1) (\bbfY_{t_k} + 2 \bm{s}(T-t_k, \bbfY_{t_k})) + (\exp[2 \int_{T-t}^{T-t_k} \beta_s \rmd s] - 1)^{1/2} Z \eqsp ,
    }
  \end{equation}
  where $Z$ is a Gaussian random variable with zero mean and identity covariance
  and the equality holds in distribution independent from
  $\bbfY_{t_k}$. Denoting $\{Y_k\}_{k \in \{0, \dots, N-1\}}$, we get that for
  any $k \in \{0, \dots, N-1\}$
  \begin{equation}
    \label{eq:discretization_improved_app}
    \textstyle{
      Y_{k+1} = Y_k + (\exp[\int_{T-t_{k+1}}^{T-t_k} \beta_s \rmd s] - 1) (Y_{k} + 2 \bm{s}(T-t_k, Y_{k})) + (\exp[2 \int_{T-t_{k+1}}^{T-t_k} \beta_s \rmd s] - 1)^{1/2} Z_k \eqsp ,
      }
  \end{equation}
  where $\{Z_k\}_{k \in \nset}$ is a collection of independent Gaussian random
  variables with zero mean and identity covariance matrix.  Using this scheme
  instead of the classical Euler--Maruyama simplifies the analysis of the
  discretization. Up to the first order this scheme is equal to the classical
  Euler--Maruyama discretization. Once again, we emphasize that computing this
  scheme is as expensive as computing the classical Euler--Maruyama
  discretization provided that the integral $\int_s^t \beta_u \rmd u$ are
  available in close form for all $s, t \in \ccint{0,T}$ which is the case for
  all the discretization schemes used in practice. We refer to
  \Cref{tab:processes} for a list of all processes used in the proof. In what
  follows, we control the stability of these processes.

  \subsection{Some useful technical lemmas}
  \label{sec:some-usef-techn}
We gather in this section some technical results.

\begin{lemma}
  \label{lemma:integration_lemma}
  For any $s, t \in \ccint{0,T}$ we have
  \begin{align}
    &\textstyle{
      \int_s^t \beta_{T-u}/\sigma_{T-u}^2 \rmd u  = [(-1/2)\log(\exp[2\int_0^{T-u} \beta_v \rmd v] - 1)]_s^t
      } \eqsp ,\label{eq:integral_log} \\
    & \textstyle{
      \int_s^t \beta_{T-u}m_{T-u}^2/\sigma_{T-u}^4 \rmd u  = [(1/2)/(1 - \exp[-2\int_0^{T-u} \beta_v \rmd v])]_s^t \eqsp . 
      } \label{eq:integral_inverse}
  \end{align}
  In particular, if $\beta = \beta_0$ then
  \begin{align}
    &\textstyle{
      \int_s^t \beta_{T-u}/\sigma_{T-u}^2 \rmd u  = [(-1/2)\log(\exp[2\beta_0(T-u)] - 1)]_s^t
      } \eqsp , \\
    & \textstyle{
      \int_s^t \beta_{T-u}m_{T-u}^2/\sigma_{T-u}^4 \rmd u  = [(1/2)/(1 - \exp[-2\beta_0(T-u)])]_s^t \eqsp . 
      }
  \end{align}  
\end{lemma}
\begin{proof}
We have
\begin{align}
  \textstyle{
    \int_s^t \beta_{T-u}/\sigma_{T-u}^2 \rmd u} &= \textstyle{\int_s^t \beta_{T-u} / (1 - \exp[-2\int_0^{T-u} \beta_v \rmd v]) \rmd u 
                                                  } \\
  &= \textstyle{\int_s^t \beta_{T-u} \exp[2\int_0^{T-u} \beta_v \rmd v] / (\exp[2\int_0^{T-u} \beta_v \rmd v] - 1) \rmd u 
    } \\
    &= \textstyle{(-1/2)\int_s^t \partial_u \log(\exp[2\int_0^{T-u} \beta_v \rmd v] - 1) \rmd u 
    } \eqsp .
\end{align}
This concludes the proof of \eqref{eq:integral_log}. We have
\begin{align}
  \textstyle{
  \int_s^t \beta_{T-u}m_{T-u}^2/\sigma_{T-u}^4 \rmd u}  &= \textstyle{\int_s^t \beta_{T-u}m_{T-u}^2/ (1 - \exp[-2\int_0^{T-u} \beta_v \rmd v])^2 \rmd u} \\
                                                        &= \textstyle{\int_s^t \beta_{T-u}\exp[-2\int_0^{T-u} \beta_v \rmd v]/ (1 - \exp[-2\int_0^{T-u} \beta_v \rmd v])^2 \rmd u} \\
                                                            &= \textstyle{(1/2) \int_s^t \partial_u (1 - \exp[-2\int_0^{T-u} \beta_v \rmd v])^{-1} \rmd u} \eqsp .
\end{align}
This concludes the proof of \eqref{eq:integral_inverse}.
\end{proof}

\begin{lemma}
  \label{lemma:bound_sigma_t}
  Assume \rref{assum:assumption_beta}. We have
  $\sigma_{t}^2 \leq 2 t \bar{\beta}$ and
  $\sigma_t^{-2} \leq 1 + \bar{\beta}/(2t) $.
\end{lemma}

\begin{proof}
  First, using that for any $a \geq 0$, $\exp[-a] \geq 1 - a$ we have
  \begin{equation}
    \textstyle{
      \sigma_t^2 = 1 - \exp[-2 \int_0^t \beta_s \rmd s] \leq 1 - \exp[-2\bar{\beta} t] \leq 2 \bar{\beta} t \eqsp .
      }
    \end{equation}
    Second, using that for any $a \geq 0$, $1/(1+\exp[-a]) \leq 1 +1/a$ we have 
  \begin{equation}
    \textstyle{
      \sigma_t^{-2} = (1 - \exp[-2 \int_0^t \beta_s \rmd s])^{-1} \leq 1 + (2\int_0^t \beta_s \rmd s)^{-1}     \eqsp ,
      }
    \end{equation}
    which concludes the proof.
  \end{proof}

  For any $k \in \{0, \dots, K\}$ we introduce $\kappa_k = \sup_{u \in \ccint{T-t_{k+1}, T-t_k}} \beta_u / \sigma_u^2$.
  \begin{lemma}
  \label{lemma:kappa_control}
  Assume \rref{assum:assumption_beta}.  Then, we have that for any
  $k \in \{0, \dots, N-1\}$
  \begin{equation}
    \kappa_k \leq \bar{\beta}(1 + \bar{\beta}/t) \eqsp . 
  \end{equation}
\end{lemma}

\begin{proof}
  Recall that for any $s \in \ccint{0,T}$,
  $\sigma_s^2 = 1 - \exp[-2\int_0^s \beta_u \rmd u]$. Using that for any
  $v \geq 0$, $(1 - \rme^{-2v})^{-1} \leq 1 + 1/(2v)$ we have for any
  $t \in \ocint{0,T}$
  \begin{equation}
    \textstyle{\beta_{t}/\sigma_t^2 \leq \beta_t + \beta_t (2\int_0^t \beta_s \rmd s )^{-1} \leq \beta_t (1 + \bar{\beta}/t)  \eqsp , }
  \end{equation}
  which concludes the proof.
\end{proof}

\subsection{Stability and Lipschitz properties of the backward processes}
\label{sec:stab-lipsch-prop}
The controls derived in \Cref{sec:grad-hess-contr} allow for uniform control of
the moments of the backward processes. Note that such Lyapunov techniques were
used in \citet{fontaine2021convergence} to control energy functionals in convex
optimization problems.  The following lemma is not used directly in our final
analysis but provides intuitive controls on the backward process.
  
  \begin{lemma}
    Assume \rref{assum:manifold_hyp}, \rref{assum:score_control} and that
    there exists $\eta >0$ such that
    $\Mtt + \eta \diam(\M) \leq 1/4$. Then, for any
    $t \in \ccint{0, T}$ we have
    \begin{equation}
      \expeLigne{\normLigne{\bhfY_t}^2} \leq d +  8(1 + \Mtt +  \diam(\M)/\eta)  \eqsp . 
    \end{equation}
    In particular if $\Mtt \leq 1/8$ then
for any
    $t \in \ccint{0, T}$ we have
    \begin{equation}
      \expeLigne{\normLigne{\bhfY_t}^2} \leq d +  8(1 + \Mtt +  8 \diam(\M)^2)  \eqsp . 
    \end{equation}    
  \end{lemma}

  \begin{proof}
    First, using \Cref{lemma:control_grad}, we have that for any
    $t \in \coint{0,T}$, $\expeLigne{\normLigne{\bhfY_t}^2} <+\infty$.  Hence,
    using It\^o's lemma, we have
    \begin{equation}
      \textstyle{
        \rmd (1/2) \normLigne{\bhfY_t}^2 =  \beta_{T-t} \{ \normLigne{\bhfY_t}^2 + 2\langle \bm{s}(T-t, \bhfY_t), \bhfY_t \rangle + 2 \} \rmd t  + \sqrt{2\beta_{T-t}} \langle \bhfY_t, \rmd \bfB_t \rangle \eqsp .
        }
    \end{equation}
    Therefore, using \rref{assum:score_control}, \Cref{lemma:control_grad} and
    that for any $a,b,\eta > 0$, $2ab \leq a^2\eta +b^2/\eta$ we get that for
    any $t \in \coint{0,T}$ we have that
    $(1/2) \expeLigne{\normLigne{\bhfY_t}^2} \leq u_t$, where $u_0 = d/2$ and
    \begin{align}
      \rmd u_t &= \beta_{T-t} (1 - 2/\sigma_{T-t}^2 + 2 (\Mtt + \eta m_{T-t} \diam(\M))/\sigma_{T-t}^2) u_t \\
      & \qquad \qquad + \beta_{T-t} (2 + 2\Mtt +  2m_{T-t} \diam(\M)/\eta)/\sigma_{T-t}^2 \\
               &= \beta_{T-t} (\sigma_{T-t}^2 - 2 + 2 (\Mtt + \eta m_{T-t} \diam(\M)))/\sigma_{T-t}^2 u_t \\
      & \qquad \qquad + \beta_{T-t} (2 + 2\Mtt +  2m_{T-t} \diam(\M)/\eta)/\sigma_{T-t}^2\eqsp . 
    \end{align}
    For any $t \in \ccint{0,T}$ and $x \in \rset$ define $F(t,x)$ given by 
    \begin{equation}
      F(t,x) = -(2 - \sigma_{T-t}^2 - 2 (\Mtt + m_{T-t} \eta \diam(\M)))/\sigma_{T-t}^2 x + (2 + 2\Mtt+2m_{T-t} \diam(\M)/\eta)/\sigma_{T-t}^2 \eqsp . 
    \end{equation}
    Using that for any $t \in \ccint{0,T}$, $m_t, \sigma_t^2 \in \ccint{0,1}$,
    we have that for any $t \in \ccint{0,T}$
    \begin{equation}
      2 - \sigma_{T-t}^2 - 2 (\Mtt + \eta m_{T-t} \diam(\M)) \geq 1 - 2 (\Mtt +  \eta\diam(\M)) \geq 1/2 \eqsp .
    \end{equation}
    Using \cite[Lemma 3]{fontaine2021convergence}
    we get that for any $t \in \ccint{0,T}$
    \begin{equation}
      u_t\leq d/2 + 2(1 + \Mtt +  \diam(\M)/\eta)/(1 - 2(\Mtt + \eta \diam(\M)))\leq d/2 + 4(1 + \Mtt +  \diam(\M)/\eta) \eqsp ,
    \end{equation}
    which concludes the proof.
  \end{proof}
  
\begin{lemma}
  \label{lemma:control_growth_discrete_process}
  Assume \rref{assum:manifold_hyp}, \rref{assum:assumption_beta} and
  \rref{assum:score_control}. Assume that there exists $\delta >0$ such that for
  any $k \in \{0, \dots, K\}$,
  $\gamma_k \beta_{T-t_k} / \sigma_{T-t_k}^2 \leq \delta$. Assume that there
  exists $\eta >0$ such that $A(\delta, \Mtt, \eta, \diam(\M)) > 0$ with
    \begin{align}
      &A(\delta, \Mtt, \eta, \diam(\M)) = 2 - 2\delta - 32 \delta (1
    + \Mtt^2) - 8 \Mtt -
    4\eta \diam(\M) \eqsp , \\
     & B(\delta, \Mtt, \eta, \diam(\M)) = 32\delta (\Mtt^2 +
    \diam(\M)^2) + 2 
    (1+\delta)(\diam(\M)/\eta + \Mtt) + 4d \eqsp . 
    \end{align}
  Then, we have for any $k \in \{0, \dots, K\}$
  \begin{equation}
    \label{eq:Ktt_def}
      \expeLigne{\normLigne{Y_k}^2} \leq \Ktt = d + B(\delta, \Mtt, \eta, \diam(\M))(1/A(\delta, \Mtt, \eta, \diam(\M)) + \delta) \eqsp .
    \end{equation}
    In particular if $\Mtt \leq 1/32$ and $\delta \leq 1/32$ then for any $k \in \{0, \dots, K\}$
    \begin{equation}
      \expeLigne{\normLigne{Y_k}^2} \leq \Ktt_0 = 5d + 320 (1+ \diam(\M))^2 \eqsp . 
    \end{equation}
  \end{lemma}

  \begin{proof}
    Recall that using \eqref{eq:discretization_improved}, we have that for any $k \in \{0, \dots, K-1\}$
    \begin{equation}
      \label{eq:discretization_improved_deet_linf}
    \textstyle{
      Y_{k+1} = Y_k + (\exp[\int_{T-t_{k+1}}^{T-t_k} \beta_s \rmd s] - 1) (Y_{k} + 2 \bm{s}(T-t_k, Y_{k})) + (\exp[2 \int_{T-t_{k+1}}^{T-t_k} \beta_s \rmd s] - 1)^{1/2} Z_k \eqsp ,
      }
    \end{equation}
    For simplicity, we denote
    \begin{equation}
      \textstyle{ \gamma_{1,k} = (\exp[\int_{T-t_{k+1}}^{T-t_k} \beta_s \rmd s] - 1)/\beta_{T-t_k} \eqsp , \qquad \gamma_{2,k} = (\exp[2 \int_{T-t_{k+1}}^{T-t_k} \beta_s \rmd s] - 1)/(2\beta_{T-t_k}) \eqsp . }
    \end{equation}
    Then, \eqref{eq:discretization_improved_deet_linf} can be rewritten for any $k \in \{0, \dots, K-1\}$ as
    \begin{equation}
      \label{eq:discretization_improved_deet_duo_linf}
          \textstyle{
      Y_{k+1} = Y_k + \gamma_{1,k} \beta_{T-t_k} (Y_{k} + 2 \bm{s}(T-t_k, Y_{k})) + \sqrt{2 \gamma_{2,k} \beta_{T-t_k}} Z_k \eqsp .
      }
    \end{equation}
    In what follows, we denote $\bar{\gamma}_{1,k} = \gamma_{1,k} \beta_{T-t_k}$
    and $\bar{\gamma}_{2,k} = \gamma_{2,k} \beta_{T-t_k}$. In addition, using
    that $\gamma_k \beta_{T-t_k} \leq \delta \leq 1/4$, we have that
    $\gamma_{1,k} \leq \gamma_{2,k} \leq 2 \gamma_{1,k}$.
    Indeed, we have that for any $k \in \{0, \dots, K-1\}$
    \begin{align}
      &\textstyle{\gamma_{2,k} / \gamma_{1,k} = (1/2) (\exp[\int_{T-t_{k+1}}^{T-t_k} \beta_s \rmd s] + 1) \geq 1 \eqsp ,} \\ 
        &\textstyle{\gamma_{2,k} / \gamma_{1,k} = (1/2) (\exp[\int_{T-t_{k+1}}^{T-t_k} \beta_s \rmd s] + 1) \leq (1/2)(\exp[\gamma_k \beta_{T-t_k}] + 1) \leq 2 \eqsp ,}
    \end{align}
    Using \Cref{lemma:control_grad}, we have that for any $t \in \ccint{0,T}$,
    $x_t \in \rset^d$ and $\eta >0$
    \begin{align}
      \langle x_t, \bm{s}(t, x_t) \rangle &\leq -\normLigne{x_t}^2/\sigma_t^2 + m_t \diam(\M)\normLigne{x_t}/\sigma_t^2 + \Mtt (1 + \normLigne{x_t})\normLigne{x_t} /\sigma_t^2  \\
      &\leq (-1 + 2 \Mtt + \eta  m_t \diam(\M))\norm{x_t}^2/\sigma_t^2 + (m_t \diam(\M)/\eta + \Mtt) / \sigma_t^2 \eqsp , \label{eq:s_scalar_product_linf}
    \end{align}
    where we have used that for any $a,b \geq 0$, $2ab \leq \eta a^2 + b^2/\eta$
    in the last line.  In addition, using \Cref{lemma:control_grad}, for any
    $t \in \ccint{0,T}$ and $x_t \in \rset^d$ we have
    \begin{align}      
      \normLigne{\bm{s}(t, x_t)}^2 &\leq 2 \normLigne{\bm{s}(t, x_t) - \nabla \log p_t(x_t)}^2 + 2 \normLigne{\nabla \log p_t(x_t)}^2 \\
                                   &\leq 4\Mtt^2(1 + \normLigne{x_t}^2)/\sigma_t^4 + 4 \normLigne{x_t}^2/\sigma_t^4 + 4 m_t^2 \diam(\M)^2/\sigma_t^4 \\
      &\leq 4(1 + \Mtt^2)\normLigne{x_t}^2/\sigma_t^4 + 4(\Mtt^2 + m_t^2 \diam(\M)^2 ) / \sigma_t^4 \eqsp .  \label{eq:s_norm_control_linf}
    \end{align}
    Combining \eqref{eq:discretization_improved_deet_duo_linf},
    \eqref{eq:s_scalar_product_linf} and \eqref{eq:s_norm_control_linf} we have for any
    $k \in \{0, \dots, K-1\}$
    \begin{align}
      \expeLigne{\normLigne{Y_{k+1}}^2} &= (1 + \bar{\gamma}_{1,k})^2 \expeLigne{\normLigne{Y_{k}}^2} + 4 \bar{\gamma}_{1,k}^2 \expeLigne{\normLigne{\bm{s}(T-t_k, Y_k)}^2} \\
      & \qquad +4 \bar{\gamma}_{1,k} (1 + \bar{\gamma}_{1,k}) \expeLigne{\langle Y_k, \bm{s}(T-t_k, Y_k)\rangle} + 2\bar{\gamma}_{2,k}d \\
                                        &\leq (1 + 2\bar{\gamma}_{1,k} + \bar{\gamma}_{1,k}^2) \expeLigne{\normLigne{Y_k}^2} + 16 (\bar{\gamma}_{1,k} / \sigma_{T-t_k}^2)^2 (1 + \Mtt^2)\expeLigne{\normLigne{Y_k}^2}  \\
                                        & \qquad + 16(\bar{\gamma}_{1,k} / \sigma_{T-t_k}^2)^2 (\Mtt^2 + m_{T-t_k}^2 \diam(\M)^2 ) \\
      &\qquad + 4 \bar{\gamma}_{1,k} (1 + \bar{\gamma}_{1,k}) \expeLigne{\langle Y_k, \bm{s}(T-t_k, Y_k)\rangle} + 4\bar{\gamma}_{1,k}d\\
                                        &\leq (1 + 2\bar{\gamma}_{1,k} + \bar{\gamma}_{1,k}^2) \expeLigne{\normLigne{Y_k}^2} + 16 (\bar{\gamma}_{1,k} / \sigma_{T-t_k}^2)^2 (1 + \Mtt^2)\expeLigne{\normLigne{Y_k}^2}  \\
                                        & \qquad + 16(\bar{\gamma}_{1,k} / \sigma_{T-t_k}^2)^2 (\Mtt^2 + m_{T-t_k}^2 \diam(\M)^2 ) \\
      & \qquad + 4 (\bar{\gamma}_{1,k}/\sigma_{T-t_k}^2) (1 + \bar{\gamma}_{1,k}) (-1 + 2 \Mtt + \eta m_{T-t_k} \diam(\M))\expeLigne{\norm{Y_k}^2} \\
                                        & \qquad + (\bar{\gamma}_{1,k}/ \sigma_{T-t_k}^2 ) (1 + \bar{\gamma}_{1,k}) (m_{T-t_k} \diam(\M)/\eta + \Mtt) + 4 \bar{\gamma}_{1,k} d \eqsp . 
    \end{align}
    In what follows, we let 
    $\delta_k = \bar{\gamma}_{1,k} /
    \sigma_{T-t_k}^2$. Using
    that for any
    $t \in \ccint{0,T}$,
    $m_t, \sigma_t \in \ccint{0,1}$
    we have
    \begin{align}
      \expeLigne{\normLigne{Y_{k+1}}^2} & \leq (1 + 2\delta_k + \delta_k^2) \expeLigne{\normLigne{Y_k}^2} + 16 \delta_k^2 (1 + \Mtt^2)\expeLigne{\normLigne{Y_k}^2}  \\
                                        & \qquad + 16\delta_k^2 (\Mtt^2 + \diam(\M)^2 ) \\
      & \qquad + 4 \delta_k (1 + \delta_k) (-1 + 2 \Mtt +\eta \diam(\M))\expeLigne{\norm{Y_k}^2} \\
      & \qquad +  \delta_k (1 + \delta_k) (\diam(\M)/\eta + \Mtt) + 4 \delta_k d
    \end{align}
    Since $s \mapsto \beta_s$ is non-decreasing, that
    $\beta_{T-t_k} \gamma_k \leq  \delta \leq 1/4$ and that for any
    $w \in \ccint{0,1/2}$, $\rme^{w} -1 \leq 1 + 2 w$, we have 
    \begin{equation}
      \textstyle{ \exp[\int_{T-t_{k+1}}^{T-t_k} \beta_s \rmd s] - 1 \leq \exp[\beta_{T-t_k} \gamma_k] - 1 \leq 2 \beta_{T-t_k} \gamma_k \eqsp . }
    \end{equation}
    Therefore, we get that
    $\delta_k \leq 2 \gamma_k \beta_{T-t_k} /\sigma_{T-t_k}^2 \leq 2 \delta$.
    Since $\delta_k \leq 2\delta$ we have
    \begin{align}
      \expeLigne{\normLigne{Y_{k+1}}^2} & \leq (1 + 2\delta_k +2 \delta_k \delta) \expeLigne{\normLigne{Y_k}^2} + 32 \delta_k \delta (1 + \Mtt^2)\expeLigne{\normLigne{Y_k}^2}  \\
                                        & \qquad + 4 \delta_k  (-1 + 2 \Mtt + \eta \diam(\M))\expeLigne{\norm{Y_k}^2} \\
                                        & \qquad + 32\delta_k \delta (\Mtt^2 + \diam(\M)^2 ) \\      
                                        & \qquad + 2\delta_k (1 + \delta) (\diam(\M)/\eta  + \Mtt) + 4 \delta_k d \eqsp .
    \end{align}
    Hence, we have that
    \begin{align}
      \expeLigne{\normLigne{Y_{k+1}}^2} &\leq (1 + \delta_k[-2 + 2\delta + 32 \delta (1 + \Mtt^2) + 8 \Mtt + 4\eta \diam(\M)]) \expeLigne{\normLigne{Y_k}^2} \\
      & \qquad + \delta_k [32\delta (\Mtt^2 + \diam(\M)^2) + 2(1+\delta)(\diam(\M)/\eta + \Mtt) + 4 d] \eqsp . 
    \end{align}
    Denote
    $A = 2 - 2\delta - 32 \delta (1
    + \Mtt^2) - 8 \Mtt -
    4\eta \diam(\M)$, 
    $B = 32\delta (\Mtt^2 +
    \diam(\M)^2) + 2 
    (1+\delta)(\diam(\M)/\eta + \Mtt) + 4d$. Then, we have
    \begin{equation}
      \expeLigne{\normLigne{Y_{k+1}}^2} \leq (1 - \delta_k A) \expeLigne{\normLigne{Y_k}^2} + \delta_k B \eqsp . 
    \end{equation}
    Hence, if
    $\expeLigne{\normLigne{Y_k}^2}
    \geq B / A$ we have that
    $\expeLigne{\normLigne{Y_{k+1}}^2}
    \leq
    \expeLigne{\normLigne{Y_{k}}^2}$. In
    addition, if
    $\expeLigne{\normLigne{Y_k}^2}
    \leq B / A$ then,
    $\expeLigne{\normLigne{Y_{k+1}}^2}
    \leq B/A + \delta B$. Therefore,
    we conclude by recursion that for any $k \in \{0, \dots, K\}$
    \begin{equation}
      \expeLigne{\normLigne{Y_k}^2} \leq d + B(1/A + \delta) \eqsp ,
    \end{equation}
    which concludes the first part of the proof. If $\delta, \Mtt \leq 1/32$
    then, $A(\delta, \Mtt, \eta, \diam(\M)) \geq 1/2 - 4 \eta \diam(\M)$. We
    conclude upon setting $\eta = 1/(16\diam(\M))$. In that case
    $A(\delta, \Mtt, \eta, \diam(\M)) \geq 1/4$ and
    \begin{align}
      B(\delta, \Mtt, \eta, \diam(\M)) &\leq 32\delta (\Mtt^2 +
    \diam(\M)^2) + 2 
                                         (1+\delta)(16\diam(\M) + \Mtt) + 4d \\
      &\leq 64(1 + \diam(\M) + \diam(\M)^2)+ 4d \eqsp ,
    \end{align}
    which concludes the proof.
  \end{proof}

  Note that the same result holds for $(\bbfY_t)_{t \in \ccint{0,t_K}}$.
  
  \begin{lemma}
    \label{lemma:lipschitz-behavior}
    Assume \rref{assum:manifold_hyp}, \rref{assum:assumption_beta} and
    \rref{assum:score_control}. In addition, assume that for any
    $k \in \{0, \dots, K-1\}$,
    $\gamma_k \beta_{T-t_k} / \sigma_{T-t_k}^2 \leq \delta \leq 1/4$. Then, we
    have for any $k \in \{0, \dots, K-1\}$ and $t \in \ccint{t_k, t_{k+1}}$
    \begin{equation}
      \expeLigne{\normLigne{\bbfY_t - \bbfY_{t_k}}^2}
      \leq \Ltt \beta_{T-t_k} \gamma_k \eqsp ,
    \end{equation}
    with $\Ltt = 8 (1 + \delta) ( 16 (5 + \Mtt^2) \Ktt + 16 (4\diam(\M)^2 + \Mtt^2)) + 4$ and
    where $\Ktt$ is defined in \eqref{eq:Ktt_def}. In particular if $\Mtt \leq 1/32$ and $\delta \leq 1/32$
    \begin{equation}
      \expeLigne{\normLigne{\bbfY_t - \bbfY_{t_k}}^2} \leq \Ltt_0 \beta_{T-t_k} \gamma_k 
      = (64 d + 20544 (1 + \diam(\M))^2) \beta_{T-t_k} \gamma_k \eqsp .
    \end{equation}    
  \end{lemma}
  \begin{proof}
    Recall that
    \begin{equation}
  \textstyle{
    \bbfY_t = \bbfY_{t_k} + (\exp[\int_{T-t}^{T-t_k} \beta_s \rmd s] - 1) (\bbfY_{t_k} + 2 \bm{s}(T-t_k, \bbfY_{t_k})) + (\exp[2 \int_{T-t}^{T-t_k} \beta_s \rmd s] - 1)^{1/2} Z \eqsp ,
    }
  \end{equation}
    where $Z$ is a Gaussian random variable with zero mean and identity covariance
    and the equality holds in distribution independent from $\bbfY_{t_k}$.
    Therefore, we get that
    \begin{align}
      \expeLigne{\normLigne{\bbfY_t - \bbfY_{t_k}}^2} &\leq 2\textstyle{(\exp[\int_{T-t}^{T-t_k} \beta_s \rmd s] - 1)^2 (\expeLigne{\normLigne{\bbfY_{t_k}}^2} + 4 \expeLigne{\normLigne{\bm{s}(T-t_k, \bbfY_{t_k})}^2})} \\
      & \qquad \textstyle{+ (\exp[2 \int_{T-t}^{T-t_k} \beta_s \rmd s] - 1)d \eqsp . \label{eq:easy_integration}}
    \end{align}
    In addition, using \Cref{assum:score_control}, \Cref{lemma:control_grad} and
    that for any $t \in \ccint{0,T}$, $m_t \in \ccint{0,1}$, we get that for any
    $u \in \ccint{0,T}$ and $x_u \in \rset^d$
    \begin{align}
      \normLigne{\bm{s}(u, x_u)} &\leq \Mtt(1 + \normLigne{x_u})/\sigma_u^2 + 2 \normLigne{x_u}/\sigma_u^2 + 2 \diam(\M)/\sigma_u^2 \\
      &\leq (1/\sigma_u^2)\{ (\Mtt + 2) \norm{x_u} + (\Mtt + 2 \diam(\M)) \} \eqsp .
    \end{align}
    Combining this result and \eqref{eq:easy_integration}, we get that
    \begin{align}
      \expeLigne{\normLigne{\bbfY_t - \bbfY_{t_k}}^2} &= 2\textstyle{(\exp[\int_{T-t}^{T-t_k} \beta_s \rmd s] - 1)^2 (\expeLigne{\normLigne{\bbfY_{t_k}}^2} } \\
      & \qquad \textstyle{+ 4 \expeLigne{\normLigne{\bm{s}(T-t_k, \bbfY_{t_k})}^2}) + (\exp[2 \int_{T-t}^{T-t_k} \beta_s \rmd s] - 1)d} \\
                                                      &\leq 2\textstyle{(\exp[\int_{T-t}^{T-t_k} \beta_s \rmd s] - 1)^2} \expeLigne{\normLigne{\bbfY_{t_k}}^2 } \\
      & \qquad + 32 (4 + \Mtt^2) \textstyle{(\exp[\int_{T-t}^{T-t_k} \beta_s \rmd s] - 1)^2} \expeLigne{\normLigne{\bbfY_{t_k}}^2}/\sigma_{T-t_k}^2 \\
      & \qquad  + 32 \textstyle{(\exp[\int_{T-t}^{T-t_k} \beta_s \rmd s] - 1)^2} (4\diam(\M)^2 + \Mtt^2)/\sigma_{T-t_k}^2 \\
      & \qquad \textstyle{+ (\exp[2 \int_{T-t}^{T-t_k} \beta_s \rmd s] - 1)d} \eqsp . \label{eq:inter_bound_exp_time}
    \end{align}
    Since $s \mapsto \beta_s$ is non-decreasing, that
    $\beta_{T-t_k} \gamma_k \leq \delta \leq 1/4$ and that for any
    $w \in \ccint{0,1/2}$, $\rme^{2 w} -1 \leq 1 + 4 w$, we have for any $\alpha \in \{1,2\}$
    \begin{equation}
      \textstyle{ \exp[\alpha \int_{T-t}^{T-t_k} \beta_s \rmd s] - 1 \leq \exp[\alpha \beta_{T-t_k} \gamma_k] - 1 \leq 2 \alpha \beta_{T-t_k} \gamma_k \eqsp . }
    \end{equation}
    Combining this result and \eqref{eq:inter_bound_exp_time}, we get that
    \begin{align}
      \expeLigne{\normLigne{\bbfY_t - \bbfY_{t_k}}^2}  &\leq 8 \beta_{T-t_k}^2 \gamma_k^2 \expeLigne{\normLigne{\bbfY_{t_k}}^2 + 32 (4 + \Mtt^2) \normLigne{\bbfY_{t_k}}^2/\sigma_{T-t_k}^2 + 32 (4\diam(\M)^2 + \Mtt^2)/\sigma_{T-t_k}^2} \\
                                                       & \qquad + 4 \beta_{T-t_k} \gamma_k d \label{eq:inter_bound_exp_time} \\
                                                       &\leq 8 (\beta_{T-t_k}^2 \gamma_k^2/\sigma_{T-t_k}^2) ( 32 (5 + \Mtt^2) \expeLigne{\normLigne{\bbfY_{t_k}}^2} + 32 (4\diam(\M)^2 + \Mtt^2))  + 4 \beta_{T-t_k} \gamma_k d \eqsp .      
    \end{align}
    Therefore, using \Cref{lemma:control_growth_discrete_process} and \eqref{eq:Ktt_def}, we get that
    \begin{align}
      \expeLigne{\normLigne{\bbfY_t - \bbfY_{t_k}}^2} &\leq 8 (\beta_{T-t_k}^2 \gamma_k^2/\sigma_{T-t_k}^2) ( 32 (5 + \Mtt^2) \Ktt + 32 (4\diam(\M)^2 + \Mtt^2))  + 4 \beta_{T-t_k} \gamma_k d\\
      &\leq \{ 256 \delta ( (5 + \Mtt^2) \Ktt + 4\diam(\M)^2 + \Mtt^2)  + 4d\} \beta_{T-t_k}  \gamma_k 
    \end{align}
    which concludes the first part of the proof. Now assuming that $\delta, \Mtt \leq 1/32$ we have
    \begin{align}
      \expeLigne{\normLigne{\bbfY_t - \bbfY_{t_k}}^2} &\leq (64 \Ktt + 64 \diam(\M)^2 + 16d)\beta_{T-t_k} \gamma_k \\
      &\leq 64 d + 20544 (1 + \diam(\M))^2 \eqsp ,
    \end{align}
    which concludes the proof.
  \end{proof}

  \subsection{Control of the tangent backward process}
  \label{sec:contr-tang-backw}

We now introduce the tangent process associated with
$(\bfY_t)_{t \in \ccint{0,T}}$. We have
\begin{equation}
  \label{eq:stochastic_flow}
  \rmd \nabla \bfY_t = \beta_{T-t} (\Id + 2 \nabla^2 \log p_{T-t}(\bfY_t))\nabla \bfY_t \rmd t \eqsp , \qquad \nabla \bfY_0 = \Id \eqsp .
\end{equation}
We recall that controlling the tangent process allows to control the Wasserstein
distance between the original process and its target measure. Indeed, let
$(\bfY_t^x)_{t\in\ccint{0,T}}$ and $(\bfY_t^y)_{t\in\ccint{0,T}}$ be the
processes given by \eqref{eq:continuous_exact_app} with initial condition $x$ and
$y$ respectively.  Then we have that for any $t \in \ccint{0,T}$
\begin{equation}
  \label{eq:pro_wasserstein}
  \textstyle{
    \normLigne{\bfY_t^x - \bfY_t^y} \leq \int_0^1 \normLigne{\nabla \bfY_t^{z_\lambda}} \rmd \lambda \normLigne{x-y} \eqsp ,
    }
\end{equation}
where $(\nabla \bfY_t^x)_{t \in \ccint{0,T}}$ is the tangent process given by
\eqref{eq:stochastic_flow} and associated with
$(\bfY_t^{z_\lambda})_{t\in \ccint{0,T}}$, where
$z_\lambda = \lambda x+(1-\lambda)y$. Before providing controls in the general
setting, we take a detour and focus on the case where $\diam(\M) = 0$, \ie
$\pi = \updelta_0$. In the following proposition, we show that in this case the
backward process converges in finite-time. This highlights the role of the
diameter of the manifold in the subsequent analysis.

\begin{proposition}
  \label{prop:convergence_dirac_zero}
  Assume \rref{assum:manifold_hyp} and that $\diam(\M) = 0$, \ie
  $\pi = \updelta_0$. Then, we have that for any $x,y \in \rset^d$ and
  $t \in \ccint{0,T}$
  \begin{equation}
    \wassersteinD[1](\updelta_x \Qker_t, \updelta_y \Qker_y) \leq \textstyle{ 2\exp[(1/2)\int_{T-t}^T \beta_{s} \rmd s] (\exp[2\int_0^{T} \beta_s \rmd s] -1)^{-1/2} (\exp[2\int_0^{T-t} \beta_s \rmd s] -1)^{1/2} \norm{x-y} } \eqsp . 
  \end{equation}
  In particular, we have that for any $x \in \rset^d$,
  $\updelta_x \Qker_T = \pi$, \ie the backward diffusion converges in finite
  time no matter the initialization distribution.
\end{proposition}

\begin{proof}
  Let $t \in \ccint{0,T}$. Using \Cref{lemma:control_hessian}, we have that for
  any $\mathrm{M} \in \mathcal{M}_d(\rset)$
  \begin{equation}
    \langle \mathrm{M}, \nabla^2 \log q_t(x_t) \mathrm{M} \rangle \leq - \sigma_t^{-2} \normLigne{\mathrm{M}}^2 \eqsp .
  \end{equation}
  In particular, we have that for
  any $\mathrm{M} \in \mathcal{M}_d(\rset)$
  \begin{equation}
    \label{eq:inequality_M}
   \beta_{T-t} \langle \mathrm{M}, \Id + 2 \nabla^2 \log q_{T-t}(x_t) \mathrm{M} \rangle \leq (\beta_{T-t} - 2 \beta_{T-t} \sigma_{T-t}^{-2}) \normLigne{\mathrm{M}}^2 \eqsp . 
 \end{equation}
 Using \Cref{lemma:integration_lemma}, we have
 \begin{align}
   &\textstyle{\int_0^t (\beta_{T-s} - 2 \beta_{T-s} \sigma_{T-s}^{-2}) \rmd s} \\
   & \qquad \textstyle{= \int_0^t \beta_{T-s} \rmd s + \log(\exp[2\int_0^{T-t} \beta_s \rmd s] -1) - \log(\exp[2\int_0^{T} \beta_s \rmd s] -1) } \\
   & \qquad \textstyle{= \int_{T-t}^T \beta_{s} \rmd s + \log(\exp[2\int_0^{T-t} \beta_s \rmd s] -1) - \log(\exp[2\int_0^{T} \beta_s \rmd s] -1) \eqsp . }     
 \end{align}
 Finally, we have that
 \begin{align}
   &\textstyle{\exp[\int_0^t (\beta_{T-s} - 2 \beta_{T-s} \sigma_{T-s}^{-2}) \rmd s]} \\
   & \qquad \leq \textstyle{ \exp[\int_{T-t}^T \beta_{s} \rmd s] (\exp[2\int_0^{T} \beta_s \rmd s] -1)^{-1} (\exp[2\int_0^{T-t} \beta_s \rmd s] -1) \eqsp . }
 \end{align}
 Hence, using this result, \eqref{eq:stochastic_flow} and \eqref{eq:inequality_M}, we get that
 \begin{equation}
  (1/2) \normLigne{\nabla \bfY_t}^2 \leq \textstyle{ \exp[\int_{T-t}^T \beta_{s} \rmd s] (\exp[2\int_0^{T} \beta_s \rmd s] -1)^{-1} (\exp[2\int_0^{T-t} \beta_s \rmd s] -1) \eqsp , }
 \end{equation}
 which concludes the first part of the proof, using
 \eqref{eq:pro_wasserstein}. For the second part of the proof, we first remark
 that for any $x, y \in \rset^d$,
 $\wassersteinD[1](\updelta_x \Qker_T, \updelta_y \Qker_T) = 0$. Therefore, for
 any probability measures $\mu, \nu$ such that
 $\int_{\rset^d} \normLigne{x} \rmd \mu(x) <+\infty$ and
 $\int_{\rset^d} \normLigne{x} \rmd \nu(x) <+\infty$, we have
 $\wassersteinD[1](\mu \Qker_T, \nu \Qker_T) = 0$ We conclude upon combining
 this result and that $(\updelta_x \Pker_T) \Qker_T = \updelta_x$.
\end{proof}

Note that it is also possible to explicitly write down the backward stochastic
process in this case since $\nabla \log p_t$ is available in close form. One can
remark that in this case we recover an Ornstein--Uhlenbeck bridge.

In the case where the diameter is non-zero, we cannot recover such a
contraction. However, it is possible to obtain a contraction up to a certain
point.  The following lemma will allow us to control the growth of the tangent
process.
\begin{lemma}
  \label{lemma:growth_tangent_process}
  Assume \rref{assum:assumption_beta} and that
  $T \geq 2\bar{\beta}(1 + \log(1+\diam(\M))$. Let $t^\star \in \ccint{0,T}$
  given by
  \begin{equation}
    t^\star = T - 2\bar{\beta}(1 + \log(1+\diam(\M)) \eqsp . 
  \end{equation}
  Then, for any $t \in \ccint{0,t^\star}$ we have
  \begin{equation}
\textstyle{    \int_0^{t}\beta_{T-s} (1 -2/\sigma_{T-s}^2 + m_{T-s}^2\diam(\M)^2/\sigma_{T-s}^4) \rmd s  \leq -(1/2) \int_0^{t} \beta_{T-t} \rmd s  \eqsp . }
  \end{equation}
  In addition, for any $t \in \ccint{t^\star,T}$
  \begin{equation}
    \textstyle{    \int_{t^\star}^{t}\beta_{T-s} (1 -2/\sigma_{T-s}^2 + m_{T-s}^2\diam(\M)^2/\sigma_{T-s}^4) \rmd s}  \leq (\diam(\M)^2/2)(\sigma_{T-t}^{-2} - \sigma_{T-t^\star}^{-2}) \eqsp . 
  \end{equation}
\end{lemma}

\begin{proof}
  Let $s \in \ccint{0,T}$. Note that we have
  \begin{equation}
    \label{eq:upper_bound_neg}
    1 - 2/\sigma_{T-s}^2 + m_{T-s}^2 \diam(\M)^2/\sigma_{T-s}^4 \leq -1/2 \eqsp , 
  \end{equation}
  if and only if
  \begin{equation}
    3 \sigma_{T-t}^4 -4\sigma_{T-s}^2 + 2m_{T-s}^2 \diam(\M)^2 \leq 0 \eqsp . 
  \end{equation}
  Hence, using this result and the fact that $\sigma_{T-s}^2 = 1 - m_{T-s}^2$ we
  have that \eqref{eq:upper_bound_neg} is satisfied if and only if
  \begin{equation}
    3 m_{T-s}^4 +2 (\diam(\M)^2 - 1) m_{T-s}^2 -1 \leq 0 \eqsp . 
  \end{equation}
  Introduce $P(u) = 3 u^2 + 2(\diam(\M)^2 -1)u - 1$. We have that $P(u) \leq 0$
  for $u \in \ccint{0, u_0}$ with
  \begin{equation}
    u_0 = [-(\diam(\M)^2 -1) + ((\diam(\M)^2 -1)^2 + 3)^{1/2}]/3 = (\delta + (\delta^2 +3)^{1/2})/3\eqsp ,
  \end{equation}
  where $\delta = \diam(\M)^2 -1 \in \coint{-1, +\infty}$.  If
  $\diam(\M)^2 -1 \geq 1$ then, using this result and the fact that for any
  $a \in \coint{0,+\infty}$, $(1+a)^{1/2} \geq 1 + a/2 - a^2/8$ we have
  \begin{equation}
    \label{eq:u0_lower}
    u_0 \geq \delta(-1 + (1+3/\delta^2)^{1/2})/3 \geq \delta(1/(2\delta^2) - 3/(8\delta^4)) \geq (1/2 - 3/8)/\delta \geq 1/(8\delta) \eqsp . 
  \end{equation}
  In addition, if $\delta \leq 1$ then $\delta^2 \in \ccint{0,1}$. Using this
  result and the fact that for any $a \in \ccint{0,1}$,
  $\sqrt{3 +a} \geq (2 - \sqrt{3}) a + \sqrt{3}$ we have
  \begin{align}
    u_0 &\geq (-\delta + (3 + \delta^2)^{1/2})/3 \\
    &\geq (-\absLigne{\delta} + (3 + \absLigne{\delta}^2)^{1/2})/3 \geq ((1 - \sqrt{3})\absLigne{\delta} + \sqrt{3})/3 \geq 1/3 \eqsp . 
  \end{align}
  Combining this result and \eqref{eq:u0_lower}, we get that
  \begin{equation}
    u_0 \geq 1/(8(1 + \absLigne{\delta})) \eqsp . 
  \end{equation}
  Therefore, we get that for any $t \in \ccint{0,T}$ such that
  $m_{T-t}^2 \leq 1/(8(1 + \absLigne{\delta}))$
  \begin{equation}
    1 - 2/\sigma_{T-s}^2 + m_{T-s}^2 \diam(\M)^2/\sigma_{T-s}^4 \leq -1/2 \eqsp . 
  \end{equation}
  Hence, for any $t \in \ccint{0,T}$ such that
  $\exp[-2(T-t)/\bar{\beta}] \leq 1/(8(1 + \absLigne{\delta}))$
  \begin{equation}
    1 - 2/\sigma_{T-s}^2 + m_{T-s}^2 \diam(\M)^2/\sigma_{T-s}^4 \leq -1/2 \eqsp . 
  \end{equation}
  Let $t^\star_0$ such that
  $\exp[-2(T-t_0^\star)/\bar{\beta}] = 1/(8(1 + \absLigne{\delta}))$.
  We get that
  \begin{equation}
    t_0^\star = T - (\bar{\beta}/2)\log(8(1 + \absLigne{\delta})) \eqsp . 
  \end{equation}
  Using that for any $a \geq 0$, $1 + \absLigne{a^2 -1} \leq 2 (1 +a)^2$ and
  that $\log(16) \leq 4$, we get that
  \begin{equation}
    t_0^\star \geq T - (\bar{\beta}/2)(\log(16(1+\diam(\M))^2)) \geq t^\star = T - 2\bar{\beta}(1 + \log(1+\diam(\M))) \eqsp . 
  \end{equation}
  Hence, since $t \mapsto m_{T-t}$ is non-decreasing, we get that for any $t \in \ccint{0, t^\star}$, $m_{T-t} \leq 1/(8(1+\absLigne{\delta}))$ and therefore
    \begin{equation}
    1 - 2/\sigma_{T-s}^2 + m_{T-s}^2 \diam(\M)^2/\sigma_{T-s}^4 \leq -1/2 \eqsp ,
  \end{equation}
  which concludes the first part of the proof. The second part of the proof
  follows from \Cref{lemma:integration_lemma} and
  \begin{align}
    &\textstyle{    \int_{t^\star}^{t}\beta_{T-s} (1 -2/\sigma_{T-s}^2 + m_{T-s}^2\diam(\M)^2/\sigma_{T-s}^4) \rmd s}  \\
    & \qquad \leq \textstyle{ (\diam(\M)^2/2)[(1-\exp[-2\int_0^{T-t} \beta_s \rmd s])^{-1} - (1-\exp[-2\int_0^{T-t^\star} \beta_s \rmd s])^{-1}]  } \\
    & \qquad \leq (\diam(\M)^2/2)(\sigma_{T-t}^{-2} - \sigma_{T-t^\star}^{-2}) \eqsp ,
  \end{align}
  which concludes the proof.
\end{proof}

The control of the tangent process in the case where $\diam(\M) > 0$ is given in
\Cref{prop:control_gradient}.

  \begin{proposition}
    \label{prop:wasserstein_forward}
    Assume \rref{assum:manifold_hyp} and
    $T \geq 2\bar{\beta}(1 + \log(1+\diam(\M))$. Then, for any
    $x, y \in \rset^d$ and $t \in \ccint{0,t_K}$
  \begin{equation}
    \textstyle{
      \wassersteinD[1](\updelta_x \Qker_{t}, \updelta_y \Qker_{t}) \leq  \exp[(\diam(\M)^2/2)\sigma_{T-t_K}^{-2}] \normLigne{x-y} \eqsp .
      }
    \end{equation}    
  \end{proposition}

  \begin{proof}
    This is a direct consequence of \eqref{eq:pro_wasserstein} and \Cref{prop:control_gradient}.
  \end{proof}


  \section{A stochastic interpolation formula}
\label{sec:stoch-interp-form}

In this section, we present a formula first introduced in \citet{del2019backward}
which is a stochastic extension of the Alekseev--Gr\"obner formula
\citep{alekseev1961estimate}.  We recall the definition of the stochastic flows
$(\bfY_{s,t}^x)_{s, t \in \ccint{0,T}}$ and the interpolation of its
discretization $(\bbfY_{s,t}^x)_{s, t \in \ccint{0,T}}$, for any $x \in \rset^d$
\begin{equation}
  \textstyle{
    \rmd \bfY_{s,t}^x = \beta_{T-t} \{ \bfY_{s,t}^x + 2 \nabla \log q_{T-t}(\bfY_{s,t}^x) \} \rmd t + \sqrt{2 \beta_{T-t}} \rmd \bfB_t \eqsp , \qquad \bfY_{s,s}^x = x \eqsp .
    }
\end{equation}
and 
\begin{equation}
  \textstyle{
    \rmd \bbfY_{s,t}^x = \beta_{T-t} \{ \bbfY_{s, t}^x + 2 \bm{s}(T-t_k, \bbfY_{s, t_k}^x) \} \rmd t + \sqrt{2\beta_{T-t}} \rmd \bfB_t \eqsp , \qquad \bbfY_{s,s}^x = x \eqsp .
    }
\end{equation}
The following proposition is a straightforward application of \cite[Theorem
1.2]{del2019backward}. Note that these results apply since the drift and the
volatility of the backward processes have bounded differential up to order
three, see \Cref{lemma:higher_order}.

\begin{proposition}
  \label{prop:pierre-extended_appendix}
  Assume \rref{assum:manifold_hyp}. Then, for any $s, t \in \coint{0,T}$ with $s<t$ and $x \in \rset^d$
  \begin{equation}
\textstyle{    \bfY_{s,t}^x - \bbfY_{s,t}^x = \int_s^t (\nabla \bfY_{u,t}^{\bbfY_{s,u}})^\top \Delta b_u((\bbfY_{s,v})_{v\in \ccint{s,u}}) \rmd u } \eqsp ,
\end{equation}
where for any $u \in \coint{0,T}$ such that $u \in \coint{t_k, t_{k+1}}$ for
some $k \in \{0, \dots, K-1\}$ and $(\omega_v)_{v \in \ccint{s,u}} \in \rmc(\ccint{s,u}, \rset^d)$ we have 
  \begin{align}
    &b_u(\omega) = \beta_{T-u}( \omega_u + 2 \nabla \log q_{T-u}(\omega_u)) \eqsp , \qquad \bar{b}_u(\omega) = \beta_{T-t_u}(\omega_{u} + 2 \bm{s}(T -t_k, \omega_{t_k})) \eqsp , \\
    &\Delta b_u(\omega) = b_u(\omega) - \bar{b}_u(\omega) \eqsp . 
  \end{align}
\end{proposition}


\section{Additional comments on \Cref{thm:convergence_general}}
\label{sec:comment-thm}

In this section, we discuss the validity of
  \rref{assum:score_control} and then comment the suboptimality of the bound of
  \Cref{thm:convergence_general}.

\subsection{Validity of \rref{assum:score_control}}
\label{sec:validity-}

First, we highlight that under \rref{assum:manifold_hyp},
  \rref{assum:score_control} is satisfied if for any $t \in \ocint{0,T}$ and
  $x_t \in \rset^d$ we have
  $\normLigne{\bm{s}(t,x_t) - \nabla \log p_t(x_t)} \leq \Mtt_r
  \normLigne{\nabla \log p_t(x_t)}$ with $\Mtt_r \geq 0$. Indeed, using this
  condition \rref{assum:manifold_hyp}, \Cref{lemma:control_grad} and letting
  $\Mtt= 4\Mtt_r(1 + \diam(\M))$, we get that \rref{assum:score_control} is
  satisfied. Hence, \rref{assum:score_control} is implied by a control on the
  \emph{relative} error between the score function and its approximation.

Assume that $\pi = \updelta_0$. Then in that case
  $\nabla \log p_t = \nabla \log p_{t|0}(\cdot|0)$ and we can compute explicitly the
  error $\normLigne{\bm{s}(t,x_t) - \nabla \log p_t(x_t)}$ at given query points
  $(t,x_t)$. In \Cref{fig:dirac_mass}, we analyze this error in a two-dimensional setting. In
  particular, we recover that the behavior of the error is explosive as
  $\normLigne{x} \to +\infty$ and $t \to 0$.

\begin{figure}[h]  
  \centering
  \includegraphics[width=.3\linewidth]{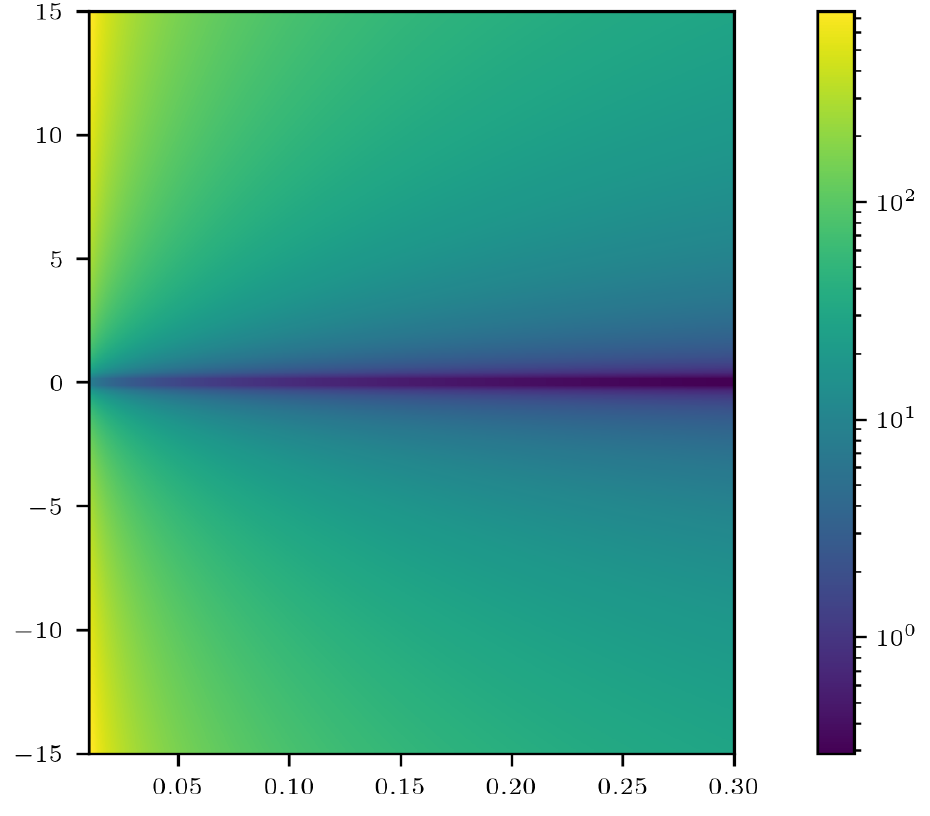} \hspace{1cm}
  \includegraphics[width=.3\linewidth]{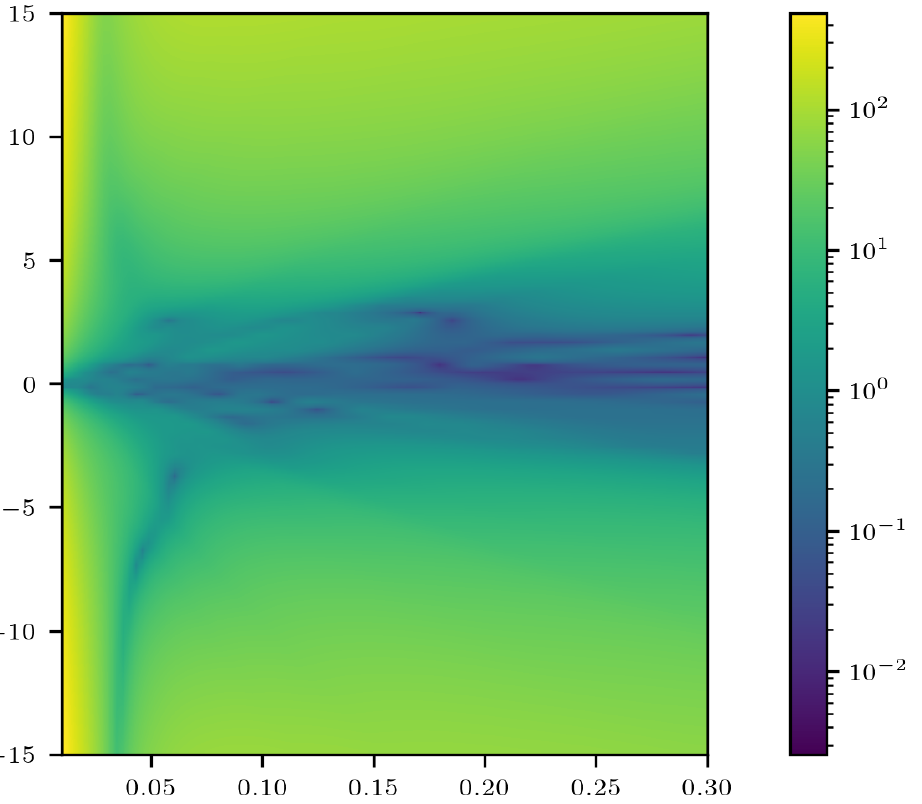} 
  \caption{Target is a Dirac mass at 0. Left: norm of the true
      score $\normLigne{\nabla \log p_{t|0}(x)}$ ($x$-axis: time evolution
      ($T=1$), $y$-axis: spatial evolution (along the first coordinate, second
      coordinate is fixed to $0$). Right: norm of the error between the
      estimated score and the true score
      $\normLigne{\bm{s}(t, x) - \nabla \log p_{t|0}(x)}$ ($x$-axis: time
      evolution ($T=1$), $y$-axis: spatial evolution (along the first
      coordinate, second coordinate is fixed to $0$).}
  \label{fig:dirac_mass}
\end{figure}
Finally, we conclude this study by illustrating the explosive
  behavior of the norm of the score in a two-dimensional setting (we restrict
  ourselves to this small dimensional setting so that we can get a dense grid of
  query points without encountering memory issues), see \Cref{fig:s_curve}. We
  emphasize that the norm of the score has a similar behavior as the error term,
  \ie it is explosive as $\normLigne{x} \to +\infty$ and $t \to 0$.
\begin{figure}[h]
  \centering
  \includegraphics[width=.3\linewidth]{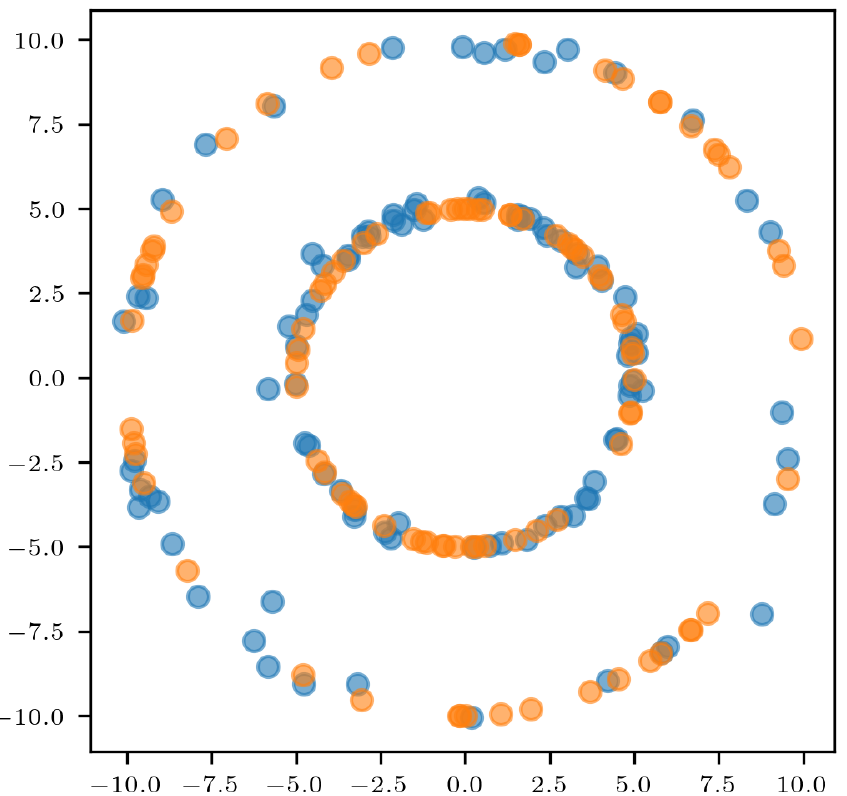} \hfill
  \includegraphics[width=.3\linewidth]{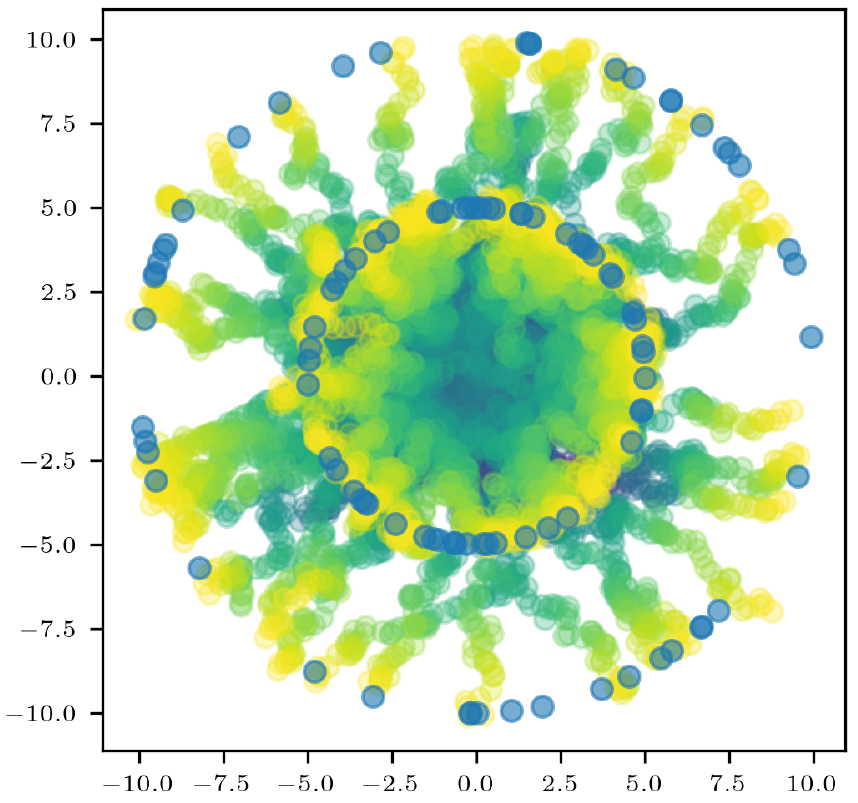} \hfill 
  \includegraphics[width=.345\linewidth]{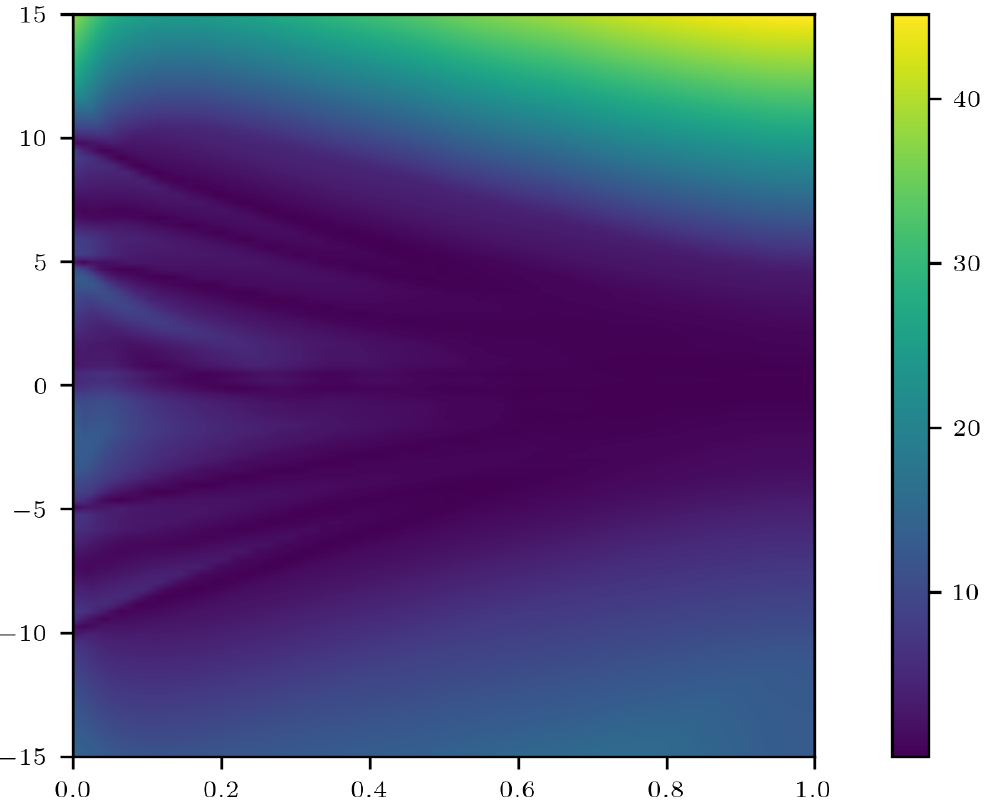}
  \caption{Target is the uniform distribution on two concentric
      circles. Left: samples from the target distribution (orange) and samples
      from the diffusion model (blue). Middle: trajectories of the diffusion
      model. Right: norm of the estimated score $\normLigne{\bm{s}(t,x)}$
      ($x$-axis: time evolution ($T=1$), $y$-axis: spatial evolution (along the
      first coordinate, second coordinate is fixed to $0$).}
    \label{fig:s_curve}
\end{figure}

In both settings, the score function is learned by minimizing the Denoising
Score Matching objective \eqref{eq:dsm_loss} using the ADAM optimizer. The
architecture of the network and training settings are similar to the ones used
in \cite{de2021simulating}.

\subsection{Suboptimality of the bound}

Recall that the generative model is given by
  $\mathcal{L}(Y_K)$. In \Cref{thm:convergence_general}, we provide an upper
  bound on $\wassersteinD[1](\mathcal{L}(Y_K), \pi)$. In this section, we
  compare the obtained bound with a bound on
  $\wassersteinD[1](\mathcal{L}(Y_0), \pi)$. Recall that
  $Y_0 \sim \mathrm{N}(0, \Id)$. Considering any coupling $(X, Y_0)$ such that
  $X \sim \pi$ and $Y_0 \sim \mathrm{N}(0, \Id)$, we have
\begin{equation}  
  \wassersteinD[1](\mathcal{L}(Y_0), \pi) \leq \expeLigne{\normLigne{X}} + \expeLigne{\normLigne{Y_0}} \leq \diam(\M) + \sqrt{d} \eqsp . 
\end{equation}
This naive bound can be better than the one obtained in
\Cref{thm:convergence_general}, especially for large values of $\Dtt_0$. If that
is the case then the derived bound seems to be vacuous at first sight since
diffusing the process backward does \emph{not} improve the Wasserstein distance
of order one between the obtained model and the target measure. However, we
argue that such cases are possible especially if we are using a poor estimation
of the score, \ie a large value of $\Mtt$ in \rref{assum:score_control}. Indeed,
neglecting the discretization error and setting $\beta = 1$ for simplicity, the
backward process is given by
  \begin{equation}
    \label{eq:bound_init}    
  \textstyle{
    \rmd \hat{\bfY}_t = \{\hat{\bfY}_t + 2 \bm{s}(T-t, \hat{\bfY}_t) \} \rmd t + \sqrt{2} \rmd \bfB_t \eqsp , \qquad \hat{\bfY}_0 \sim \pi_\infty = \mathrm{N}(0, \Id) \eqsp . 
    }
  \end{equation}Assume that $\bm{s}=0$ (which is approximately the case at initialization if we parameterize $\bm{s}$ with a neural network with a fully connected last layer and no non-linearity), then the dynamics becomes
  \begin{equation}
    \label{eq:explosive_ou}
  \textstyle{
    \rmd \hat{\bfY}_t = \hat{\bfY}_t \rmd t + \sqrt{2} \rmd \bfB_t \eqsp , \qquad \hat{\bfY}_0 \sim \pi_\infty = \mathrm{N}(0, \Id) \eqsp . 
    }
  \end{equation}
In that case, for any $t \in \ccint{0,T}$, $\hat{\bfY}_t$ is Gaussian and we have that $\mathcal{L}(\bhfY_t) = \mathrm{N}(0, ((3\exp[2t]-1)/2)\Id)$. In addition, we have that for any $f: \ \rset^d \to \rset$ which is $1$-Lipschitz we have
  \begin{equation}
    \wassersteinD[1](\mathcal{L}(\bhfY_T), \pi) \geq \textstyle{\expeLigne{f(\bhfY_t)} - \int_{\rset^d} f(x) \rmd \pi(x) \eqsp .}
  \end{equation}
  Choosing $f(x) = \absLigne{x_1}$ for any $x=(x_1, \dots, x_d) \in \rset^d$ in the previous inequality we get,
  \begin{equation}
    \wassersteinD[1](\mathcal{L}(\bhfY_T), \pi) \geq (3\exp[2T]-1)^{1/2} - \diam(\M) \eqsp .
  \end{equation}Therefore, choosing $T \geq 0$ large enough and using \eqref{eq:bound_init}, we get that $\wassersteinD[1](\mathcal{L}(\bhfY_T), \pi) \geq \wassersteinD[1](\mathcal{L}(\bhfY_0), \pi)$. This result implies that even in idealized setting, the backward process might steer the Gaussian distribution \emph{away} from the target distribution $\pi$. This is due to the \emph{explosive} property of the Ornstein-Uhlenbeck process \eqref{eq:explosive_ou} which should be compared to the \emph{contractive} behavior of the forward Ornstein-Uhlenbeck process \eqref{eq:ornstein_ulhenbeck}.


\section{Assumptions on the schedule}
\label{sec:assumptions-schedule}

In what follows, we consider three schedules commonly used in practice:
\begin{enumerate*}[label=(\alph*)]
\item the constant schedule,
\item the linear schedule,
\item the cosine schedule.
\end{enumerate*}
We show that \Cref{assum:assumption_beta} is satisfied in all these cases. We
consider a generalized version of the cosine schedule which makes it
differentiable by replacing the hard clamping by a soft version with level $r>0$
(note that letting $r \to 0$ we recover the original cosine schedule). The
constant schedule is defined by $\beta_s = \beta_0$ for all $s \in
\ccint{0,T}$. The linear schedule was introduced in \citet{ho2020denoising} and
is defined by $\beta_s = \beta_0 + (\beta_T - \beta_0)t/T$ with
$\beta_T > \beta_0 >0$. Finally, the cosine schedule was introduced in
\cite[Equation (17)]{nichol2021improved} in discrete-time and can be defined as
follows in continuous-time
\begin{equation}
  \beta_t = \softmin_{r}(1,\lim_{h \to 0} (\bar{\alpha}_{t-h} - \bar{\alpha}_t)/(\bar{\alpha}_{t-h} h))  = \softmin_{r}(1,f)_t \eqsp , \qquad f(t)= - \bar{\alpha}'_t / \bar{\alpha}_t \eqsp ,
\end{equation}
with $\bar{\alpha}$ defined as
\begin{equation}
  \bar{\alpha}_t  = \cos((t/T + \eta)/(1+\eta)(\uppi/2))^2/\cos(\eta/(1+\eta)(\uppi/2))^2 \eqsp ,
\end{equation}
and where $\eta \geq 0$, $r >0$ are parameters and for any $f_1, f_2 : \ \ccint{0,T} \to \rset_+$ 
\begin{equation}
  \softmin_r(f_1,f_2)_t = -r\log(\exp[-f_1(t)/r] + \exp[-f_2(t)/r]) \eqsp . 
\end{equation}
In the special case where $f_1 = 1$ we have
\begin{equation}
  \softmin_r(1,f_2)_t = 1 -r \log(1 + \exp[(1-f_2(t))/r]) \eqsp ,
\end{equation}
We have 
\begin{equation}
  \bar{\alpha}'(t)/\bar{\alpha}(t) = -\uppi/(T(1+\eta)) \tan((t/T + \eta)/(1+\eta)(\uppi/2)) \eqsp . 
\end{equation}
In particular, $t \mapsto \beta_t$ is increasing and bounded above and below on $\ccint{0,T}$.

Finally, we end this section by remarking that if one aims at studying the
Euler-Maruyama discretization of the approximate backward, \ie the process
given by \eqref{eq:classical_EM_disc} then one needs to also assume some
Lipschitz property on the schedule $s \mapsto \beta_s$. 


\section{A short proof of the results of \cite{franzese2022much}}
\label{sec:short-proof}

In \cite[Equation (9)]{franzese2022much}, the authors show that (under mild
regularity assumptions\footnote{We assume that all probability measures admit
  densities \wrt the Lebesgue measure and that all the integrals we consider
  are well-defined.})
\begin{equation}
  \label{eq:franzese_res}
  \textstyle{
    \int_{\rset^d} \log p_{\theta,T}(x) p_0(x) \rmd x \geq \int_{\rset^d} \log p_{0}(x) p_0(x) \rmd  x - \mathcal{G}(\bm{s}_\theta,T) - \KLLigne{p_T}{p_\infty} \eqsp .
    }
  \end{equation}
  To do so, they rearrange the ELBO result from \cite{huang2021variational}. We
  have that \eqref{eq:franzese_res} is equivalent to
  \begin{equation}
    \label{eq:KL_ineq}
    \KLLigne{p_0}{p_{\theta, T}} \leq \mathcal{G}(\bm{s}_\theta,T) + \KLLigne{p_T}{p_\infty} \eqsp . 
  \end{equation}
  The definition of $\mathcal{G}(\bm{s}_\theta,T)$ is given by
  \begin{align}
    \textstyle{
      \mathcal{G}(\bm{s}_\theta,T)} &= \textstyle{(1/2) (\int_0^T  \beta_t^2 \expeLigne{\normLigne{\bm{s}_\theta(t,\bfX_t) - \nabla \log p_{t|0}(\bfX_t|\bfX_0)}^2 } \rmd t} \\
    & \qquad \textstyle{- \int_0^T  \beta_t^2 \expeLigne{\normLigne{\nabla \log p_t(\bfX_t) - \nabla \log p_{t|0}(\bfX_t|\bfX_0)}^2 })} \eqsp . 
  \end{align}
  Developing the square and using that
  $\CPELigne{\nabla \log p_{t|0}(\bfX_t|\bfX_0)}{\bfX_t} = \nabla \log
    p_t(\bfX_t)$, we get that
  \begin{equation}
    \textstyle{ \mathcal{G}(\bm{s}_\theta,T)} = \int_0^T \beta_t^2 \expeLigne{\normLigne{\bm{s}_\theta(t,\bfX_t) - \nabla \log p_{t}(\bfX_t)}^2} \rmd t  \eqsp .
  \end{equation}
  Hence, combining this result and \eqref{eq:KL_ineq}, we have that
  \eqref{eq:franzese_res} is equivalent to \cite[Theorem 1]{durkan2021maximum}
  which is obtained upon combining the data-processing inequality, the
  decomposition of the Kullback-Leibler via conditioning and the Girsanov
  theorem.


\section{Wasserstein controls under $\mathrm{L}^2$ errors}
\label{sec:wass-contr-under}

In this section, we replace the assumption \Cref{assum:score_control} by the
following weaker control.

\begin{assumption}
    \label{assum:score_control_L2}
    There exist $\bm{s} \in \rmc(\ccint{0,T} \times \rset^d, \rset^d)$ and
    $\Mtt \geq 0$ such that for any $k \in \{0, \dots, K\}$ and $x_t \in \rset^d$,
  \begin{equation}
    \expeLigne{\normLigne{\bm{s}(T-t_k, Y_k) - \nabla \log p_{T-t_k}(Y_k)}^2} \leq \Mtt^2 \expeLigne{(1+\normLigne{Y_k}^2)} / \sigma_{T-t_k}^4 \eqsp ,
  \end{equation}
  where we recall that $(Y_k)_{k \in \{0, \dots, K\}}$ is given by \eqref{eq:discretization_improved}.
\end{assumption}

Note that this assumption is different from the one of
  \cite{lee2022convergence} as the expectation is considered
  w.r.t. $\{Y_k\}_{k=0}^K$ and not $\{\bfY_{t_k}\}_{k=0}^K$. In order to control
  the $\mathrm{L}^2$ error, \cite{lee2022convergence} use a change of measure and
  control the $\chi^2$ divergence between the density of $Y_k$ and the one
  $\bfY_{t_k}$ for any $k \in \{0, \dots, K\}$. These controls are obtained
  using a logarithmic Sobolev assumption on the target measure $\pi$. Adapting
  these results to our Wasserstein distance setting is not straightforward and
  is left for future work.  Under \rref{assum:score_control_L2}, we have the following
theorem, which is an extension of \Cref{thm:convergence_general}. To prove this
theorem, we extend \cite[Theorem 4.1]{lee2022convergence} to the Wasserstein
distance of order one and weaker growth conditions.

\begin{theorem}
  \label{thm:convergence_general_L2}
  Assume \rref{assum:manifold_hyp}, \rref{assum:assumption_beta},
  \rref{assum:step_size}, \rref{assum:score_control_L2} that
  $T \geq 2\bar{\beta}(1 + \log(1+\diam(\M))$, $t_K = T - \vareps$ and
  $\vareps, \Mtt, \Mtt/\zeta, \delta \leq 1/32$.
   Then, there exists $\Dtt_0 \geq 0$ such that
   \begin{equation}
         \wassersteinD[1](\mathcal{L}(Y_K), \pi) \leq \Dtt_0 (K \zeta + \exp[\kappa/\vareps] (\Mtt/\zeta + \delta^{1/2})/ \vareps^2 + \exp[\kappa/\vareps]\exp[-T/\bar{\beta}] + \vareps^{1/2}) \eqsp ,
       \end{equation}
       with $\kappa = \diam(\M)^2(1+\bar{\beta})/2$ and
       \begin{equation}
         \label{eq:constant_diam_app}         
         \Dtt_0 = D (1 + \bar{\beta})^5(1 + d + \diam(\M)^4) (1 + \log(1 + \diam(\M))) \eqsp ,
       \end{equation}
       and $D$ is a numerical constant.
     \end{theorem}

     We start with the following lemma, which is an extension of
     \Cref{lemma:control_growth_discrete_process} to the setting where
     \rref{assum:score_control} is replaced by \rref{assum:score_control_L2}. 

     \begin{lemma}
  \label{lemma:control_growth_discrete_process_weak}
  Assume \rref{assum:manifold_hyp}, \rref{assum:assumption_beta} and
  \rref{assum:score_control_L2}. Assume that there exists $\delta >0$ such that for
  any $k \in \{0, \dots, K\}$,
  $\gamma_k \beta_{T-t_k} / \sigma_{T-t_k}^2 \leq \delta$. Assume that there
  exists $\eta >0$ such that $A(\delta, \Mtt, \eta, \diam(\M)) > 0$ with
    \begin{align}
      &A(\delta, \Mtt, \eta, \diam(\M)) = 2 - 2\delta - 32 \delta (1
    + \Mtt^2) - 8 \Mtt -
    4\eta \diam(\M) \eqsp , \\
     & B(\delta, \Mtt, \eta, \diam(\M)) = 32\delta (\Mtt^2 +
    \diam(\M)^2) + 2 
    (1+\delta)(\diam(\M)/\eta + \Mtt) + 4 d \eqsp . 
    \end{align}
  Then, we have for any $k \in \{0, \dots, K\}$
  \begin{equation}
    \label{eq:Ktt_def_l2}
      \expeLigne{\normLigne{Y_k}^2} \leq \Ktt = d + B(\delta, \Mtt, \eta, \diam(\M))(1/A(\delta, \Mtt, \eta, \diam(\M)) + \delta) \eqsp .
    \end{equation}
    In particular if $\Mtt \leq 1/32$ and $\delta \leq 1/32$ then for any $k \in \{0, \dots, K\}$
    \begin{equation}
      \expeLigne{\normLigne{Y_k}^2} \leq \Ktt_0 = 5d + 320 (1+ \diam(\M))^2 \eqsp . 
    \end{equation}
  \end{lemma}

  \begin{proof}
    Recall that using \eqref{eq:discretization_improved}, we have that for any $k \in \{0, \dots, K-1\}$
    \begin{equation}
      \label{eq:discretization_improved_deet}
    \textstyle{
      Y_{k+1} = Y_k + (\exp[\int_{T-t_{k+1}}^{T-t_k} \beta_s \rmd s] - 1) (Y_{k} + 2 \bm{s}(T-t_k, Y_{k})) + (\exp[2 \int_{T-t_{k+1}}^{T-t_k} \beta_s \rmd s] - 1)^{1/2} Z_k \eqsp ,
      }
    \end{equation}
    For simplicity, we denote
    \begin{equation}
      \textstyle{ \gamma_{1,k} = (\exp[\int_{T-t_{k+1}}^{T-t_k} \beta_s \rmd s] - 1)/\beta_{T-t_k} \eqsp , \qquad \gamma_{2,k} = (\exp[2 \int_{T-t_{k+1}}^{T-t_k} \beta_s \rmd s] - 1)/(2\beta_{T-t_k}) \eqsp . }
    \end{equation}
    Then, \eqref{eq:discretization_improved_deet} can be rewritten for any $k \in \{0, \dots, K-1\}$ as
    \begin{equation}
      \label{eq:discretization_improved_deet_duo}
          \textstyle{
      Y_{k+1} = Y_k + \gamma_{1,k} \beta_{T-t_k} (Y_{k} + 2 \bm{s}(T-t_k, Y_{k})) + \sqrt{2 \gamma_{2,k} \beta_{T-t_k}} Z_k \eqsp .
      }
    \end{equation}
    In what follows, we denote $\bar{\gamma}_{1,k} = \gamma_{1,k} \beta_{T-t_k}$
    and $\bar{\gamma}_{2,k} = \gamma_{2,k} \beta_{T-t_k}$. In addition, using
    that $\gamma_k \beta_{T-t_k} \leq \delta \leq 1/4$, we have that
    $\gamma_{1,k} \leq \gamma_{2,k} \leq 2 \gamma_{1,k}$.
    Indeed, we have that for any $k \in \{0, \dots, K-1\}$
    \begin{align}
      &\textstyle{\gamma_{2,k} / \gamma_{1,k} = (1/2) (\exp[\int_{T-t_{k+1}}^{T-t_k} \beta_s \rmd s] + 1) \geq 1 \eqsp ,} \\ 
        &\textstyle{\gamma_{2,k} / \gamma_{1,k} = (1/2) (\exp[\int_{T-t_{k+1}}^{T-t_k} \beta_s \rmd s] + 1) \leq (1/2)(\exp[\gamma_k \beta_{T-t_k}] + 1) \leq 2 \eqsp ,}
    \end{align}
    In what follows, for any $t\in \ccint{0,T}$ and $x_t \in \rset^d$, we denote
    $\Delta_t = \normLigne{\bm{s}(t,x_t) - \nabla \log p_t(x_t)}$.  Using
    \Cref{lemma:control_grad}, we have that for any $t \in \ccint{0,T}$,
    $x_t \in \rset^d$ and $\eta >0$
    \begin{align}
      \langle x_t, \bm{s}(t, x_t) \rangle &\leq -\normLigne{x_t}^2/\sigma_t^2 + m_t \diam(\M)\normLigne{x_t}/\sigma_t^2 +  \Delta_t(x_t) \normLigne{x_t}  \\
      &\leq (-1 + \eta  m_t \diam(\M))\norm{x_t}^2/\sigma_t^2 + (m_t \diam(\M)/\eta) / \sigma_t^2 +  \Delta_t(x_t) \normLigne{x_t} \eqsp , \label{eq:s_scalar_product}
    \end{align}
    where we have used that for any $a,b \geq 0$, $2ab \leq \eta a^2 + b^2/\eta$
    in the last line.  In addition, using \Cref{lemma:control_grad}, for any
    $t \in \ccint{0,T}$ and $x_t \in \rset^d$ we have
    \begin{align}      
      \normLigne{\bm{s}(t, x_t)}^2 &\leq 2 \normLigne{\bm{s}(t, x_t) - \nabla \log p_t(x_t)}^2 + 2 \normLigne{\nabla \log p_t(x_t)}^2 \\
                                   &\leq 2\Delta_t(x_t)^2 + 4 \normLigne{x_t}^2/\sigma_t^4 + 4 m_t^2 \diam(\M)^2/\sigma_t^4 \eqsp , \label{eq:s_norm_control}
    \end{align}
    In addition, using \rref{assum:score_control_L2}, the Cauchy-Schwarz
    inequality and \eqref{eq:s_scalar_product} we have
    \begin{align}
      &\expeLigne{\langle Y_k, \bm{s}(T-t_k, Y_k)\rangle } \leq (-1 + \eta  m_{T-t_k} \diam(\M))\expeLigne{\norm{Y_k}^2}/\sigma_{T-t_k}^2 + (m_{T-t_k} \diam(\M)/\eta) / \sigma_{T-t_k}^2 \\
                                                                  & \qquad \qquad + \expeLigne{\Delta_{T-t_k}(Y_k) \normLigne{Y_k}} \\ 
                                                                    &\leq (-1 + \eta  m_{T-t_k} \diam(\M))\expeLigne{\norm{Y_k}^2}/\sigma_{T-t_k}^2 + (m_{T-t_k} \diam(\M)/\eta) / \sigma_{T-t_k}^2 \\
                                                                  & \qquad + \sqexpeLigne{ \Delta_{T-t_k}(Y_k)^2} \sqexpeLigne{\normLigne{Y_k}^2} \\ 
&\leq (-1 + \eta  m_{T-t_k} \diam(\M))\expeLigne{\norm{Y_k}^2}/\sigma_{T-t_k}^2 + (m_{T-t_k} \diam(\M)/\eta) / \sigma_{T-t_k}^2 \\
                                                                  & \qquad + \Mtt (1 + \sqexpeLigne{\normLigne{Y_k}^2}) \sqexpeLigne{\normLigne{Y_k}^2}/\sigma_{T-t_k}^2 \\
      &\leq (-1 + \eta  m_{T-t_k} \diam(\M)+\Mtt)\expeLigne{\norm{Y_k}^2}/\sigma_{T-t_k}^2 + (m_{T-t_k} \diam(\M)/\eta) / \sigma_{T-t_k}^2  + \Mtt \sqexpeLigne{\normLigne{Y_k}^2}/\sigma_{T-t_k}^2 \\
&\leq (-1 + \eta  m_{T-t_k} \diam(\M)+2 \Mtt)\expeLigne{\norm{Y_k}^2}/\sigma_{T-t_k}^2 + (m_{T-t_k} \diam(\M)/\eta + \Mtt) / \sigma_{T-t_k}^2   \eqsp . \label{eq:s_scalar_product_esperance}
    \end{align}
    Finally, using \rref{assum:score_control_L2} and \eqref{eq:s_norm_control} we have
    \begin{align}
      \expeLigne{\normLigne{\bm{s}(T-t_k, Y_k)}^2} &\leq 4\expeLigne{\Delta_{T-t_k}(Y_k)^2} + 4 \expeLigne{\normLigne{Y_{k}}^2}/\sigma_{T-t_k}^4 + 4 m_{T-t_k}^2 \diam(\M)^2/\sigma_{T-t_k}^4 \\
                                                   & \leq  4\Mtt^2 (1 + \expeLigne{\normLigne{Y_k}^2})/\sigma_{T-t_k}^4 + 4 \expeLigne{\normLigne{Y_{k}}^2}/\sigma_{T-t_k}^4 + 4 m_{T-t_k}^2 \diam(\M)^2/\sigma_{T-t_k}^4 \\
      & \leq  4 (1 + \Mtt^2) \expeLigne{\normLigne{Y_k}^2})/\sigma_{T-t_k}^4 + 4 (\Mtt^2 + m_{T-t_k}^2 \diam(\M)^2)/\sigma_{T-t_k}^4 \eqsp . \label{eq:s_norm_control_esperance}
    \end{align}
    
    Combining \eqref{eq:discretization_improved_deet_duo},
    \eqref{eq:s_scalar_product_esperance} and \eqref{eq:s_norm_control_esperance} we have for any
    $k \in \{0, \dots, K-1\}$
    \begin{align}
      \expeLigne{\normLigne{Y_{k+1}}^2} &= (1 + \bar{\gamma}_{1,k})^2 \expeLigne{\normLigne{Y_{k}}^2} + 4 \bar{\gamma}_{1,k}^2 \expeLigne{\normLigne{\bm{s}(T-t_k, Y_k)}^2} \\
      & \qquad +4 \bar{\gamma}_{1,k} (1 + \bar{\gamma}_{1,k}) \expeLigne{\langle Y_k, \bm{s}(T-t_k, Y_k)\rangle} + 2\bar{\gamma}_{2,k}d \\
                                        &\leq (1 + 2\bar{\gamma}_{1,k} + \bar{\gamma}_{1,k}^2) \expeLigne{\normLigne{Y_k}^2} + 16 (\bar{\gamma}_{1,k} / \sigma_{T-t_k}^2)^2 (1 + \Mtt^2)\expeLigne{\normLigne{Y_k}^2}  \\
                                        & \qquad + 16(\bar{\gamma}_{1,k} / \sigma_{T-t_k}^2)^2 (\Mtt^2 + m_{T-t_k}^2 \diam(\M)^2 ) \\
      &\qquad + 4 \bar{\gamma}_{1,k} (1 + \bar{\gamma}_{1,k}) \expeLigne{\langle Y_k, \bm{s}(T-t_k, Y_k)\rangle} + 4\bar{\gamma}_{1,k}d\\
                                        &\leq (1 + 2\bar{\gamma}_{1,k} + \bar{\gamma}_{1,k}^2) \expeLigne{\normLigne{Y_k}^2} + 16 (\bar{\gamma}_{1,k} / \sigma_{T-t_k}^2)^2 (1 + \Mtt^2)\expeLigne{\normLigne{Y_k}^2}  \\
                                        & \qquad + 16(\bar{\gamma}_{1,k} / \sigma_{T-t_k}^2)^2 (\Mtt^2 + m_{T-t_k}^2 \diam(\M)^2 ) \\
      & \qquad + 4 (\bar{\gamma}_{1,k}/\sigma_{T-t_k}^2) (1 + \bar{\gamma}_{1,k}) (-1 + 2 \Mtt + \eta m_{T-t_k} \diam(\M))\expeLigne{\norm{Y_k}^2} \\
                                        & \qquad + (\bar{\gamma}_{1,k}/ \sigma_{T-t_k}^2 ) (1 + \bar{\gamma}_{1,k}) (m_{T-t_k} \diam(\M)/\eta + \Mtt) + 4 \bar{\gamma}_{1,k}d \eqsp . 
    \end{align}
    The rest of the proof is identical to the one of \Cref{lemma:control_growth_discrete_process}.
  \end{proof}

  Let $\zeta>0$. For any $k \in \{0, \dots, K\}$ we define $\msa_k$ such that
  \begin{equation}
    \msa_k = \ensembleLigne{y \in \rset^d}{\normLigne{\bm{s}(T-t_k, y) - \nabla \log p_{T-t_k}(y)} > (\Mtt/ \zeta) / \sigma_{T-t_k}^2} \eqsp .
  \end{equation}
  We define the process $(Y_k^\star)_{k \in \{0, \dots, K\}}$ such that
  $Y_0^\star = Y_0$ and for any $k \in \{0, \dots, K-1\}$, if $Y_k = Y_k^\star$
  and $Y_k \in \msa_k$  then $Y_{k+1}^\star = Y_{k+1}$. Otherwise, we define 
  \begin{equation}
    \label{eq:discretization_improved_inter}
    \textstyle{
      Y^\star_{k+1} = Y^\star_k + (\exp[\int_{T-t_{k+1}}^{T-t_k} \beta_s \rmd s] - 1) (Y^\star_{k} + 2 \nabla \log p_{T-t_k} (Y^\star_{k})) + (\exp[2 \int_{T-t_{k+1}}^{T-t_k} \beta_s \rmd s] - 1)^{1/2} Z_k \eqsp .
      }
  \end{equation}
  This is similar to assuming that there exists $\bm{s}^\star$ such that for any $k \in \{0, \dots, K-1\}$
    \begin{equation}
    \label{eq:discretization_improved_inter}
    \textstyle{
      Y^\star_{k+1} = Y^\star_k + (\exp[\int_{T-t_{k+1}}^{T-t_k} \beta_s \rmd s] - 1) (Y^\star_{k} + 2 \bm{s}^\star (T-t_k, Y^\star_{k})) + (\exp[2 \int_{T-t_{k+1}}^{T-t_k} \beta_s \rmd s] - 1)^{1/2} Z_k \eqsp ,
      }
  \end{equation}
  with $\bm{s}^\star$ which satisfies\footnote{Note that we slightly abuse since
    $\bm{s}^\star$ is random (depending on the behavior of $Y_k$) but one can
    check that all our proofs remain unchanged in this slightly larger setting}
  \rref{assum:score_control} with $\Mtt$ replaced by $\Mtt/\zeta$ and while $Y_k \in \msa_k$,
  $Y_{k+1}^\star = Y_{k+1}$.

  We have the following lemma which is an extension of \cite[Theorem
  4.1]{lee2022convergence} to the Wasserstein setting. Note that contrary to
  \cite[Theorem 4.1]{lee2022convergence} which states results in total variation
  we also need control on the moments of the distribution under a $\mathrm{L}^2$
  error, which is precisely \Cref{lemma:control_growth_discrete_process_weak}.

  \begin{lemma}
    \label{lemma:wasserstein_l2_linf}
    Assume \rref{assum:manifold_hyp}, \rref{assum:assumption_beta} and
    \rref{assum:score_control_L2}. Assume that there exists $\delta >0$ such
    that for any $k \in \{0, \dots, K\}$,
    $\gamma_k \beta_{T-t_k} / \sigma_{T-t_k}^2 \leq \delta$ and that
    $\Mtt, \Mtt/\zeta, \delta \leq 1/32$.  Then, we have for any
    $k \in \{0, \dots, K\}$
    \begin{equation}
      \expeLigne{\normLigne{Y_k^\star - Y_k}} \leq 4 (1+ \Ktt_0) \zeta k \eqsp ,
    \end{equation}
    where $\Ktt_0$ is defined in \Cref{lemma:control_growth_discrete_process_weak}.
  \end{lemma}

  \begin{proof}
    Using the Cauchy-Schwarz inequality we have
    \begin{align}
      \expeLigne{\normLigne{Y_k^\star - Y_k}} &= \expeLigne{\normLigne{Y_k^\star - Y_k}\1_{Y_k \neq Y_k^\star}} \\
                                              &\leq \textstyle{ \sqrt{2}(\sqexpeLigne{\normLigne{Y_k^\star}^2} + \sqexpeLigne{\normLigne{Y_k}^2}) (\sum_{j=1}^k \probaLigne{Y_j \in \msa_j})^{1/2} } \\
                                              &\leq \textstyle{ \sqrt{2}(\sqexpeLigne{\normLigne{Y_k^\star}^2} + \sqexpeLigne{\normLigne{Y_k}^2})} \\
      & \qquad \times \textstyle{ (\sum_{j=1}^k \probaLigne{\normLigne{\bm{s}(T-t_j, Y_j) - \nabla \log p_{T-t_j}(Y_j)} > (\Mtt/ \zeta) / \sigma_{T-t_j}^2})^{1/2} } \eqsp . \label{eq:cauchy_schwarz_lee}
    \end{align}
    Using the Markov inequality, we have for any $j \in \{0, \dots, K\}$
    \begin{align}
      &\probaLigne{\normLigne{\bm{s}(T-t_j) - \nabla \log p_{T-t_j}(Y_j)} > (\Mtt/ \zeta)  / \sigma_{T-t_j}^2} \\
      & \qquad  \leq \expeLigne{\normLigne{\bm{s}(T-t_j, Y_j) - \nabla \log p_{T-t_j}(Y_j)}^2} \sigma_{T-t_j}^4 \zeta^2 / \Mtt^2 \leq \zeta^2 \expeLigne{1 +\normLigne{Y_j}^2} \eqsp . 
    \end{align}
    Therefore, combining this result, \eqref{eq:cauchy_schwarz_lee} and
    \Cref{lemma:control_growth_discrete_process_weak} we have
    $\expeLigne{\normLigne{Y_k^\star - Y_k}} \leq 4 (1 + \Ktt_0) \zeta k$.
  \end{proof}

We are now ready to complete the proof of \Cref{thm:convergence_general_L2}
  
\begin{proof}
  We have
  \begin{equation}
    \label{eq:ineq_triangle_l2_linf}
    \wassersteinD[1](\mathcal{L}(Y_K), \pi) \leq \wassersteinD[1](\mathcal{L}(Y_K), \mathcal{L}(Y_k^\star)) + \wassersteinD[1](\mathcal{L}(Y_k^\star), \pi) \eqsp . 
  \end{equation}
  Note that using \Cref{thm:convergence_general} we have
  \begin{equation}
   \wassersteinD[1](\mathcal{L}(Y_k^\star), \pi) \leq  \Dtt_0 (\exp[\kappa/\vareps] (\Mtt/ \zeta + \delta^{1/2})/ \vareps^2 + \exp[\kappa/\vareps]\exp[-T/\bar{\beta}] + \vareps^{1/2}) \eqsp .  \label{eq:convergence_Yk_star}
 \end{equation}
 In addition, using \Cref{lemma:wasserstein_l2_linf} we have
 \begin{equation}
   \wassersteinD[1](\mathcal{L}(Y_K), \mathcal{L}(Y_k^\star)) \leq 4 (1+ \Ktt_0) \zeta K \eqsp .
 \end{equation}
 Combining this result and \eqref{eq:convergence_Yk_star} in
 \eqref{eq:ineq_triangle_l2_linf} concludes the proof.
\end{proof}


\section{Improved bounds under Hessian conditions}
\label{sec:impr-bounds-under}

In \Cref{sec:proof-convergence-hessian}, we prove
\Cref{thm:convergence_bound_hessian} which is an improvement upon
\Cref{thm:convergence_general} under tighter conditions on the Hessian
$\nabla^2 \log p_t$. In \Cref{sec:hess-bounds-unif}, we show that this condition
is satisfied in the case of a uniform measure over $\ccint{-1/2, 1/2}^p$ for
some $p \in \{1, \dots, d\}$. Finally, in \Cref{sec:non-conv-count}, we show,
under appropriate smoothness conditions, that the condition is never satisfied
on non-convex sets.

\subsection{Proof of \Cref{thm:convergence_bound_hessian}}
\label{sec:proof-convergence-hessian}

In this section, we prove \Cref{thm:convergence_bound_hessian}. We start by
deriving an improvement on \Cref{prop:control_gradient}. The main difference
between \Cref{prop:control_gradient} and \Cref{prop:control_gradient_improved}
lies into the dependency \wrt $\sigma_{T-t_K}^{-2}$. In
\Cref{prop:control_gradient}, we have an exponential dependency
$\exp[(\diam(\M)^2/2)\sigma_{T-t_K}^{-2}]$ whereas in \Cref{prop:control_gradient_improved}, we
have a polynomial dependency $\sigma_{T-t_K}^{-2\Gamma}$.

For ease of notation we introduce the following assumption.

\begin{assumption}
  \label{assum:hessian_bound}
  There exists $\Gamma \geq 0$ such that for any $t \in \ocint{0,T}$ and $x_t \in \rset^d$, $\normLigne{\nabla^2 \log p_t(x_t)} \leq \Gamma / \sigma_t^2$.
\end{assumption}

We start with the following proposition.

\begin{proposition}
  \label{prop:control_gradient_improved}
  Assume \rref{assum:manifold_hyp}, \rref{assum:hessian_bound} and that
  $T \geq 2\bar{\beta}(1 + \log(1+\diam(\M))$.  Let $t_K \in \coint{0,
    T}$. Then, for any $t\in \ccint{0,t_K}$ and $x \in \rset^d$ we have
  \begin{equation}
    \textstyle{
      \normLigne{\nabla \bfY_{t,t_K}^x} \leq \textstyle{ \exp[-(1/2) \int_{T-t^\star}^{T-t} \beta_s \rmd s \1_{\coint{0,t^\star}}(t)] \sigma_{T-t_K}^{-2\Gamma} \exp[(\Gamma +1) \int_{T -t_K}^{T-t^\star} \beta_{u} \rmd u]\eqsp . 
      }}
    \end{equation}
  \end{proposition}

  \begin{proof}
    Let $x\in \rset^d$. First, using \eqref{eq:tangent_process} and \Cref{lemma:control_hessian} we have that for any $s, t\in \ccint{0,T}$ with $s \leq t$
    \begin{equation}
      \rmd \normLigne{\nabla \bfY_{s,t}^x}^2 \leq 2\beta_{T-t}(\normLigne{\nabla \bfY_{s,t}^x }^2 - 2(1  -  m_{T-t}^2 \diam(\M)^2 / (2\sigma_{T-t}^2) )/\sigma_{T-t}^2 \normLigne{\nabla \bfY_{s,t}^x}^2) \rmd t  \eqsp . 
    \end{equation}
    First, assume that $s \leq t^\star$ and that $t \geq t^\star$. In that case, using
    \Cref{lemma:growth_tangent_process} we have that
      \begin{equation}
\textstyle{    \int_s^{t^\star}\beta_{T-u} (1 -2/\sigma_{T-u}^2 + m_{T-u}^2\diam(\M)^2/\sigma_{T-u}^4) \rmd u  \leq -(1/2) \int_s^{t^\star} \beta_{T-u} \rmd u  \eqsp . }
\end{equation}
Therefore, using that result and the fact that $\nabla \bfY_{s,s}^x = \Id$ we get that
\begin{equation}\label{eq:inter_s_lower_imp_i}
  \textstyle{\normLigne{\nabla \bfY_{s,t^\star}^x} \leq \exp[-(1/2)\int_{T-t^\star}^{T-s} \beta_u \rmd u] \eqsp . }
\end{equation}
In addition, using that for any $t \in \ocint{0,T}$ and $x_t \in \rset^d$, $\normLigne{\nabla^2 \log p_t(x_t)} \leq \Gamma /\sigma_t^2$ we have
\begin{equation}
  \rmd \normLigne{\nabla \bfY_{s,t}^x}^2 \leq 2 \beta_{T-t} (1 + 2 \Gamma / \sigma_{T-t}^2) \normLigne{\nabla \bfY_{s,t}^x}^2 \rmd t \eqsp . 
\end{equation}
In addition, using \Cref{lemma:integration_lemma} we have that
\begin{align}
  \textstyle{    \int_{t^\star}^{t}\beta_{T-u} (1 + 2 \Gamma / \sigma_{T-u}^2) \rmd u}  & \textstyle{\leq \Gamma[\log(\exp[2\int_0^{T-t^\star} \beta_{T-u}\rmd u ] - 1)} \\ & \qquad  \textstyle{- \log(\exp[2\int_0^{T-t} \beta_{T-u} \rmd u ] - 1)] + \int_{T -t}^{T-t^\star} \beta_{u} \rmd u } \\
                                                                                        & \textstyle{\leq \Gamma[\log(\sigma_{T-t^\star}^2) - \log(\sigma_{T-t}^2)] +(\Gamma +1) \int_{T -t}^{T-t^\star} \beta_{u} \rmd u  } \\
  & \textstyle{\leq  - \Gamma \log(\sigma_{T-t}^2) +(\Gamma +1)  \int_{T -t}^{T-t^\star} \beta_{u} \rmd u  \eqsp. }
  \end{align}
  Therefore, combining this result and \eqref{eq:inter_s_lower_imp_i}, we get
  that
  \begin{align}
    \normLigne{\nabla \bfY_{s,t}} &\leq \textstyle{ \sigma_{T-t}^{-2\Gamma} \exp[(\Gamma +1) \int_{T -t}^{T-t^\star} \beta_{u} \rmd u] \normLigne{\nabla \bfY_{s,t^\star}}} \\
                                  &\leq \textstyle{ \sigma_{T-t}^{-2\Gamma} \exp[(\Gamma +1) \int_{T -t}^{T-t^\star} \beta_{u} \rmd u] } \exp[-(1/2)\int_{T-t^\star}^{T-s} \beta_u \rmd u] \eqsp .
  \end{align}
        The proof in the cases where $s \geq t^\star$, $t \geq t^\star$ and
        $s \leq t^\star$, $t \leq t^\star$ are similar and left to the reader.
      \end{proof}

      The rest of the proof follows the proof of
      \Cref{sec:proof_theorem_one}. The following proposition is the counterpart
      of \Cref{prop:discretization_bound_final}. Again note that the exponential
      dependency \wrt $1/\vareps$ has been replaced by a polynomial dependency.

      \begin{proposition}
  \label{prop:discretization_bound_final_improved}.
  Assume \rref{assum:manifold_hyp}, \rref{assum:assumption_beta},
  \rref{assum:score_control}, \rref{assum:step_size}, \rref{assum:hessian_bound} and $t_K = T - \vareps$. In
  addition, assume that $\vareps, \delta, \Mtt \leq 1/32$.  Then
  \begin{equation}
      \wassersteinD[1](\pi_\infty \rmQ_{t_K}, \pi_\infty \rmR_{K})  \leq  \Dtt_0 (\Mtt + \delta^{1/2})/ \vareps^{\Gamma +2}  \eqsp ,
  \end{equation}
  where
  \begin{equation}
    \Dtt_0 = 4 (4 + 256 d + 43664 (1+\diam(\M))^4)\exp[3(1+\bar{\beta})^2(\Gamma+2)(1+\log(1+\diam(\M))))] \eqsp . 
  \end{equation}
\end{proposition}

\begin{proof}
  Using \Cref{prop:pierre-extended}, we have
    \begin{equation}
\textstyle{    \normLigne{\bfY_{t_K} - Y_K} \leq \int_0^{t_K} \normLigne{\nabla \bfY_{u,t_K}(\bbfY_{0,u})} \normLigne{ \Delta b_u((\bbfY_{0,v})_{v\in \ccint{0,T}})} \rmd u } \eqsp .
\end{equation}
Combining this result, recalling that $t^\star$ is defined in
\eqref{eq:def_t_star} and \Cref{prop:control_gradient_improved}, we get
    \begin{align}
      \textstyle{    \normLigne{\bfY_{t_K} - Y_K}} \leq & \textstyle{ \int_0^{t_K} \exp[-(1/2) \int_{T-t^\star}^{T-u} \beta_s \rmd s \1_{\coint{0,t^\star}}(u)] \sigma_{T-t_K}^{-2\Gamma}} \\
      & \qquad \times \textstyle{\exp[(\Gamma+1)\int_{T-t_K}^{T-t^\star}\beta_s \rmd s]  \normLigne{ \Delta b_u((\bbfY_{0,v})_{v\in \ccint{0,T}})} \rmd u } \\
                                                        &\leq \sigma_{T-t_k}^{-2\Gamma} \textstyle{\exp[(\Gamma+1)\int_{T-t_K}^{T-t^\star}\beta_s \rmd s]} ( \int_0^{t^\star} \exp[-(1/2) \int_{T-t^\star}^{T-u} \beta_s \rmd s] \Delta b_u((\bbfY_{0,v})_{v\in \ccint{0,T}}) \rmd u \\
      & \qquad \textstyle{+ \int_{t^\star}^{t_K} \Delta b_u((\bbfY_{0,v})_{v\in \ccint{0,T}})}\rmd u ) \eqsp .
\end{align}
Using this result and \Cref{prop:local-error-control} we get
\begin{align}
  &\wassersteinD[1](\pi_\infty \rmQ_{t_K}, \pi_\infty \rmR_{K}) \leq \expeLigne{\normLigne{\bfY_{t_K} - Y_K}} \\
  &\qquad \leq \textstyle{ \sigma_{T-t_K}^{-2\Gamma} \textstyle{\exp[(\Gamma+1)\int_{T-t_K}^{T-t^\star}\beta_s \rmd s]}(\int_{0}^{t^\star} \exp[-(1/2) \int_{T-t^\star}^{T-u} \beta_s \rmd s] \expeLigne{\normLigne{ \Delta b_u((\bbfY_{0,v})_{v\in \ccint{0,T}})}} \rmd u} \\
      & \qquad \qquad \textstyle{+ \int_{t^\star}^{t_K} \expeLigne{\normLigne{ \Delta b_u((\bbfY_{0,v})_{v\in \ccint{0,T}})}} \rmd u ) \eqsp .  } \\
  &\qquad \leq \textstyle{ \sigma_{T-t_K}^{-2\Gamma} \textstyle{\exp[(\Gamma+1)\int_{T-t_K}^{T-t^\star}\beta_s \rmd s]} \Ctt_0 (T - t_K + \bar{\beta})^2 (\Mtt + \delta^{1/2}) /(T-t_K)^2 } \\
  & \qquad \qquad \textstyle{ \times (\int_0^{t^\star} \exp[-(1/2) \int_{T-t^\star}^{T-u} \beta_s \rmd s]  \rmd u + t_K - t^\star)} \\
    &\qquad \leq \textstyle{ \sigma_{T-t_K}^{-2\Gamma} \textstyle{\exp[(\Gamma+1)\bar{\beta}(t_K - t^\star) \rmd s]} \Ctt_0 (T - t_K + \bar{\beta})^2 (\Mtt + \delta^{1/2}) /(T-t_K)^2 } \\
  & \qquad \qquad \textstyle{ \times (\int_0^{t^\star} \exp[-(1/2) \int_{T-t^\star}^{T-u} \beta_s \rmd s]  \rmd u + t_K - t^\star)}\eqsp . \label{eq:bound_wasserstein_intermediate_improved}
\end{align}
We have that
\begin{equation}
  \textstyle{ \int_0^{t^\star} \exp[-(1/2) \int_{T-t^\star}^{T-u} \beta_s \rmd s]  \rmd u \leq \int_0^{t^\star} \exp[-(t^\star -u)/(2\bar{\beta})]  \rmd u \leq 2 \bar{\beta} \eqsp . } \label{eq:bound_integration_improved}
\end{equation}
In addition, using \eqref{eq:def_t_star}  we have
\begin{equation}
  t_K - t^\star = T - \vareps - T +2 \bar{\beta} (1 + \log(1 + \diam(\M))) \leq 2 \bar{\beta} (1 + \log(1 + \diam(\M)))  \eqsp . \label{eq:bound_tk_tstar_improved}
\end{equation}
Finally, using \Cref{lemma:bound_sigma_t}, we have that
$\sigma_{T-t_K}^{-2} \leq (1 + \bar{\beta})/\vareps$. Combining this result,
\eqref{eq:bound_integration_improved} and \eqref{eq:bound_tk_tstar_improved} in
\eqref{eq:bound_wasserstein_intermediate_improved} and that for any $a \geq 0$,
$1+a \leq \exp[a]$, we get
\begin{align}
  \wassersteinD[1](\pi_\infty \rmQ_{t_K}, \pi_\infty \rmR_{K}) &\leq 4(1+ \bar{\beta})^{\Gamma+3} (1 + \log(1 + \diam(\M))) \\
   & \qquad \times \exp[2\bar{\beta}^2(\Gamma +1)(1 + \log(1+\diam(\M)))] 
     \Ctt_0 (\Mtt + \delta^{1/2})\vareps^{-(\Gamma+2)} \\
  &\leq 4(1+ \bar{\beta})^{\Gamma+3} (1 + \log(1 + \diam(\M))) \\
   & \qquad \times \exp[2\bar{\beta}^2(\Gamma +1)(1 + \log(1+\diam(\M)))] 
     \Ctt_0 (\Mtt + \delta^{1/2})\vareps^{-(\Gamma+2)} \\
  &\leq 4 \exp[3\bar{\beta}^2(\Gamma +2)(1 + \log(1+\diam(\M)))] 
     \Ctt_0 (\Mtt + \delta^{1/2})\vareps^{-(\Gamma+2)} \eqsp ,
\end{align}
which concludes the proof.
\end{proof}

We now state the equivalent of \Cref{prop:wasserstein_forward}.

\begin{proposition}
    \label{prop:wasserstein_forward_improved}
    Assume \rref{assum:manifold_hyp}, \rref{assum:hessian_bound} and
    $T \geq 2\bar{\beta}(1 + \log(1+\diam(\M))$. Then, for any
    $x, y \in \rset^d$ and $t \in \ccint{0,t_K}$
  \begin{equation}
    \textstyle{
      \wassersteinD[1](\updelta_x \Qker_{t}, \updelta_y \Qker_{t}) \leq  \exp[2(\Gamma+1)\bar{\beta}^2(1 + \log(1+\diam(\M)))]\sigma_{T-t_K}^{-2\Gamma} \normLigne{x-y} \eqsp .
      }
    \end{equation}    
  \end{proposition}

  \begin{proof}
    This is a direct consequence of \eqref{eq:pro_wasserstein},
    \Cref{prop:control_gradient_improved} and
    \eqref{eq:bound_tk_tstar_improved}.
  \end{proof}

  Finally, we control $\wassersteinD[1](\pi_\infty \Qker_{t_K}, \pi \Pker_{T-t_K})$. First, we have 
\begin{align}
\label{eq:convergence_bound_backward_inter_improved}  
  \wassersteinD[1](\pi_\infty \Qker_{t_K}, \pi \Pker_{T-t_K}) &= \wassersteinD[1](\pi_\infty \Qker_{t_K}, \pi \Pker_{T}  \Qker_{t_K}) \\
  &\leq \exp[2(\Gamma+1)\bar{\beta}^2(1 + \log(1+\diam(\M)))]\sigma_{T-t_K}^{-2\Gamma} \wassersteinD[1](\pi \Pker_T, \pi_\infty) \eqsp . 
\end{align}
To control $\wassersteinD[1](\pi \Pker_T, \pi_\infty)$, we use a
synchronous coupling, \ie we set $(\bfY_t, \bfZ_t)_{t \in \ccint{0,T}}$ such that
\begin{equation}
  \rmd \bfY_t = - \beta_t \bfY_t \rmd t + \sqrt{2 \beta_t} \rmd \bfB_t \eqsp , \qquad 
  \rmd \bfZ_t = - \beta_t \bfZ_t \rmd t + \sqrt{2 \beta_t} \rmd \bfB_t \eqsp ,
\end{equation}
where $(\bfB_t)_{t \in \ccint{0,T}}$ is a $d$-dimensional Brownian motion and
$\bfY_0 \sim \pi$, $\bfZ_0 \sim \pi_\infty$. We have that for any
$t \in \ccint{0,T}$, $\bfZ_t \sim \pi_\infty$. In addition, denoting
$u_t = \expeLigne{\normLigne{\bfY_t - \bfZ_t}}$ for any $t \in \ccint{0,T}$, we
have that
\begin{equation}
  \textstyle{\rmd u_t \leq u_0 \exp[-\int_0^t \beta_s \rmd s] \eqsp . }
\end{equation}
Therefore, combining this result and
\eqref{eq:convergence_bound_backward_inter_improved}, we get that
\begin{equation}
  \label{eq:convergence_bound_backward_improved}
  \textstyle{
    \wassersteinD[1](\pi_\infty \Qker_{t_K}, \pi \Pker_{T-t_K}) \leq \exp[2(\Gamma+1)\bar{\beta}^2(1 + \log(1+\diam(\M)))]\sigma_{T-t_K}^{-2\Gamma} \exp[-\int_0^T \beta_t \rmd t] \wassersteinD[1](\pi, \pi_\infty) \eqsp .
    }
  \end{equation}
Therefore, similarly as in the proof of \Cref{prop:discretization_bound_final}, we have
\begin{equation}
  \textstyle{
    \wassersteinD[1](\pi_\infty \Qker_{t_K}, \pi \Pker_{T-t_K}) \leq \Dtt_1 \exp[-T/\bar{\beta}]/\vareps^{\Gamma} \eqsp , }
\end{equation}
with
\begin{equation}
  \Dtt_1 =  \exp[2(\Gamma+2)(1+\bar{\beta})^2(1 + \log(1+\diam(\M)))] (\sqrt{d} + \diam(\M)) \eqsp . 
\end{equation}
We conclude the proof of \Cref{thm:convergence_bound_hessian} upon combining
this result, \Cref{prop:discretization_bound_final_improved} and
\eqref{eq:control_noising}.

\subsection{Hessian bounds for the uniform distribution}
\label{sec:hess-bounds-unif}

The goal of this section is to prove the following result.

\begin{proposition}
  Assume that $\pi$ is the uniform distribution over $\ccint{-1/2, 1/2}^p$ for
  some $p \in \{1, \dots, d\}$. Then, there exists $\Gamma \geq 0$ such that for
  any $t \in \ocint{0,T}$ and $x \in \rset^d$,
  $\normLigne{\nabla^2 \log p_t(x_t)} \leq \Gamma / \sigma_t^2$.
\end{proposition}

\begin{proof}
  Let $t \in \ocint{0,T}$. We start by deriving a closed form expression for
  $p_t$. We have for any $x \in \rset^d$
  \begin{align}
    p_t(x) &= \textstyle{\int_{\M} \exp[-\normLigne{x-m_t z}^2/(2\sigma_t^2)]\rmd \pi(z)}(2\uppi \sigma_t)^{-d/2} \\
           &= \textstyle{\exp[-\sum_{i=p+1}^d x_i^2/(2\sigma_t^2)] \prod_{i=1}^p \int_{-1/2}^{1/2} \exp[-(x_i-m_t z_i)^2/(2\sigma_t^2)]\rmd z_i}(2\uppi \sigma_t)^{-d/2} \\
           &= \textstyle{\exp[-\sum_{i=p+1}^d x_i^2/(2\sigma_t^2)] \prod_{i=1}^p \int_{x_i-m_t/2}^{x_i + m_t/2} \exp[- z_i^2/(2\sigma_t^2)]\rmd z_i}(2\uppi \sigma_t)^{-d/2} m_t^{-p}   \\
           & = \textstyle{\exp[-\sum_{i=p+1}^d x_i^2/(2\sigma_t^2)] (2\uppi \sigma_t)^{-d/2}} \textstyle{\times \prod_{i=1}^p \int_{(x_i-m_t/2)/\sigma_t}^{(x_i + m_t/2)/\sigma_t} \exp[- z_i^2/2]\rmd z_i} (\sigma_t/m_t)^{p} \\ 
           & = (2\uppi)^{d/2} (2\uppi \sigma_t)^{-d/2} (\sigma_t/m_t)^{p} \textstyle{\prod_{i=p+1}^d \varphi(x_i /\sigma_t) } \\
    & \qquad \textstyle{\times \prod_{i=1}^p \{\Phi((x_i + m_t/2)/\sigma_t) - \Phi((x_i - m_t/2)/\sigma_t)\} }  \eqsp ,
  \end{align}
  where $\varphi(t) = \exp[-t^2/2]/\sqrt{2\uppi}$ and
  $\Phi(t) = \int_{-\infty}^t \varphi(s) \rmd s$. Note that from this
  expression, $\nabla^2 \log p_t$ is diagonal. Hence, we only need to compute
  $\partial_i^2 \log p_t(x)$ for any $x \in \rset^d$ and
  $i \in \{1, \dots, d\}$. Let $i \in \{p+1, \dots, d\}$. We have for any $x \in \rset^d$
  \begin{equation}
    \partial_i^2 \log p_t(x) = -\sigma_t^{-2} \eqsp .
  \end{equation}
  We now turn to the case where $i \in \{1, \dots, p\}$. In this case, we denote
  $F_t(\Phi, a, b) = \Phi(a + b) - \Phi(a - b) $ and we have
  \begin{equation}
    \label{eq:expression_hessienne}
    \partial_i^2 \log p_t(x) = (1/\sigma_t^2) \partial_2^2 \log F_t(\Phi, x_i/\sigma_t, m_t/\sigma_t) \eqsp . 
  \end{equation}
  We have that for any $a,b \in \rset$ with $a \neq b$
  \begin{equation}
    \label{eq:decomposition_der_second}
    \partial_2^2 \log F_t(\Phi, a, b) = F_t(\varphi', a, b)/F_t(\Phi, a, b) - F_t(\varphi, a, b)^2/F_t(\Phi, a, b)^2 \eqsp . 
  \end{equation}
  In what follows, we assume that $a >b$ and we define for any $t \in \rset$
  \begin{equation}
    \textstyle{\erfc(t) = 1 - (2/\sqrt{\uppi}) \int_0^t \exp[-s^2] \rmd s \eqsp . }
  \end{equation}
  Note that $F_t(\Phi, a, b) = (1/2)(\erfc((a-b)/2) - \erfc((a+b)/2))$. In addition, there exists $C > 0$ such that  for any $t \neq 0$ we have 
  \begin{equation}
    \erfc(t) \geq \exp[-t^2]/(\sqrt{\uppi}t) (1 + C/t^2) \eqsp , \qquad \erfc(t) \leq \exp[-t^2]/(\sqrt{\uppi}t) (1 - /(Ct^2)) \eqsp . 
  \end{equation}
  In particular, we have
  \begin{align}
    \label{eq:expansion_F_Phi}
    &\textstyle{F_t(\Phi, a, b) \leq  \exp[-(a-b)^2/2]/(\sqrt{2\uppi}(a-b)) (1 + R_0(a,b)) \eqsp , } \\
    &\textstyle{F_t(\Phi, a, b) \geq  \exp[-(a-b)^2/2]/(\sqrt{2\uppi}(a-b)) (1 - \bar{R_0}(a,b)) \eqsp , }
  \end{align}
  where
  \begin{align}
    R_0(a,b) &= C/(a-b)^2 + \exp[+(a-b)^2/2-(a+b)^2/2](a-b)/(a+b )(1 - /(C(a+b)^2)) \eqsp , \\
    \bar{R}_0(a,b) &= -C/(a-b)^2 + \exp[+(a-b)^2/2-(a+b)^2/2](a-b)/(a+b )(1 + C/(a+b)^2) \eqsp , \\
  \end{align}
  Note that there exists $C_0 \geq 0$ such that for any $a, b\in \rset$ with $a \geq b +1$
  \begin{equation}
    \label{eq:control_C0}
    (R_0(a,b) + \bar{R}_0(a,b))(a-b)^2 \leq C_0 \eqsp . 
  \end{equation}
  In particular, there exists $a_0 \geq 0$ such that if $a \geq b + a_0$, $\bar{R}_0(a,b) + R_0(a,b) \leq 1/2$.
  Similarly, we have
  \begin{equation}
    \label{eq:expansion_F_varphi}
    F_t(\varphi, a, b) = (2\uppi)^{-1/2} \exp[-(a-b)^2/2] (-1 + \exp[-((a+b)^2-(a-b)^2)/2]) \eqsp . 
  \end{equation}
  We denote $R_1(a,b) = \exp[-((a+b)^2-(a-b)^2)/2]$ and note that there exists
  $C_1 \geq 0$ such that for any $a, b\in \rset$ with $a \geq b +1$
  \begin{equation}
    \label{eq:control_C1}
    R_1(a,b)(a-b)^2 \leq C_1 \eqsp . 
  \end{equation}
  Finally, we have
  \begin{equation}
    \label{eq:expansion_F_varphi_prime}
    F_t(\varphi', a, b) = (2\uppi)^{-1/2} \exp[-(a-b)^2/2](a-b) (1 + (a+b)/(a-b)\exp[-((a+b)^2-(a-b)^2)/2]) \eqsp . 
  \end{equation}
  We denote $R_2(a,b) = (a+b)/(a-b) \exp[-((a+b)^2-(a-b)^2)/2]$ and note that there exists
  $C_2 \geq 0$ such that for any $a, b\in \rset$ with $a \geq b +1$
  \begin{equation}
    \label{eq:control_C2}
    R_2(a,b)(a-b)^2 \leq C_2 \eqsp . 
  \end{equation}
  Combining \eqref{eq:decomposition_der_second}, \eqref{eq:expansion_F_Phi},
  \eqref{eq:expansion_F_varphi} and \eqref{eq:expansion_F_varphi_prime}, we get
  that for any $a, b \in \rset$ with $a \geq b + 1$
  \begin{equation}
    \partial_2^2 \log F_t(\Phi, a, b) \leq (a-b)^2 [(1+R_2(a,b))/(1+\bar{R}_0(a,b)) - (1+R_1(a,b))^2/(1+R_0(a,b))^2] \eqsp . 
  \end{equation}
  In addition, we have for any $a, b \in \rset$ with $a \geq b + 1$
  \begin{align}
    &(1+R_2(a,b))/(1+\bar{R}_0(a,b)) - (1+R_1(a,b))^2/(1+R_0(a,b))^2 \\ &\qquad = (R_2(a,b)(1 +R_0(a,b))^2 -R_1(a,b)(1+\bar{R_0}(a,b) )) /[(1+\bar{R}_0(a,b))^2(1+\bar{R}_0(a,b))] \eqsp . 
  \end{align}
  Combining this result, \eqref{eq:control_C0}, \eqref{eq:control_C1} and
  \eqref{eq:control_C2}, we get that for any $a, b \in \rset$ with
  $a \geq b + a_0$
  \begin{equation}
    R_2(a,b)(1 +R_0(a,b))^2 -R_1(a,b)(1+\bar{R_0}(a,b) \leq 4(C_1 + C_2)/(a-b)^2 \eqsp . 
  \end{equation}
  Therefore, there exists $C_3 \geq 0$ such that for any $a \geq b +a_0$,
  $\absLigne{\partial_2^2 \log F_t(\Phi,a,b)} \leq C_3$. Similarly there exists
  $C_4 \geq 0$ such that if $a \leq - b -a_0$,
  $\absLigne{\partial_2^2 \log F_t(\Phi,a,b)} \leq C_4$. By symmetry, there
  exists $C_5 \geq 0$ such that for any $b \geq a + a_0$ or $-b \geq -a - a_0$,
  we have $\absLigne{\partial_2^2 \log F_t(\Phi,a,b)} \leq C_5$ and we conclude
  by continuity that there exists $C \geq 0$ such that for any $a, b \in \rset$,
  \begin{equation}
    \partial_2 \log F_t(\Phi, a, b) \leq C \eqsp ,
  \end{equation}
  which concludes the proof upon combining this result with \eqref{eq:expression_hessienne}.
\end{proof}

\subsection{The role of convexity}
\label{sec:non-conv-count}

In this section, for any $x \in \rset^d$, we define $f_x : \M \to \rset_+$ given
for any $y \in \M$ by $f_x(y) = \normLigne{x - y}^2$. Before giving our main
result we need to introduce a few useful tools.

For any subset $\msx \subset \rset^d$ and $x \in \rset^d$
we define $\rmP(x) = \ensembleLigne{y \in \msx}{d(x,\msx) = d(x,y)}$. Note that
$\rmP(x)$ can be empty. We say that a set $\msx$ is \emph{Chebyshev} if for all
$x \in \rset^d$, there exists $p(x) \in \msx$ such that $\rmP(x) = \{p(x)\}$,
\ie Chebyshev sets are the subsets of $\rset^d$ such that each point admits a
unique projection on $\msx$. It is clear that all closed and convex sets are
Chebyshev sets. Note that all Chebyshev sets are closed since for any Chebyshev
set $\msx$ and $x \in \partial \msx$ (the frontier of $\msx$) we have that there
exists $p(x) \in \msx$ such that $d(x,p(x)) = d(x, \msx) = 0$, \ie 
$x \in \msx$. In addition, Chebyshev sets are also convex, see
\citep{kritikos1938quelques,motzkin1935quelques,bundt1934chebyshev}. This result
implies the following proposition.

\begin{proposition}
  \label{prop:chebyshev_convex}
  Let $\msx \subset \rset^d$. $\msx$ is a closed convex set if and only if
  $\msx$ is a Chebyshev set.
\end{proposition}

In order to prove our main result, we introduce some basics on \emph{Morse
  theory}. Assume that $\M$ is a smooth manifold and
$f \in \rmc^\infty(\M, \rset)$. We say that $x\in \M$ is a non-degenerate
minimizer if $x \in \M$ is a minimizer of $f$ and the Hessian of $f$ at $x$ is
not singular. We then have the following proposition (see
\citep{matsumoto2002introduction} for instance).

\begin{proposition}
  \label{prop:morse}
  Let $f \in \rmc^\infty(\M, \rset)$ and $x \in \M$ a non degenerate minimizer
  of $f$. Then there exist $\msu \subset \M$ open and
  $\varphi: \msu \to \varphi(\msu) \subset \rset^p$ a local chart such that
  $\varphi(x) = 0$ and for any $\hat{y} \in \varphi(\msu)$
  \begin{equation}
    \textstyle{f(\varphi^{-1}(\hat{y})) = f(x) + \normLigne{\hat{y}}^2 \eqsp . }
  \end{equation}
  In addition, we have $\rmD \varphi(x) = \nabla^2 f(x)$.
\end{proposition}

Note that upon considering $\varphi^{-1}(\varphi(\msu)/2)$ instead of $\msu$ we
can always assume that $\varphi^{-1}$ has bounded derivatives.

Finally, we introduce the concept of \emph{shape operator} (also called
\emph{second fundamental form} or \emph{Weingarten form}), see
\citep{bishop2011geometry}. We assume that the metric on $\M$ is the induced
Euclidean metric. Let
$N \in \Gamma(\mathrm{T} \M^\top)$ a section on the normal bundle
$\mathrm{T} \M^\top$. We define
$\mathfrak{A}^N: \Gamma(\mathrm{T} \M)^2 \to \rmc^\infty(\M)$ such that for any
$V_1, V_2 \in \Gamma(\mathrm{T} \M)$
\begin{equation}
  \mathfrak{A}^N(V_1, V_2) = -\langle V_1, \nabla_{V_2} N \rangle \eqsp .
\end{equation}
Note that $\mathfrak{A}^N$ is symmetric, linear and the scalar product and
covariant derivative $\nabla$ are considered \wrt the ambient Euclidean metric. The shape
operator encodes the local geometry of $\M$. For example in the case of
$\mathbb{S}^{d-1}$, we have that for any $N \in \Gamma(\mathrm{T} \M^\top)$ and
$V_1, V_2 \in \Gamma(\mathrm{T} \M)$
\begin{equation}
  \label{eq:shape_operator_sphere}
  \mathfrak{A}^N(V_1, V_2) = - \langle V_1, V_2 \rangle \langle N_0, N\rangle \eqsp,
\end{equation}
where $N_0$ is the normal vector field pointing outward of the sphere, see
\citep{absil2013extrinsic}.  We have the following result, see \cite[Theorem
3]{bishop2011geometry} and \citep{bishop1974infinitesimal}.

\begin{proposition}
  If $\M$ is convex then for any $N \in \Gamma(\mathrm{T} \M^\top)$ such that
  for any $x, y \in \M$, $\langle y - x, N(x) \rangle \geq 0$, $\mathfrak{A}^N$
  is non-negative.
\end{proposition}

Finally, let $\bar{f} \in \rmc^\infty(\rset^d, \rset)$ and $f$ its restriction
to $\M$. Using \citep{absil2013extrinsic}, we have for any
$V_1, V_2 \in \Gamma(\mathrm{T} \M)$
\begin{equation}
  \label{eq:hessian_formula}
  \nabla^2 f(V_1, V_2) = \langle V_1, \Pi(\nabla^2 \bar{f}(V_2)) \rangle + \mathfrak{A}^{\Pi^\top(\nabla \bar{f})}(V_1, V_2) \eqsp ,
\end{equation}
where for any $x\in \M$, $\Pi_x$ is the orthogonal projection operator on
$\mathrm{T}_x\M$. We are now ready to state our main result.

\begin{theorem}
  \label{prop:non-convex-sets}
  Assume that $\M \subset \rset^d$ is a smooth manifold and that $\pi$ admits a
  smooth density \wrt the Hausdorff measure on $\M$. The following hold:
  \begin{enumerate}[wide, labelwidth=!, labelindent=0pt, label=(\alph*)]
  \item \label{item:convex} If $\M$ is convex then for any $x \in \M$ we have $\limsup_{t \to 0} \sigma_t^2 \normLigne{\nabla^2 \log p_t(m_t x)} < +\infty$.     
  \item \label{item:non-convex} If there exists $x \in \rset^d$ such that
    $\absLigne{\rmP(x)} > 1$ and for any $p(x) \in \rmP(x)$,
    $\mathfrak{A}^{p(x) - x}_{p(x)} \succ -\Id$. Then, we have
    $\liminf_{t \to 0} \sigma_t^4 \normLigne{\nabla^2 \log p_t(m_t x)} > 0$.
  \end{enumerate}
\end{theorem}

\Cref{prop:non-convex-sets} implies that one can obtain information about the
geometry of $\M$ by computing the Hessian of the logarithmic gradient of the
densities of $(\mathcal{L}(\bfX_t))_{t \in \ccint{0,T}}$. In the convex case the scaling
\wrt $\sigma_t$ is of order $\sigma_t^{-2}$ whereas in the second scenario
the scaling is of order $\sigma_t^{-4}$. Note that the condition ``there exists
$x \in \rset^d$ such that $\absLigne{\rmP(x)} > 1$'' is equivalent to assuming
that $\M$ is not a Chebyshev set and hence a non convex set in virtue of
\Cref{prop:chebyshev_convex}. Therefore in
\Cref{prop:non-convex-sets}-\ref{item:non-convex}, we assume that $\M$ is non
convex and a curvature condition. The condition
$\mathfrak{A}^{p(x) - x}_x \succ -\Id$ implies that the manifold is not too
``negatively curved'' at the projection points.  The non-strict inequality is
always true, \ie for any $p(x) \in \rmP(x)$, $\mathfrak{A}^{p(x) - x}_x \succeq -\Id$
since $\nabla^2 f^x(p(x)) \succeq 0$.

We conjecture that this curvature condition can be relaxed. Indeed, it is not
satisfied in the case where
$\M = \ensembleLigne{x \in \rset^d}{\normLigne{x}=1}$, since the only point
$x \in \rset^d$ such that $\absLigne{\rmP(x)} > 1$ is $x=0$ and in that case
$\rmP(x) = \M$ and therefore $\nabla^2 f^0 = 0$, which implies that for any
$x \in \M$, $\mathfrak{A}_{x}^{x} = -\Id$. This formula could also have been
obtained from \eqref{eq:shape_operator_sphere}. However, one can show that we
still have $\liminf_{t \to 0} \sigma_t^4 \normLigne{\nabla^2 \log p_t(0)} >
0$. In future works, we would like to relax these curvature conditions assuming
that the manifold has an analytic structure and using results from
\citep{combet2006integrales}.

Finally, we highlight that \Cref{prop:non-convex-sets}-\ref{item:convex} is
weaker than the condition \Cref{assum:hessian_bound}. To bridge the gap between
\Cref{prop:non-convex-sets}-\ref{item:convex} and \Cref{assum:hessian_bound} one
would need to strengthen \Cref{prop:non-convex-sets}-\ref{item:convex} to derive
\emph{uniform in space} bounds. This would require to use \emph{quantitative}
version of the Morse lemmas, see \citep{le2014numerical} for instance. We
postpone this study to future works. However,
\Cref{prop:non-convex-sets}-\ref{item:non-convex} implies that
\Cref{assum:hessian_bound} does not hold. Hence, any non convex set which is not
too ``negatively curved'' does not satisfy \Cref{assum:hessian_bound}.


\begin{proof}
  \begin{enumerate}[wide, labelwidth=!, labelindent=0pt, label=(\alph*)]  
  \item First, we assume that $\M$ is convex. We show that for any
    $x \in \rset^d$ and $p(x) \in \mathrm{P}(x)$, $\nabla^2f^x(p(x)) \succ 0$,
    \ie the Hessian of $f^x$ is not degenerate. For any $x \in \rset^d$, we define
    $\bar{f}^x$ such that for any $y \in \rset^d$,
    $\bar{f}^x(y) = \normLigne{x-y}^2$. Using
    \eqref{eq:hessian_formula}, we have for any $x \in \rset^d$
    \begin{equation}
      \label{eq:hessian_invertible}
      \nabla^2 f^x(p(x)) = \Id + \mathfrak{A}^{p(x) - x} \succeq \Id \succ 0 \eqsp . 
    \end{equation}
    Since $\M$ is convex for any $x\in \rset^d$, $\rmP(x) = \{p(x)\}$ and note
    that $\rmP(x)$ is the set of minimizers of $f^x$. Let $x \in \rset^d$. Using
    \Cref{prop:morse} and \eqref{eq:hessian_invertible}, there exist $\msu \subset \M$ open and
    $\varphi: \msu \to \varphi(\msu) \subset \rset^p$ a local chart such that
    $\varphi(p(x)) = 0$ and for any $\hat{y} \in \varphi(\msu)$
    \begin{equation}
      \label{eq:changement_variable}
    \textstyle{\normLigne{x - \varphi^{-1}(\hat{y})}^2 = \normLigne{x - p(x)}^2 + \normLigne{\hat{y}}^2 \eqsp , }
  \end{equation}
  with $\varphi^{-1}(0) = p(x)$. Note that
  $\rmD \varphi^{-1}(0) = (\nabla^2 f^x(p(x)))^{-1}$. For any
  $t \in \ocint{0,T}$, denote $\lambda_t = m_t / \sigma_t$. Using
  \eqref{eq:changement_variable}, we have that
  \begin{align}
    &\textstyle{\int_{\msu^2} (y_0 - y_1)^{\otimes 2} \exp[-(m_t/\sigma_t)^2\normLigne{x - y_0}^2/2]\exp[-(m_t/\sigma_t)^2\normLigne{x - y_1}^2/2] \rmd \pi(y_0) \rmd \pi(y_1)} \\
    &\textstyle{=\int_{\varphi^{-1}(\msu)^2} (\varphi^{-1}(\hat{y}_0) - \varphi^{-1}(\hat{y}_1))^{\otimes 2} \exp[-(m_t/\sigma_t)^2\normLigne{\hat{y}_0}^2/2]\exp[-(m_t/\sigma_t)^2\normLigne{\hat{y}_1}^2/2] \rmd \hat{\pi}(\hat{y}_0) \rmd \hat{\pi}_1(\hat{y}_1)} \\
    & \qquad \times \exp[-(m_t/\sigma_t)^2\normLigne{x-p(x)}^2] \\
    &\textstyle{=(1/\lambda_t)^{2p}\int_{\lambda_t \varphi^{-1}(\msu)^2} (\varphi^{-1}(\hat{y}_0/\lambda_t) - \varphi^{-1}(\hat{y}_1/\lambda_t))^{\otimes 2} \exp[-\normLigne{\hat{y}_0}^2/2]\exp[-\normLigne{\hat{y}_1}^2/2] \rmd \hat{\pi}(\hat{y}_0) \rmd \hat{\pi}_1(\hat{y}_1)} \\ 
      & \qquad \times \exp[-(m_t/\sigma_t)^2\normLigne{x-p(x)}^2] \eqsp ,
  \end{align}
  where $\hat{\pi} = \varphi_{\#} \pi$ admits a positive density \wrt the
  Lebesgue measure.  Therefore, there exists $C_0 \geq 0$ such that for any
  $t \in \ocint{0,T}$
  \begin{align}
    &(2 \uppi/\lambda_t^2)^{-p}\normLigne{\textstyle{\int_{\msu^2} (y_0 - y_1)^{\otimes 2} \exp[-(m_t/\sigma_t)^2\normLigne{x - y_0}^2/2]\exp[-(m_t/\sigma_t)^2\normLigne{x - y_1}^2/2] \rmd \pi(y_0) \rmd \pi(y_1)}} \\
    &\qquad \textstyle{\leq C_0 (2 \uppi)^{-p} \sigma_t^2 \int_{(\rset^p)^2} \normLigne{\hat{y}_0 - \hat{y}_1}^2  \exp[-\normLigne{\hat{y}_0}^2/2]\exp[-\normLigne{\hat{y}_1}^2/2] \rmd \hat{y}_0 \rmd \hat{y}_1}  \\
    &\qquad  \qquad \times \exp[-(m_t/\sigma_t)^2\normLigne{x-p(x)}^2] \\
    &\qquad \leq 2C_0 p \sigma_t^2 \exp[-(m_t/\sigma_t)^2\normLigne{x-p(x)}^2] \eqsp , \label{eq:inside_bound_convex}
  \end{align}
  where we have used that
  $\norm{a-b}^2 \leq 2(\normLigne{a}^2 + \normLigne{b}^2)$ for any
  $a, b \in \rset^p$ in the last line.  In addition, since $\msu$ is open we
  have that $\M \cap \msu^\complementary$ is compact and since for any
  $y \in \M \cap \msu^\complementary$,
  $\normLigne{y - x}^2 > \normLigne{p(x) - x}^2$, there exists $\vareps >0$ such
  that for any $y \in \M \cap \msu^\complementary$,
  $\normLigne{y - x}^2 \geq \normLigne{p(x) - x}^2 + \vareps$. Therefore, we
  have
  \begin{align}
    &\normLigne{\textstyle{\int_{(\M \cap \msu^\complementary)^2} (y_0 - y_1)^{\otimes 2} \exp[-(m_t/\sigma_t)^2\normLigne{x - y_0}^2/2]\exp[-(m_t/\sigma_t)^2\normLigne{x - y_1}^2/2] \rmd \pi(y_0) \rmd \pi(y_1)}} \\
    &\qquad \leq \diam(\M)^2 \exp[-(m_t/\sigma_t)^2\normLigne{x-p(x)}^2] \exp[-\vareps (m_t/\sigma_t)^2] \eqsp .  \label{eq:outside_bound_convex}
  \end{align}
  Combining \eqref{eq:inside_bound_convex} and \eqref{eq:outside_bound_convex}
  there exists $C_1 \geq 0$ such that for any $t \in \ocint{0,T}$
  \begin{align}
    &(2 \uppi/\lambda_t^2)^{-p}\normLigne{\textstyle{\int_{\M^2} (y_0 - y_1)^{\otimes 2} \exp[-(m_t/\sigma_t)^2\normLigne{x - y_0}^2/2]\exp[-(m_t/\sigma_t)^2\normLigne{x - y_1}^2/2] \rmd \pi(y_0) \rmd \pi(y_1)}} \\
    & \qquad \qquad \leq C_1 \sigma_t^2 \exp[-(m_t/\sigma_t)^2\normLigne{x-p(x)}^2] \eqsp . \label{eq:upper_bound_convex}
  \end{align}
  In addition, there exists $C_2 > 0$ such that for any $t \in \ocint{0,T}$
  \begin{align}
    (2 \uppi/\lambda_t^2)^{-p}\textstyle{\int_{\msu} \exp[-(m_t/\sigma_t)^2\normLigne{x - y}^2/2] \rmd \pi(y)} \geq (1/C_2) \exp[-(m_t/\sigma_t)^2\normLigne{x-p(x)}^2] \eqsp . \label{eq:lower_bound_convex}
  \end{align}
  Therefore, combining \eqref{eq:upper_bound_convex} and
  \eqref{eq:lower_bound_convex}, we get that there exists $C_3 \geq 0$ such that
  for any $t \in \ocint{0,T}$
  \begin{align}
    &\normLigne{\textstyle{\int_{\M} (y_0 - y_1)^{\otimes 2} \exp[-(m_t/\sigma_t)^2\normLigne{x - y_0}^2/2]\exp[-(m_t/\sigma_t)^2\normLigne{x - y_1}^2/2] \rmd \pi(y_0) \rmd \pi(y_1)}} \\
    & \qquad \qquad  \textstyle{/ (\int_{\M} \exp[-(m_t/\sigma_t)^2\normLigne{x - y}^2/2] \rmd \pi(y))^2 } \leq C_3 \sigma_t^2 \eqsp .
  \end{align}
  We conclude the proof in the convex case upon combining this result and  \Cref{lemma:control_hessian}.
\item Second, we assume that there exists $x \in \rset^d$ such that
  $\absLigne{\rmP(x)} > 1$ and for any $p(x) \in \rmP(x)$,
  $\mathfrak{A}^{p(x) - x} \succ -\Id$. Using \eqref{eq:hessian_formula}, we
  have for any $p(x) \in \rmP(x)$
    \begin{equation}
      \nabla^2 f^x(p(x)) = \Id + \mathfrak{A}^{p(x) - x}  \succ 0 \eqsp . 
    \end{equation}
    Using this fact and that $\rmP(x)$ is the set of minimizers of $f^x$ and is
    compact, we get that $\absLigne{\rmP(x)} < +\infty$. Hence, we assume that
    $\rmP(x) = \{p_i(x)\}_{i=1}^N$ with $N > 1$. Using \Cref{prop:morse}, for
    any $i \in \{1, \dots, N\}$, there exist $\msu_i \subset \M$ open and
    $\varphi_i: \msu_i \to \varphi_i(\msu_i) \subset \rset^p$ a local chart such that
    $\varphi_i(p_i(x)) = 0$ and for any $\hat{y} \in \varphi_i(\msu_i)$
  \begin{equation}
    \textstyle{\normLigne{x - \varphi_i^{-1}(\hat{y})}^2 = \normLigne{x - p_i(x)}^2 + \normLigne{\hat{y}}^2 \eqsp , }
  \end{equation}
  with $\varphi_i^{-1}(0) = p_i(x)$. Note that
  $\rmD \varphi_i^{-1}(0) = (\nabla^2 f^x(p_i(x)))^{-1}$. Without loss of
  generality we assume that for any $i, j \in \{1, \dots, N\}$,
  $\msu_i \cap \msu_j = \emptyset$.
  We have that
  \begin{align}
    &\textstyle{\int_{\msu_0 \times \msu_1} (y_0 - y_1)^{\otimes 2} \exp[-(m_t/\sigma_t)^2\normLigne{x - y_0}^2/2]\exp[-(m_t/\sigma_t)^2\normLigne{x - y_1}^2/2] \rmd \pi(y_0) \rmd \pi(y_1)} \\
    &\textstyle{=\int_{\varphi_0^{-1}(\msu_0) \times \varphi_1^{-1}(\msu_1)} (\varphi_0^{-1}(\hat{y}_0) - \varphi_1^{-1}(\hat{y}_1))^{\otimes 2} \exp[-(m_t/\sigma_t)^2\normLigne{\hat{y}_0}^2/2]}\\
    & \qquad \times \exp[-(m_t/\sigma_t)^2\normLigne{\hat{y}_1}^2/2] \rmd \hat{\pi}(\hat{y}_0) \rmd \hat{\pi}_1(\hat{y}_1) \exp[-(m_t/\sigma_t)^2\normLigne{x-p(x)}^2] \\
    &\textstyle{=(1/\lambda_t)^{2p}\int_{\lambda_t \varphi_0^{-1}(\msu_0) \times \lambda_t \varphi_1^{-1}(\msu_1)} (\varphi_0^{-1}(\hat{y}_0/\lambda_t) - \varphi_1^{-1}(\hat{y}_1/\lambda_t))^{\otimes 2}} \\
    & \qquad \times \exp[-\normLigne{\hat{y}_0}^2/2]\exp[-\normLigne{\hat{y}_1}^2/2] \rmd \hat{\pi}(\hat{y}_0) \rmd \hat{\pi}_1(\hat{y}_1)  \exp[-(m_t/\sigma_t)^2\normLigne{x-p(x)}^2] \eqsp . \label{eq:u0_u1_nonconvex}
  \end{align}
  In addition, using the dominated convergence theorem we have
  \begin{align}
    \label{eq:u0_u1_limit}
    &\lim_{t \to +\infty} \textstyle{(2\uppi)^{-p} \int_{\lambda_t \varphi_0^{-1}(\msu_0) \times \lambda_t \varphi_1^{-1}(\msu_1)} (\varphi_0^{-1}(\hat{y}_0/\lambda_t) - \varphi_1^{-1}(\hat{y}_1/\lambda_t))^{\otimes 2}} \\
    & \qquad \qquad \times \exp[-\normLigne{\hat{y}_0}^2/2]\exp[-\normLigne{\hat{y}_1}^2/2] \rmd \hat{\pi}(\hat{y}_0) \rmd \hat{\pi}_1(\hat{y}_1) = \hat{h}_0(p_0(x)) \hat{h}_1(p_1(x)) (p_0(x) - p_1(x))^{\otimes 2} \eqsp , 
  \end{align}
  where $\hat{h}_i$ is the density of $\hat{\pi}_i = (\varphi_i)_{\#} \pi$ \wrt the
  Lebesgue measure.  In addition, there exists $C_4 \geq 0$ such that for any
  $t \in \ocint{0,T}$
  \begin{align}
    &\textstyle{(2\uppi)^{-p/2} \int_{\msu_0} \exp[-(m_t / \sigma_t)^2\normLigne{x - y_0}^2/2] \rmd \pi(y_0)} \\
    & \qquad \qquad = \textstyle{(2\uppi)^{-p/2} \int_{\varphi_0^{-1}(\msu_0)} \exp[-(m_t / \sigma_t)^2\normLigne{\hat{y}_0}^2/2] \rmd \hat{\pi}(\hat{y}_0)} \exp[-(m_t/\sigma_t)^2\normLigne{x-p(x)}^2] \\
    & \qquad \qquad \leq C_4 \exp[-(m_t/\sigma_t)^2\normLigne{x-p(x)}^2] \lambda_t^{p/2} \eqsp . \label{eq:upper_bound_non_cvx}
  \end{align}
In addition, since $\msu_0$ is open we
  have that $\M \cap \msu_0^\complementary$ is compact and since for any
  $y \in \M \cap \msu_0^\complementary$,
  $\normLigne{y - x}^2 > \normLigne{p(x) - x}^2$, there exists $\vareps >0$ such
  that for any $y \in \M \cap \msu_0^\complementary$,
  $\normLigne{y - x}^2 \geq \normLigne{p(x) - x}^2 + \vareps$. Therefore, we
  have
  \begin{align}
    \textstyle{\int_{\M \cap \msu_0^\complementary} \exp[-(m_t/\sigma_t)^2\normLigne{x - y_0}^2/2] \rmd \pi(y_0)}  \leq C_5 \exp[-(m_t/\sigma_t)^2\normLigne{x-p(x)}^2/2] \exp[-\vareps (m_t/\sigma_t)^2] \eqsp .  \label{eq:outside_bound_non_convex}
  \end{align}
  Combining this result and \eqref{eq:upper_bound_non_cvx}, we get that there exists $C_6 > 0$ such that
  \begin{equation}
    (2 \uppi/\lambda_t^2)^{-p/2} \textstyle{\int_{\M} \exp[-(m_t/\sigma_t)^2\normLigne{x - y_0}^2/2] \rmd \pi(y_0)} \leq C_6 \exp[-(m_t/\sigma_t)^2\normLigne{x-p(x)}^2/2] \eqsp . 
  \end{equation}
  Combining this result, \eqref{eq:u0_u1_nonconvex} and \eqref{eq:u0_u1_limit}, there exists $C_7 > 0$ such that
    \begin{align}
    &\liminf_{t \to 0} \textstyle{\int_{\M} (y_0 - y_1)^{\otimes 2} \exp[-(m_t/\sigma_t)^2\normLigne{x - y_0}^2/2]\exp[-(m_t/\sigma_t)^2\normLigne{x - y_1}^2/2] \rmd \pi(y_0) \rmd \pi(y_1)} \\
    & \qquad \qquad  \textstyle{/ (\int_{\M} \exp[-(m_t/\sigma_t)^2\normLigne{x - y}^2/2] \rmd \pi(y))^2 } \geq (x_0 - x_1)^{\otimes 2} / C_6 \eqsp .
  \end{align}
  We conclude upon combining this result and \Cref{lemma:control_hessian}.
  \end{enumerate}
\end{proof}


\end{document}